%% file: neurips_2026.tex
\documentclass{article}

\PassOptionsToPackage{numbers,square,sort&compress}{natbib}
\usepackage[preprint]{neurips_2026}

\usepackage[utf8]{inputenc} 
\usepackage[T1]{fontenc}    
\usepackage{hyperref}       
\usepackage{url}            
\usepackage{booktabs}       
\usepackage{amsfonts}       
\usepackage{nicefrac}       
\usepackage{microtype}      
\usepackage{xcolor}         

\title{
DiRecT: Safe Diffusion-Based Planning via\\ Receding-Horizon Denoising
} 

\author{%
    Paolo Giaretta \\
    MIT \\
    \texttt{pgiarett@mit.edu} \\
    \And
    Zeyang Li \\
    MIT \\
    \texttt{zeyang@mit.edu} \\
    \And 
    Navid Azizan \\
    MIT \\
    \texttt{azizan@mit.edu} \\
}

\usepackage{amsmath,amssymb,amsfonts,amsthm,bm,mathtools}
\usepackage[ruled,vlined,linesnumbered]{algorithm2e}
\usepackage{algpseudocode}
\usepackage{graphicx}
\usepackage{subcaption}
\usepackage{booktabs}
\usepackage{array}
\usepackage{adjustbox}
\usepackage{multirow}
\usepackage{placeins}
\usepackage{etoc}
\usepackage{wrapfig}
\usepackage{afterpage}
\usepackage{float}

\input{macros}

\begin{document}

\maketitle

\begingroup
\renewcommand{\thefootnote}{}
\footnotetext{Correspondence to Zeyang Li (\texttt{zeyang@mit.edu}).}
\endgroup

\begin{abstract}
    \input{abstract}
\end{abstract}

\input{maintext}


\bibliographystyle{plainnat}
\bibliography{bibliography}


\newpage
\appendix

{\LARGE\bfseries Appendix}
\vspace{1em}

\etocsetlocaltop.toc{part}
\etocsetnexttocdepth{subsection}
\etocsettocstyle{}{}
\localtableofcontents

\vspace{2em}
\hrule
\vspace{2em}

\newpage
\input{appendix}

\etocsetlocaltop.toc{part} 

\FloatBarrier
\newpage

\end{document}

%% file: macros.tex
\newcommand{\R}{\mathbb{R}}
\newcommand{\N}{\mathbb{N}}

\newcommand{\I}{\bm{I}}

\newcommand{\X}{\bm{X}}

\newcommand{\traj}{\bm{\tau}}

\newcommand{\EE}{\mathbb{E}}

\newcommand{\given}{\,|\,}

\newcommand{\Normal}{\mathcal{N}}
\newcommand{\Unif}{\mathcal{U}}

\DeclareMathOperator*{\argmin}{arg\,min}
\DeclareMathOperator*{\argmax}{arg\,max}
\DeclareMathOperator{\Tr}{tr}

\newcommand{\norm}[1]{\left\lVert #1 \right\rVert}

\newcommand{\Loss}{\mathcal{L}}

\newcommand{\eps}{\varepsilon}

\newcommand{\pdata}{p_0}
\newcommand{\pprior}{p_1}

\newtheorem{theorem}{Theorem}
\newtheorem{proposition}{Proposition}
\newtheorem{remark}{Remark}

\newtheoremstyle{plainname}
  {3pt}   
  {3pt}   
  {\normalfont} 
  {}      
  {\bfseries} 
  {.}     
  {.5em}  
  {}      

\theoremstyle{plainname}
\newtheorem{problem}{Problem}
\newtheorem{problem*}{Problem}

\makeatletter
\newcommand{\algsection}[1]{%
  \Statex\hskip\ALG@thistlm \textcolor{gray}{\(\triangleright\) #1}%
}
\makeatother

\algrenewcommand\algorithmiccomment[1]{\hfill\textcolor{gray}{\(\triangleright\) #1}}

\newcolumntype{C}{>{\centering\arraybackslash}c}
\newcolumntype{G}{>{\color{gray}\centering\arraybackslash}c}

%% file: abstract.tex
Diffusion models have emerged as powerful tools for planning and control by learning multimodal distributions over actions and trajectories. Yet reliable inference-time safety enforcement remains a key barrier to their deployment in safety-critical tasks. Existing approaches typically project each denoising iterate onto the feasible set, even though constraints are defined only on the final clean trajectory. Enforcing feasibility on noisy intermediate samples can therefore overconstrain the sampling dynamics, substantially degrading sample quality.
To address this limitation, we introduce DiRecT (\textbf{Di}ffusion-based planning via \textbf{Rec}eding-horizon denoising with \textbf{T}erminal constraints), a training-free algorithm for constrained sampling from diffusion models via stochastic optimal control (SOC). DiRecT enforces constraints only on the final clean sample, avoiding unnecessary restrictions on the intermediate denoising dynamics. Inspired by model predictive control, we derive a principled receding-horizon surrogate for the otherwise intractable constrained SOC formulation, yielding an efficient algorithm that cleanly separates stochastic denoising from constraint satisfaction, progressively steering samples toward feasible final trajectories without distorting the learned diffusion dynamics. Furthermore, DiRecT is highly flexible: it can leverage off-the-shelf or domain-specific optimizers, incorporate priors over environment dynamics, and optimize additional soft rewards. Extensive experiments on safe planning benchmarks demonstrate that DiRecT substantially improves deployment safety and task performance over existing diffusion-based planning baselines. Our code is available at \url{https://github.com/azizanlab/DiRecT}.

%% file: maintext.tex
\section{Introduction}
\label{sec:introduction}
Safe and reliable planning remains a central challenge, requiring algorithms that can satisfy constraints while adapting to diverse environments, objectives, and constraint types. Classical planning methods~\citep{hart1968formal,kavraki1996probabilistic,lavalle1998rapidly,karaman2011samplingbasedalgorithmsoptimalmotion,ratliff2009chomp,schulman2014motion} have achieved strong results in many structured settings, but they face important limitations. Search-based methods can become computationally prohibitive in large-scale multi-agent problems~\citep{yu2015intractabilityoptimalmultirobotpath}, while optimization-based methods are often local and sensitive to initialization~\citep{ratliff2009chomp,schulman2014motion}. These limitations have driven growing interest in \emph{data-driven} generative planners, which learn reusable priors from offline data and can model diverse, high-dimensional behaviors.
Building on their success in image generation~\citep{dhariwal2021diffusionmodelsbeatgans}, diffusion models~\citep{sohldickstein2015deepunsupervisedlearningusing, ho2020denoisingdiffusionprobabilisticmodels, song2021scorebasedgenerativemodelingstochastic} have emerged as a powerful framework for planning and control, capable of capturing complex, multimodal state-action distributions~\citep{janner2022planningdiffusionflexiblebehavior} while supporting flexible inference-time guidance~\citep{ajay2023conditionalgenerativemodelingneed}.

Despite these advances, deploying diffusion-based planners in real-world safety-critical tasks remains challenging. For instance, a generated trajectory must avoid collisions with obstacles, as even a single violation can lead to catastrophic failure.
Since diffusion models rely on high-dimensional stochastic denoising dynamics, typically parameterized by large neural networks, they may generate unsafe plans even when trained on feasible trajectories. This challenge is further amplified when trajectories must satisfy novel constraints that are not captured by the dataset.
These limitations have motivated growing interest in \emph{test-time} mechanisms for constraining diffusion planners. Inspired by guidance techniques in image generation~\citep{dhariwal2021diffusionmodelsbeatgans, ho2021classifierfree, chung2023diffusionposterior}, early approaches steer samples during denoising by encoding constraint satisfaction as a \emph{soft} guidance signal~\citep{mizuta2024cobl}. However, soft guidance can only encourage feasibility and does not guarantee hard constraint satisfaction. More recent work has therefore focused on enforcing hard constraints directly on sampled trajectories, with the goal of providing stronger guarantees for safety-critical planning.
Particularly, most hard-constrained diffusion samplers enforce feasibility through projections or constraint-driven updates along the denoising trajectory~\citep{christopher2024constrained,luan2026projected,liang2025simultaneous,zhang2025constraineddiffuserssafeplanning,xiao2025safediffuser}. This creates a mismatch: planning constraints are defined on the final clean trajectory, while intermediate denoising iterates are noisy and need not themselves be feasible. Imposing constraints throughout sampling can therefore overconstrain the learned reverse dynamics, significantly degrading sample quality. Moreover, projection-based formulations do not naturally provide a unified mechanism for handling hard constraints together with additional soft rewards or costs.

To address these limitations, we introduce DiRecT (\textbf{Di}ffusion-based planning via \textbf{Rec}eding-horizon denoising with \textbf{T}erminal constraints), a \emph{training-free} algorithm for constrained sampling in diffusion-based planning. We formulate inference-time constraint enforcement through the lens of stochastic optimal control (SOC)~\citep{kappen2005pathintegrals}, where the learned reverse diffusion process serves as the nominal stochastic dynamics and control inputs steer the sampler toward feasibility of the final clean trajectory. Crucially, constraints are imposed only on this terminal clean sample, avoiding unnecessary restrictions on noisy intermediate denoising iterates.

Solving the resulting terminally constrained SOC problem is computationally intractable. We therefore exploit the structure of diffusion models, together with ideas from model predictive control (MPC)~\citep{rawlings2017modelpredictive}, to derive a principled and scalable receding-horizon surrogate. At each denoising step, DiRecT predicts the final clean trajectory implied by the current noisy iterate, solves a constrained optimization problem over this prediction, and converts the resulting refinement into a controlled update of the current noisy iterate.

The contributions of this paper are:
\begin{itemize}
    \item We identify a key limitation of prior constrained diffusion samplers: enforcing feasibility on noisy intermediate denoising iterates can overconstrain the sampling dynamics and degrade sample quality. In contrast, we formulate training-free constrained diffusion planning as a terminally constrained SOC problem that enforces feasibility only on the final clean trajectory while staying close to the learned diffusion dynamics.

    \item We present DiRecT, a training-free constrained diffusion sampler that reduces the intractable constrained SOC problem to a sequence of tractable clean-trajectory optimization subproblems. By refining predicted clean trajectories and translating these refinements into controlled updates of the noisy iterates, DiRecT steers samples toward constraint satisfaction without distorting the learned denoising process.

    \item We show that DiRecT is highly flexible, supporting off-the-shelf and domain-specific optimizers, equality and inequality constraints, environment-specific dynamics priors, and additional soft rewards or costs at inference time.

    \item We evaluate DiRecT across diverse robotic planning applications, including maze navigation in Maze2D, robotic manipulation in D3IL, multi-robot motion planning (MRMP), and diverse contact-rich manipulation in PushT. Across these tasks, DiRecT consistently improves constraint satisfaction and task success over existing diffusion-based planning baselines. 
\end{itemize}

\section{Related Work}
\label{sec:related works}
This section situates our work within the broader literature. First, we review constrained sampling techniques for diffusion-based planners, emphasizing the key limitation that motivates our method. Second, we discuss prior work that also explores connections between diffusion models and stochastic optimal control, clarifying how our constrained sampling perspective differs from these approaches. Finally, we mention a parallel but complementary line of work on training-based methods for planning with diffusion models. Due to space constraints, the details are deferred to Appendix~\ref{sec:Related works - appendix}.

\section{Background}
\label{sec:background}
\paragraph{Diffusion models.}
We employ the continuous-time formulation and define diffusion models as a pair of Itô processes: a \emph{forward} process that gradually corrupts samples drawn from the data distribution $\pdata$ through the stochastic differential equation (SDE)
$d X_t = f_t(X_t) d t + g_t(X_t) d W_t, \; X_0 \sim \pdata, \; t \in \left[0, 1\right]$,
and a \emph{backward} process that generates samples starting from the simple prior distribution $\pprior \approx \Normal\left(0, I_d \right)$ by time reversal of the forward SDE~\citep{anderson1982timereversal}:
\begin{equation}
\label{eq:backward-sde}
    d X_t = \left[f_t(X_t) - g_t^2(X_t) \nabla_{X_t} \log p_t(X_t) \right] dt + g_t(X_t) d \bar{W}_t, \quad X_1 \sim \pprior,
\end{equation}
where \eqref{eq:backward-sde} is integrated backward in time, $W_t$ and $\bar{W}_t$ denote the forward and backward Wiener processes, and $s_t(x_t) = \nabla_{x_t} \log p_t(x_t)$ is the score function of the marginal distribution $p_t$. Moreover, we restrict our focus to the common case of Gaussian-affine schedules, for which the conditional distribution of the forward noising process has the closed-form solution
$q_t(x_t | x_0) = \Normal\left(\alpha_t x_0, \sigma_t^2 I_d \right)$,
where the noise-schedule coefficients $\{\alpha_t, \sigma_t\}_{t \in \left[0, 1\right]}$ are related to the SDE drift and diffusion functions by
$f(x,t)=\frac{\dot \alpha_t}{\alpha_t}x, \quad g^2(t)=\frac{d}{dt}\sigma_t^2 - 2\frac{\dot \alpha_t}{\alpha_t}\sigma_t^2$. Following denoising score matching~\citep{vincent2011connection,song2021scorebasedgenerativemodelingstochastic}, the score function $s_t(x_t)$ is approximated by a neural network $s^\theta_t(x_t)$ trained to minimize the conditional score-matching (CSM) loss:
\begin{equation}
\label{eq:conditional score amtching loss}
    \Loss_{\mathrm{CSM}}(\theta) = \EE_{t,x_0,\eps} \left[ \lambda(t) \norm{s^\theta_t(\alpha_t x_0+\sigma_t\eps) + \frac{\eps}{\sigma_t}}^2 \right],
\end{equation}
where $t\sim\Unif[0,1]$, $x_0\sim\pdata$, $\eps\sim\Normal(0,I_d)$, and $\lambda(t)$ is a time-dependent weighting.
\paragraph{Sampling and Tweedie's formula.} 
Sampling from a diffusion model is achieved by numerically simulating the reverse dynamics~\eqref{eq:backward-sde} using the learned score function. Common stochastic samplers include DDPM~\citep{ho2020denoisingdiffusionprobabilisticmodels}, Euler--Maruyama~\citep{kloeden1992numerical}, and higher-order solvers~\citep{lu2025dpmsolverplusplus}. Despite our focus on \emph{stochastic} dynamics, few-step \emph{deterministic} samplers are commonly obtained by numerical integration of the probability-flow ODE~\citep{song2021scorebasedgenerativemodelingstochastic}, which shares the same marginal distributions as \eqref{eq:backward-sde}. 
Examples include DDIM~\citep{song2021denoisingdiffusionimplicitmodels} and ODE-based solvers~\citep{lu2022dpmsolver}. A one-step instantiation of these samplers yields an approximation of the posterior conditional mean of the forward process, which is related to the score function through Tweedie's formula~\citep{robbins1992empirical}:
\begin{equation}
\label{eq:tweedie formula}
\EE[X_0 \given X_t] \approx \hat{x}^\theta_0(x_t, t) = \frac{x_t + \sigma_t^2 s^\theta_t(x_t)}{\alpha_t}.
\end{equation}
Therefore, we use $s^\theta_t$ and $\hat{x}^\theta_{0}$ interchangeably when convenient, assuming they are related by \eqref{eq:tweedie formula}.
\paragraph{Diffusion planners.}
Following diffusion-based planning formulations~\citep{janner2022planningdiffusionflexiblebehavior,chi2025diffusion}, we treat the clean sample as a finite-horizon plan rather than a single configuration. Let $H\in\N$ be the prediction horizon and let $\mathcal H=\{0,\ldots,H\}$. We denote plans as $\traj=(\traj_0,\ldots,\traj_{H})\in\mathcal D^{H+1}$, where each element $\traj_k\in\mathcal D \subseteq \R^D$ may encode a state, an action, or a state-action pair depending on the planner parameterization. A diffusion planner learns a distribution over such plans and generates $\traj$ by denoising a Gaussian prior sample. In closed-loop deployment, the planner executes only the first $M\leq H$ elements before re-planning, and the concatenation of executed prefixes forms the rollout. 

\section{Method}
\label{sec:method}
We now formalize hard-constrained sampling for diffusion models and derive DiRecT. The derivation is similar in spirit to the test-time guidance framework for flow-matching models developed in~\citep{li2025hardflow}, while differing in the diffusion-model setting and the resulting algorithmic form. Due to space limitations, we defer the full derivation of the algorithm to Appendix~\ref{sec:algorithm derivation - Appendix}. Given a pretrained score model, a constraint set $\mathcal{S} \subseteq \R^d$, and a cost function $C: \R^d \rightarrow \R$, we seek to sample denoised plans $\traj$ from the model that are: (i) \emph{safe} ($\traj \in \mathcal{S}$), (ii) \emph{minimize} the cost $C$, and (iii) remain \emph{proximal} to the learned data distribution. We now expand on the problem setting and show how safe planning can naturally be framed as solving a stochastic optimal control problem with terminal constraints. We then derive a receding-horizon surrogate formulation that leads to a tractable algorithm. In this section, we use a general diffusion-centric notation and denote denoised plans as $\X_0$.  
\paragraph{Safe planning with diffusion models and stochastic optimal control.} Denote by $\tilde{f}_t^\theta (X_t)$ the reverse drift $f_t (X_t) - g(t)^2 s_t^\theta(X_t)$, so that the sampling dynamics of the trained model are $d X_t = \tilde{f}_t^\theta (X_t) dt + g(t) d \bar{W}_t, \;X_1 \sim p_1$. Let $u: \R^d \times [0, 1] \rightarrow \R^d$ be a control drift that steers the generative process toward low-cost feasible samples through the controlled SDE:
\begin{equation}
\label{eq:controlled sde}
    d X^u_t = \left[ \tilde{f}_t^\theta (X^u_t) + g(t) u_t(X^u_t) \right] dt + g(t) d \bar{W}_t, \;X_1 \sim p_1
\end{equation}
Our three requirements for a safe, performant, and proximal planner naturally fit within a terminally constrained stochastic optimal control perspective~\citep{bouchard2010optimal}, which aims to find the optimal control that solves Problem~\ref{prob:soc problem}: 
\begin{problem}[Stochastic optimal control problem with terminal constraints]
\label{prob:soc problem}
\textit{Given reference dynamics, constraint set $\mathcal{S}$, terminal cost $C$, and cost weight $\lambda$, solve for the optimal control drift $u^\star$:}
\par\noindent
\begin{equation}
\label{eq:soc problem}
\begin{aligned}
\min_{u}
\quad
& \EE \left[\lambda \, C(X^u_0) + \frac{1}{2} \int_0^1 \norm{u_t(X^u_t)}_2^2 \,dt \right] \\
\text{s.t.}\quad
& d X^u_t = \left[ \tilde{f}_t^\theta(X^u_t) + g(t) u_t(X^u_t) \right]dt + g(t)d\bar{W}_t, \quad X_1 \sim \pprior \ \\
& X^u_0 \in \mathcal S
\qquad
\mathbb{P}^u\text{-a.s.}
\end{aligned}
\end{equation}
\end{problem}
where $\mathbb{P}^u$ denotes the path measure induced by~\eqref{eq:controlled sde}. Notably, Problem~\ref{prob:soc problem} enforces constraint satisfaction only at the terminal state. Intermediate latents may therefore deviate significantly from the feasible set if doing so reduces control energy and keeps the process closer to the pretrained dynamics. In principle, the optimal controller can be obtained by solving the Hamilton--Jacobi--Bellman (HJB)~\citep{bellman1954theory} equation. However, exact solutions are rarely possible for the rich, high-dimensional dynamics induced by the backward SDE. See Appendix~\ref{sec:stochastic optimal control - Appendix} for an overview of SOC, the HJB equation, and their relation to KL control.
\paragraph{Receding-horizon surrogate formulation.}
We briefly summarize the main steps used to transform the continuous-time SOC problem into a tractable receding-horizon surrogate. We first discretize the reverse dynamics in \eqref{eq:backward-sde}. Let \(N\in\mathbb{N}\) denote the number of discretization steps, and let \(0=t_0<t_1<\cdots<t_N=1\) be the sampling grid. We write the uncontrolled reverse update from \(t_i\) to \(t_{i-1}\) as
\begin{equation}
\label{eq:sampling-step}
    X_{i-1}=\Phi_i^\theta(X_i,\varepsilon_i),
    \qquad
    \varepsilon_i\sim\mathcal{N}(0,I_d),
    \qquad i=1,\ldots,N,
\end{equation}
where \(\Phi_i^\theta:\mathbb{R}^d\times\mathbb{R}^d\to\mathbb{R}^d\) denotes the \(i\)-th denoising transition of the chosen sampler. This covers standard sampling schemes such as DDPM~\citep{ho2020denoisingdiffusionprobabilisticmodels} as well as deterministic or higher-order solvers.

To control generation, we perturb each uncontrolled transition by an Euler discretization of the control input over \([t_{i-1},t_i]\), yielding the \emph{discrete-time} version of the stochastic optimal control problem 
\begin{equation}
\label{eq:discrete-time-soc}
\begin{aligned}
\min_{\{X_i\}_{i=0}^N, \{u_i\}^N_{i=1}} \quad
& \mathbb{E}\left[\lambda C(X_0)
+ \frac{1}{2}\sum_{i=1}^N \|u_i(X_i)\|_2^2\Delta t_i\right] \\
\text{s.t.} \quad X_0 \in \mathcal{S}, \quad
& X_{i-1}=\Phi_i^\theta(X_i,\varepsilon_i)+g_i u_i(X_i)\Delta t_i,\quad
\varepsilon_i\sim\mathcal{N}(0,I_d),\quad i=1,\ldots,N,\quad
\end{aligned}
\end{equation}
Although discretization yields a finite-horizon control problem, the discrete SOC optimization remains computationally prohibitive for most diffusion-based applications. The decision variables are feedback controls over \(\R^d\), while the objective requires an expectation over all controlled stochastic trajectories. In addition, terminal constraints and costs couple all optimization variables with multi-step nonconvex neural dynamics, distorting the optimization landscape to the point that, even for a single noise realization, off-the-shelf solvers may struggle to find feasible solutions.

To tackle these limitations, we simplify the SOC formulation by introducing two additional approximations. First, we adopt a model predictive control (MPC) perspective, replacing the full-horizon problem with a sequence of one-step subproblems. To do so, we exploit the structure of diffusion models and use Tweedie's formula to construct proxy estimates of the terminal cost-to-go and terminal feasibility condition:
\begin{equation}
\label{eq:final-cost-estimation}
    \mathbb{E}\!\left[C(X_0)\mid X_{i}\right]
    \approx
    C\!\left(\hat{x}_0^\theta(X_{i},t_{i})\right),
    \qquad
    \hat{x}_0^\theta(X_{i},t_{i})\in\mathcal{S},
    \qquad i=0,\dots,N-1 .
\end{equation}
Second, to avoid repeatedly solving constrained optimization problems through the neural denoising map, we approximately invert Tweedie's formula by a fixed-point argument. This allows us to replace the latent displacement \(X_{i}-\bar{X}_{i}\) with a scaled difference between predicted clean samples:
\begin{equation}
\label{eq:state-differences-main}
    X_{i}-\bar{X}_{i}
    \approx
    \alpha_{t_{i}}
    \left(
    \hat{x}_0^\theta(X_{i},t_{i})
    -
    \hat{x}_0^\theta(\bar{X}_{i},t_{i})
    \right).
\end{equation}
Together, these approximations move the one-step constrained optimization from the latent space to the data domain, where the cost and constraint are evaluated directly on the predicted clean sample. We provide the full derivation from Problem~\ref{prob:soc problem} to the final algorithm in Appendix~\ref{sec:algorithm derivation - Appendix}, including all assumptions and approximations. We summarize the complete sampling procedure in Algorithm~\ref{alg:controlled-diffuser}.

Although several approximations are introduced to reduce the general stochastic optimal control problem to a tractable form, terminal constraint satisfaction is not relaxed, as formalized in Proposition~\ref{prop:terminal-feasibility}.
\begin{proposition}[Sample feasibility]
\label{prop:terminal-feasibility}
Suppose that the final subproblem in Algorithm~\ref{alg:controlled-diffuser} admits a feasible solution \(\hat{X}^\star_{0|0}\in\mathcal{S}\). Then the final sample returned by the recursion satisfies the terminal constraint.
\end{proposition}

\begin{proof}
By construction, the final step of Algorithm~\ref{alg:controlled-diffuser} returns \(X^\star_0=\hat{X}^\star_{0|0}\). Since the final subproblem imposes \(\hat{X}^\star_{0|0}\in\mathcal{S}\), it follows immediately that the final sample satisfies the constraint \(X^\star_0\in\mathcal{S}\).
\end{proof}

\begin{algorithm}[t]
\caption{DiRecT: Safe diffusion-based planning}
\label{alg:controlled-diffuser}
\small
\KwIn{Learned score matching model \(s_t^\theta(\cdot)\), cost \(C(\cdot)\), cost weight \(\lambda > 0\), feasible set \(\mathcal{S}\), discretization steps \(N\), time grid \(0=t_0 < t_1 < \ldots < t_N=1\), sampling scheme \(\{ \Phi^\theta_i \}_{i=1}^N\), differentiable affine scheduler \((\alpha_t, \sigma_t)\).}
\KwOut{Safe plan \(X_0^\star\).}

\(X_N^\star \sim \mathcal{N}(0,I_d)\)\tcp*{Sample prior}

\For{\(i=N,\ldots,1\)}{
    \(\Delta t_i \gets t_i-t_{i-1}\)\tcp*{Time step}

    \(\varepsilon_i\sim\mathcal{N}(0,I_d)\)\tcp*{Reverse noise}

    \(g_i \gets \left(\frac{d}{dt}\sigma_t^2 - 2\frac{\dot{\alpha}_t}{\alpha_t}\sigma_t^2\right)^{1/2}\bigg|_{t=t_i}\)\tcp*{Diffusion coefficient}

    \(\bar{X}^{\varepsilon_i}_{i-1}\gets \Phi_i^\theta(X_i^\star,\varepsilon_i)\)\tcp*{Uncontrolled denoising proposal}

    \(\tilde{X}_{0|i-1}\gets \hat{x}_0^\theta(\bar{X}^{\varepsilon_i}_{i-1},t_{i-1})\)\tcp*{Tweedie clean-sample estimate}

    Solve the constrained subproblem\;
    \refstepcounter{equation}\label{eq:controlled-diffuser-subproblem}
    \noindent\makebox[\dimexpr\linewidth-2em\relax][l]{%
    \(\displaystyle
    \hat{X}^\star_{0|i-1}
    \in
    \operatorname*{arg\,min}_{\hat{X}_{0|i-1}\in\mathcal{S}}
    \lambda C(\hat{X}_{0|i-1})
    +
    \frac{\alpha_{t_{i-1}}^2}{2g_i^2\Delta t_i}
    \|\hat{X}_{0|i-1}-\tilde{X}_{0|i-1}\|_2^2
    \)\;
    \hfill\textup{(\theequation)}%
    }
    \eIf{\(i>1\)}{
        \(X^\star_{i-1}\gets
        \bar{X}^{\varepsilon_i}_{i-1}
        +\alpha_{t_{i-1}}(\hat{X}^\star_{0|i-1}-\tilde{X}_{0|i-1})\)\tcp*{Latent correction}
    }{
        \(X^\star_0\gets \hat{X}^\star_{0|0}\)\tcp*{Return feasible terminal sample}
    }
}
\Return{\(X_0^\star\)}
\end{algorithm}

\section{Experiments}
\label{sec:experimental analysis}
We seek to evaluate DiRecT to answer the following questions: (i) \emph{Can our method enforce test-time constraints, while retaining task performance?} (ii) \emph{How does DiRecT compare with existing methods in terms of safety, task success, and computational overhead?} (iii) \emph{How flexible is our method in handling equality and inequality constraints, high-dimensional nonconvex constraint sets, and soft rewards or costs?} To answer these questions, we test our method on maze navigation in Maze2D~\citep{fu2020d4rl}, robotic manipulation in D3IL~\citep{jia2024towards}, multi-robot motion planning~\citep{shaoul2025multirobot}, and diverse contact-rich manipulation in PushT~\citep{chi2025diffusion, luan2026projected}. Detailed settings on the experimental setup are provided in Appendix~\ref{sec:experimental details - Appendix} and additional results and ablations in Appendix~\ref{sec:additional results - Appendix}.
\subsection{Safe maze navigation}
\label{sec:maze2d experiments}
\newlength{\panelheight}
\setlength{\panelheight}{3.5cm}
\begin{figure}[t]
    \centering
    \begin{subfigure}[t]{0.24\linewidth}
        \centering
        \includegraphics[height=\panelheight,keepaspectratio]{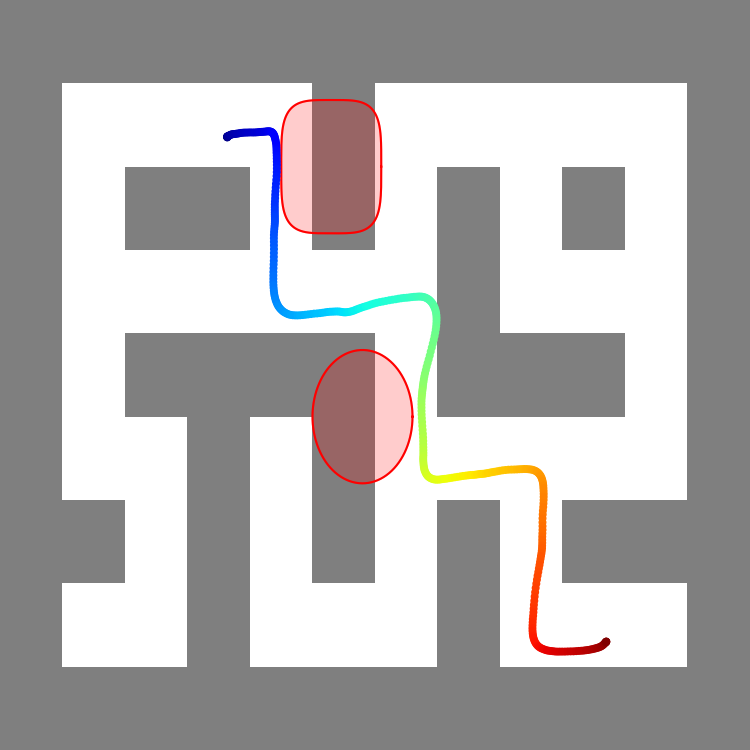}
        \caption{}
        \label{fig:maze2d broad}
    \end{subfigure}
    \hfill
    \begin{subfigure}[t]{0.24\linewidth}
        \centering
        \includegraphics[height=\panelheight,keepaspectratio]{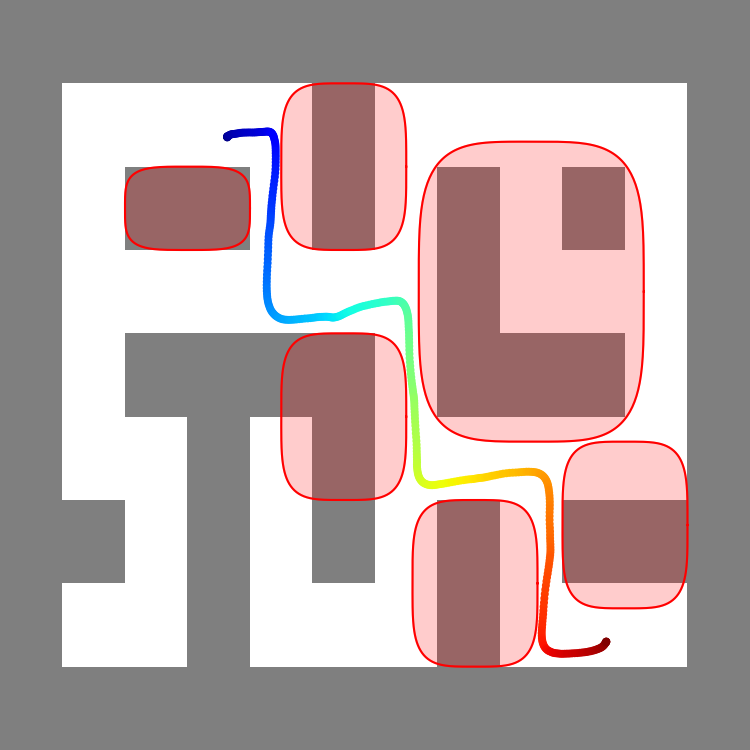}
        \caption{}
        \label{fig:maze2d narrow}
    \end{subfigure}
    \hfill
    \begin{subfigure}[t]{0.24\linewidth}
        \centering
    \includegraphics[
        height=\panelheight,
        keepaspectratio,
        trim={180pt 00pt 20pt 00pt},
        clip
    ]{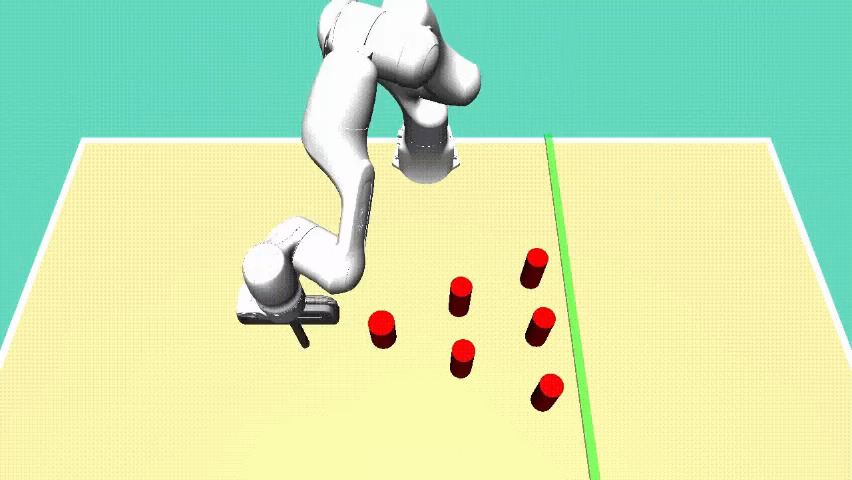}        
    \caption{}
        \label{fig:d3il 3d}
    \end{subfigure}
    \hfill
    \begin{subfigure}[t]{0.24\linewidth}
        \centering
        \includegraphics[height=\panelheight,keepaspectratio]{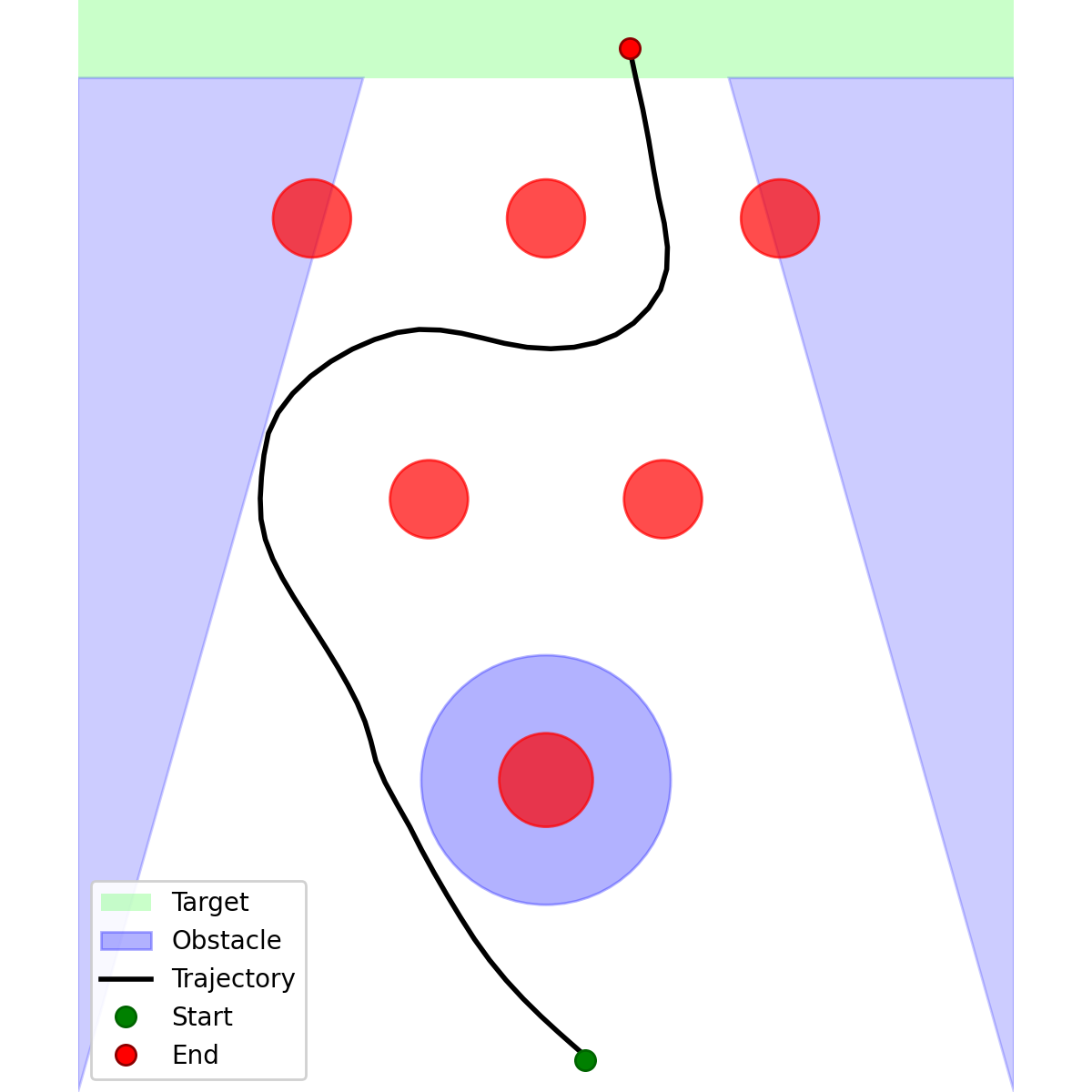}
        \caption{}
        \label{fig:d3il 2d}
    \end{subfigure}
    \caption{Visualization of DiRecT for maze navigation and robotic manipulation with test-time hard constraints. In \texttt{Maze2D-broad}~(\ref{fig:maze2d broad}) and \texttt{Maze2D-narrow}~(\ref{fig:maze2d narrow}), the robot has to navigate inside a maze toward the bottom-right corner while avoiding the obstacles represented in red. In \texttt{D3IL-avoiding}, a robotic arm moves toward the green target line, while avoiding six pillars~(\ref{fig:d3il 3d}). We constrain the problem during inference with additional obstacles, represented in blue (\ref{fig:d3il 2d}).}
    \label{fig:maze2d and d3il}
\end{figure}
First, we test our method on safe maze navigation for the D4RL \texttt{Maze2D-large} environment~\citep{fu2020d4rl}. In this task, a small ball traverses a planar maze toward a target position, while avoiding additional \emph{test-time} obstacles as depicted in Figure~\ref{fig:maze2d and d3il}. We train a continuous-time diffusion model to sample state-action pairs over the prediction horizon $H$, where each state $s \in \R^4$ includes position and velocity of the ball and actions $a \in \R^2$ represent force inputs. During inference, we sample \emph{one} trajectory from the model conditioned on endpoint states and execute actions by tracking with a proportional-derivative (PD) controller. Although the sampled \emph{planned} trajectory may satisfy test-time constraints, the \emph{rollout} trajectory can deviate from the plan and become infeasible due to mismatch between learned and environment dynamics. To improve rollout fidelity~\citep{romer2025diffusion, bouvier2025ddat}, we impose additional \emph{equality} constraints in the form of linearized dynamics fitted from the training data (see Appendix~\ref{sec:constrained navigation in Maze2d - Appendix details}). 
\paragraph{Baselines.} We compare against five baselines to represent different classes of constrained planners: \textbf{Diffuser}~\citep{janner2022planningdiffusionflexiblebehavior} (unguided reference), \textbf{Gradient Guidance}~\citep{chung2023diffusionposterior} (soft constraining with posterior sampling), \textbf{Projected Diffusion}~\citep{christopher2024constrained} (per-step projection), \textbf{Augmented Lagrangian}~\citep{zhang2025constraineddiffuserssafeplanning} (early-stage Lagrangian constraint relaxation), and \textbf{SafeDiffuser-RoS}~\citep{xiao2025safediffuser} (CBF-based guidance). We adapt existing methods to include dynamic constraints and employ IPOPT~\citep{wachter2006implementation} as an \emph{off-the-shelf} solver. Additional results are provided in Appendix~\ref{sec:constrained navigation in Maze2d - Appendix}.
\paragraph{Evaluation.} We test all methods over two progressively challenging variants~\citep{chi2025diffusion,zhang2025constraineddiffuserssafeplanning,janner2022planningdiffusionflexiblebehavior}, which we denote as \texttt{broad} (Fig.~\ref{fig:maze2d broad}) and \texttt{narrow} (Fig.~\ref{fig:maze2d narrow}). For each variant, we evaluate $100$ independent trials to measure how well each method adapts to additional prior environment knowledge in the form of dynamic constraints for safe navigation. We report in Table~\ref{tab:maze2d results main} four evaluation metrics: \emph{Safety Rate} (SR), the fraction of collision-free rollouts; \emph{Violations}, the mean number of steps the ball spends inside obstacles; D4RL normalized \emph{Score}, which reflects task completion performance; and sampling \emph{Time}. DiRecT is the only method that achieves high safety and near-zero violations while retaining strong performance.
\input{tables/maze2d_results_main}
\subsection{Safe robotic manipulation}
\label{sec:d3il experiments}
We next evaluate the reliability and performance of our method for closed-loop robotic manipulation on the D3IL \texttt{avoiding} environment~\citep{jia2024towards}. In this task, a robotic manipulator moves toward a target region while its end-effector avoids six pillars (Fig.~\ref{fig:d3il 3d}). We train a continuous-time diffusion model on the D3IL demonstration dataset, consisting of expert state-action trajectories, where the robot state $s \in \R^4$ is composed of the current and desired positions, and the action $a \in \R^2$ represents the end-effector velocity. At test time, the end-effector must reach the target region while avoiding the environment pillars and \emph{additional} planar and circular constraints (Fig.~\ref{fig:d3il 2d}). Rollouts are performed in a receding-horizon fashion, in which we execute $M < H$ actions before re-planning. Similarly to the maze navigation task, we test the ability of our method to exploit additional environment priors in the form of linearized dynamic constraints.
\paragraph{Baselines.} We compare against the same baselines as in the Maze2D setting. Additional results and baselines are presented in Appendix~\ref{sec:constrained manipulation in D3IL - Appendix}. \paragraph{Evaluation.} We test each method over $100$ i.i.d. rollouts and compute the following metrics: \emph{Safety Rate} (SR), the fraction of trials satisfying both environment and test-time constraints; \emph{Task Success} (TS), the proportion of trajectories that reach the end goal in the allowed number of steps; average \emph{Steps} taken to reach the target for safe rollouts; sampling \emph{Time}. Results in Table~\ref{tab:d3il results main} show that our method is the only one to maintain performance and safety for \emph{all} initializations. 
\input{tables/d3il_results_main}
\subsection{Safe multi-robot motion planning}
\label{sec:multi robot motion planning}
We test DiRecT on the coupled diffusion generation benchmark of~\citep{luan2026projected}, which is based on the multi-robot motion planning (MRMP) environment introduced by~\citep{shaoul2025multirobot}. Given starting locations for each robot, the goal is to simultaneously sample trajectories for all robots from a diffusion model trained on single-agent examples, while avoiding \emph{inter-robot} and \emph{obstacle} collisions, satisfying inference-time \emph{velocity} constraints, and exhibiting target environment-specific behaviors (Figure~\ref{fig:mrmp example}). Given the high-dimensional and nonconvex nature of the MRMP problem, we test the robustness of different methods by scaling the number of agents up to twenty, while imposing velocity limits as in~\citep{luan2026projected}.
\begin{wrapfigure}[13]{r}{0.30\textwidth}
    \centering
    \includegraphics[width=0.25\textwidth]{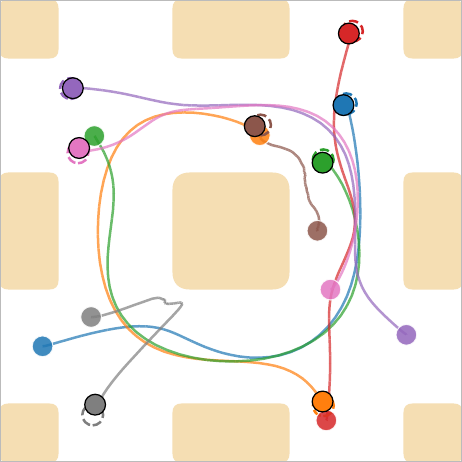}
    \caption{\texttt{Highways} with $8$ agents. Desired motion is counterclockwise around the central obstacle.}
    \label{fig:mrmp example}
\end{wrapfigure}
\paragraph{Baselines.}We compare our method with the following references: \textbf{Diffuser}~\citep{janner2022planningdiffusionflexiblebehavior} (unconstrained reference), \textbf{MMD-CBS}~\citep{shaoul2025multirobot} (tree-search collision search without velocity constraints), \textbf{PCD-LB} and \textbf{PCD-SHD}~\citep{luan2026projected} (soft guidance for collision avoidance and projection for velocity feasibility). For fast and parallelizable projection of trajectories onto the nonconvex feasible set, we develop a custom \emph{domain-specific} optimizer that is amenable to just-in-time compilation and GPU parallelization in JAX~\citep{jax2018github} (see Appendix~\ref{sec:custom MRMP optimizer}). To show that the interplay between the generative process and projection optimization is integral to the success and reliability of the algorithm, we implement \textbf{Final Projection}, in which samples are generated unconditionally and then projected \emph{post-sampling} onto the feasible set.
\paragraph{Evaluation.} We assess each method over \(100\) starting configurations, generating \(128\) candidate trajectories for each initialization as in~\citep{luan2026projected}. We report the following metrics: \emph{Success Rate}, which is the fraction of trials with at least one collision-free trajectory among the generated candidate trajectories; \emph{Constraint Safety}, which is the proportion of \emph{all} samples satisfying collision and kinematic constraints, and \emph{Data Adherence}, which measures environment-specific task success. As shown in Figure~\ref{fig:mrmp results main}, DiRecT outperforms all other baselines in terms of safety, especially when scaling to twenty agents. Note that the relatively lower data adherence on \texttt{Conveyor} is not representative of a relative lack of performance, as all other baselines mostly generate infeasible samples. We provide additional experimentation details in Appendix~\ref{sec:safe multi-robot motion planning - Appendix details} and extensive sweep results in Appendix~\ref{sec:safe multi-robot motion planning - Appendix}.

\begin{figure}[t]
    \centering

    \begin{subfigure}[t]{\linewidth}
        \centering
        \includegraphics[width=\linewidth,keepaspectratio]{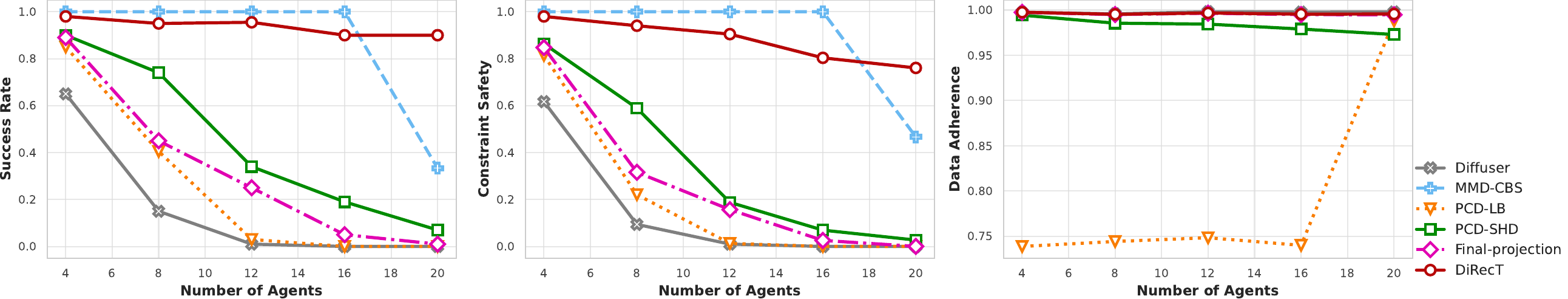}
        \caption{\texttt{Empty} with agent velocities limited to $0.675$.}
        \label{fig:four-rows-plot1}
    \end{subfigure}

    \vspace{0.35em}

    \begin{subfigure}[t]{\linewidth}
        \centering
        \includegraphics[width=\linewidth,keepaspectratio]{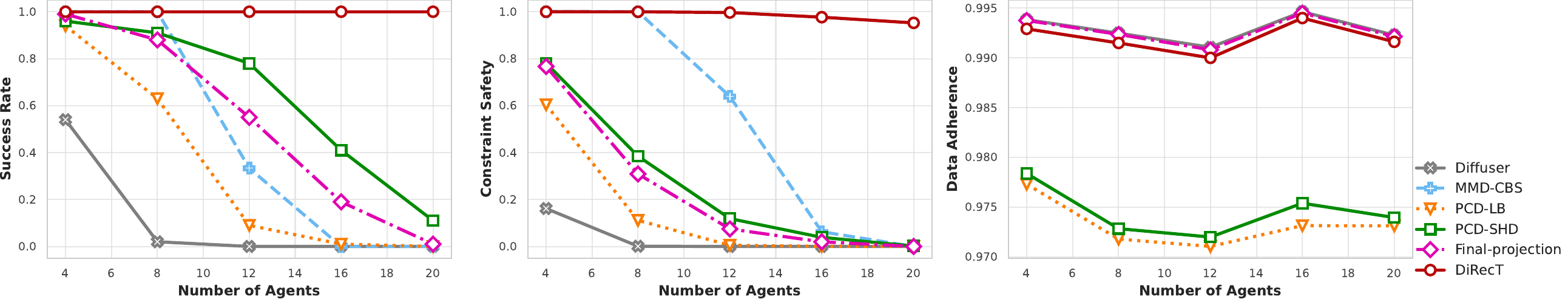}
        \caption{\texttt{Highways} with agent velocities limited to $0.647$.}
        \label{fig:four-rows-plot2}
    \end{subfigure}

    \vspace{0.35em}

    \begin{subfigure}[t]{\linewidth}
        \centering
        \includegraphics[width=\linewidth,keepaspectratio]{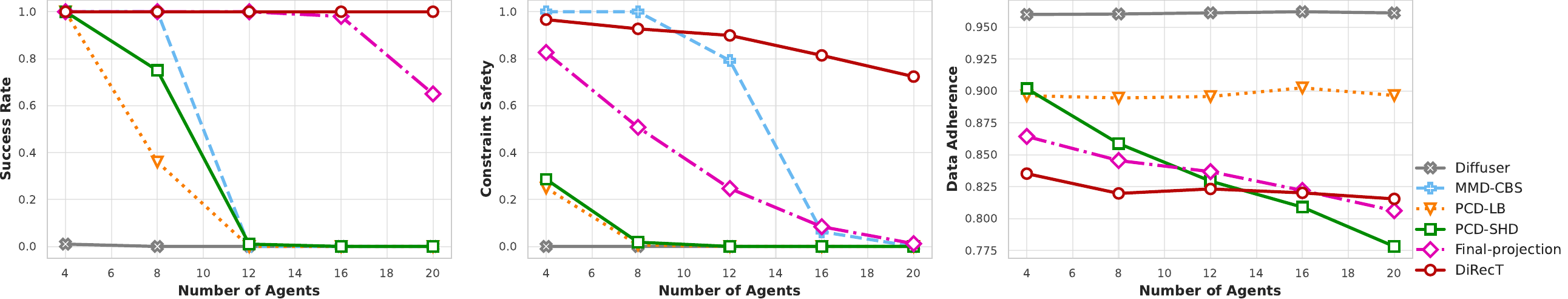}
        \caption{\texttt{Conveyor} with agent velocities limited to $1.21$.}
        \label{fig:four-rows-plot3}
    \end{subfigure}

    \vspace{0.35em}

    \begin{subfigure}[t]{\linewidth}
        \centering
        \includegraphics[width=\linewidth,keepaspectratio]{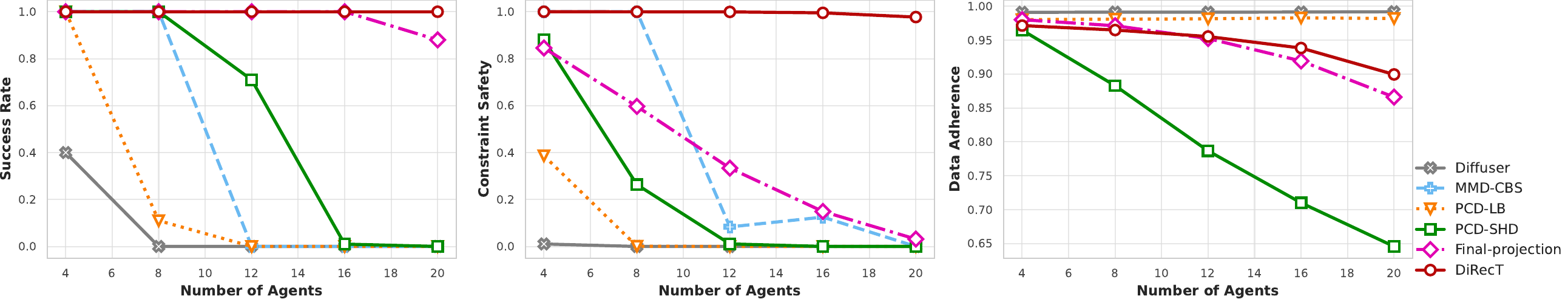}
        \caption{\texttt{Drop-Region} with agent velocities limited to $0.928$.}
        \label{fig:four-rows-plot4}
    \end{subfigure}

    \caption{Performance comparison for $N \in \{4, 8, 12, 16, 20\}$ agents on the four MMD environments~\citep{shaoul2025multirobot}. We impose obstacle and inter-robot constraints and restrict agent velocities. Our method outperforms all other baselines in terms of safety, while maintaining adherence to the expected motion pattern.}
    \label{fig:mrmp results main}
\end{figure}
\subsection{Safe and diverse contact-rich manipulation}
Finally, we show that DiRecT can effectively incorporate and optimize over \emph{soft} costs while satisfying test-time constraints. We base our evaluation on simultaneous coupled trajectory generation on \texttt{PushT}~\citep{florence2021implicit}. In the single-agent setting, the objective is to push a T-shaped block from a randomized starting position until it overlaps with a target pose (Figure~\ref{fig:pusht example}). Following~\citep{luan2026projected}, we consider the simultaneous generation of a pair of distinct trajectories while imposing \emph{hard} constraints on the maximum velocity and \emph{soft} coupling costs to encourage non-intersecting, diverse trajectories. 
\paragraph{Baselines.} We compare against vanilla \textbf{Diffusion Policy}~\citep{chi2025diffusion}, Projected Coupled Diffusion variants with Determinantal Point Process (DPP) and Log-barrier (LB) losses, both without velocity projection (\textbf{CD-DPP}, \textbf{CD-LB}~\citep{luan2026projected}) and with velocity projection (\textbf{PCD-DPP}, \textbf{PCD-LB}~\citep{luan2026projected}). To demonstrate that our method can improve diversity with the same computational overhead, we incorporate DPP and LB costs into \eqref{eq:controlled-diffuser-subproblem}, and solve the resulting optimization with one step of Projected Gradient Descent~\citep{levitin1966constrained}, thereby requiring one projection and one gradient computation per sampling step (see Appendix~\ref{sec:diverse constrained planning in PushT - Appendix Details}). We denote the two variants of our method as \textbf{DiRecT-DPP} and \textbf{DiRecT-LB}.
\begin{wrapfigure}[15]{r}{0.30\textwidth}
    \centering
    \includegraphics[width=0.25\textwidth]{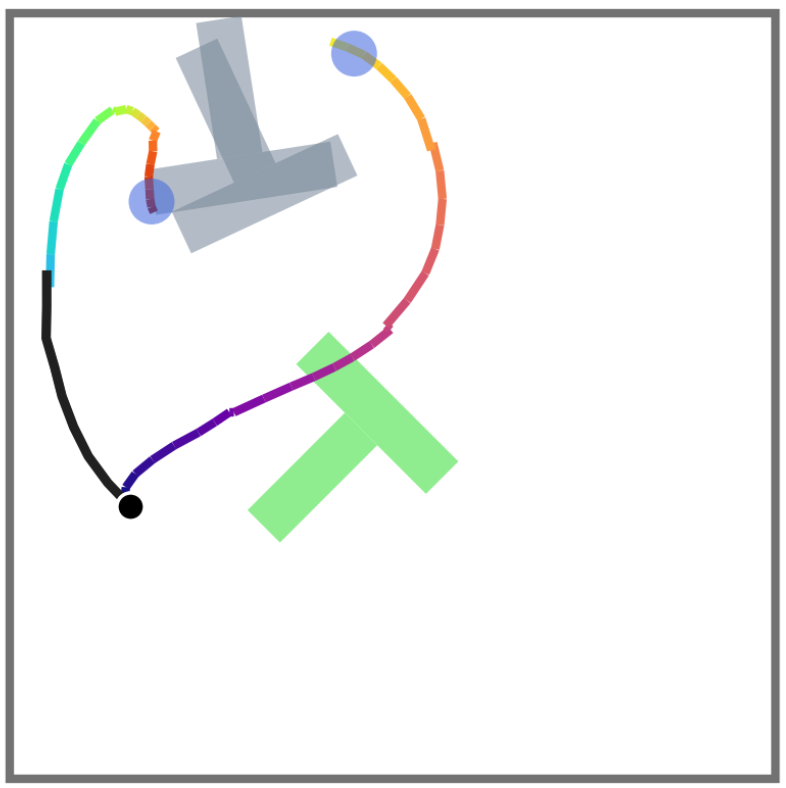}
    \caption{Coupled generation on \texttt{PushT}. Two agents push gray blocks onto a green target pose while avoiding intersecting trajectories.}
    \label{fig:pusht example}
\end{wrapfigure}
\paragraph{Evaluation.} We follow the same experimental procedure, generating $100$ \emph{pairs} of trajectories over $50$ random initializations for the block and agent positions and report the following metrics: \emph{Dynamic Time Warping} (DTW)~\citep{berndt1994using, muller2007information} and \emph{Discrete Fréchet Distance} (DFD)~\citep{alt1995discretefrechetdistance} to evaluate diversity; \emph{Task Completion} score (TC)~\citep{florence2021implicit, chi2025diffusion}, measuring successful overlap of the blocks onto the respective targets; \emph{Constraint Satisfaction} rate (CS)~\citep{luan2026projected} as the fraction of samples satisfying velocity limits; and sampling \emph{Time}. As summarized in Table~\ref{tab:pusht results main}, our formulation improves both diversity and task success relative to other constrained baselines, highlighting the importance of performing both \emph{soft} and \emph{hard} guidance near the data manifold. As observed by~\citep{luan2026projected}, diversity objectives conflict with task completion, leading to Pareto-optimal solutions as a function of guidance strength. In Appendix~\ref{sec:pareto-optimal analysis for PushT}, we show that DiRecT dominates PCD in terms of Pareto optimality. 
\input{tables/pusht_results_main}

\section{Conclusion}
\label{sec: Conclusions}
In this paper, we introduced DiRecT, a training-free algorithm for constrained sampling in safe diffusion-based planning. By formulating constrained sampling as a terminally constrained stochastic optimal control problem and deriving a scalable receding-horizon surrogate, DiRecT separates stochastic denoising from trajectory-level constraint satisfaction. This design avoids overconstraining noisy intermediate iterates while retaining the flexibility to incorporate domain-specific optimizers, dynamics priors, and additional soft objectives. Experiments across diverse robotic planning domains demonstrate that DiRecT improves both constraint satisfaction and task success, highlighting its potential for reliable deployment in safety-critical planning tasks.

\paragraph{Limitations and future work.}
Our method enforces hard constraints for pretrained diffusion planners in a \emph{training-free} manner, making it broadly applicable to existing models without retraining. This flexibility, however, comes with additional inference-time computation. A natural direction for future work is to combine terminal constraint control with model fine-tuning, reducing runtime overhead and improving adaptability to test-time constraints.

%% file: tables/maze2d_results_main.tex
\begin{table}[t]
\centering
\small
\caption{Results on the safe navigation in \texttt{Maze2D}. Mean values over $100$ i.i.d. samples and reported alongside their standard deviation. Boldface indicates the best score.}
\label{tab:maze2d results main}
\begin{adjustbox}{max width=\linewidth}
\begin{tabular}{llcccc}
\toprule
\textbf{Env}
& \textbf{Method}
& \textbf{SR} $\uparrow$
& \textbf{Violations} $\downarrow$
& \textbf{Score} $\uparrow$
& \textbf{Time (s)} $\downarrow$ \\
\midrule

\multirow{6}{*}{Maze2D Broad}
& Diffuser~\citep{janner2022planningdiffusionflexiblebehavior}
& 0.000
& $45.92{\scriptstyle \pm 11.86}$
& $\mathbf{1.601{\scriptstyle \pm 0.025}}$
& $0.316{\scriptstyle \pm 0.078}$ \\
& Gradient Guidance~\citep{chung2023diffusionposterior}
& 0.040
& $25.52{\scriptstyle \pm 10.18}$
& $1.587{\scriptstyle \pm 0.164}$
& $3.637{\scriptstyle \pm 0.076}$ \\
& Projected Diffusion~\citep{christopher2024constrained}
& 0.220
& $42.00{\scriptstyle \pm 47.28}$
& $1.119{\scriptstyle \pm 0.708}$
& $46.398{\scriptstyle \pm 11.096}$ \\
& Augmented Lagrangian~\citep{zhang2025constraineddiffuserssafeplanning}
& 0.080
& $72.92{\scriptstyle \pm 92.75}$
& $-0.010{\scriptstyle \pm 0.145}$
& $5.676{\scriptstyle \pm 0.566}$ \\
& SafeDiffuser-RoS~\citep{xiao2025safediffuser}
& 0.040
& $66.26{\scriptstyle \pm 41.72}$
& $1.186{\scriptstyle \pm 0.533}$
& $148.967{\scriptstyle \pm 3.831}$ \\
& \textbf{DiRecT (Ours)}
& \textbf{0.970}
& $\mathbf{0.45{\scriptstyle \pm 3.81}}$
& $1.500{\scriptstyle \pm 0.419}$
& $17.657{\scriptstyle \pm 0.629}$ \\

\midrule

\multirow{6}{*}{Maze2D Narrow}
& Diffuser~\citep{janner2022planningdiffusionflexiblebehavior}
& 0.000
& $103.94{\scriptstyle \pm 21.88}$
& $1.601{\scriptstyle \pm 0.025}$
& $0.321{\scriptstyle \pm 0.086}$ \\
& Gradient Guidance~\citep{chung2023diffusionposterior}
& 0.000
& $51.30{\scriptstyle \pm 19.22}$
& $1.619{\scriptstyle \pm 0.0029}$
& $3.608{\scriptstyle \pm 0.023}$ \\
& Projected Diffusion~\citep{christopher2024constrained}
& 0.780
& $37.29{\scriptstyle \pm 129.88}$
& $0.163{\scriptstyle \pm 0.459}$
& $112.918{\scriptstyle \pm 19.517}$ \\
& Augmented Lagrangian~\citep{zhang2025constraineddiffuserssafeplanning}
& 0.000
& $217.51{\scriptstyle \pm 159.43}$
& $0.156{\scriptstyle \pm 0.449}$
& $6.022{\scriptstyle \pm 0.122}$ \\
& SafeDiffuser-RoS~\citep{xiao2025safediffuser}
& 0.000
& $278.87{\scriptstyle \pm 209.87}$
& $0.735{\scriptstyle \pm 0.681}$
& $509.850{\scriptstyle \pm 19.929}$ \\
& \textbf{DiRecT (Ours)}
& \textbf{0.940}
& $\mathbf{0.23{\scriptstyle \pm 1.04}}$
& $\mathbf{1.624{\scriptstyle \pm 0.027}}$
& $98.093{\scriptstyle \pm 5.047}$ \\

\bottomrule
\end{tabular}
\end{adjustbox}
\end{table}

%% file: tables/d3il_results_main.tex
\begin{table}[t]
\centering
\small
\setlength{\tabcolsep}{5.5pt}
\renewcommand{\arraystretch}{1.08}
\caption{Results on constrained manipulation in D3IL \texttt{avoiding}. Mean values are computed over $100$ i.i.d. trials and reported with standard deviations where applicable. Boldface indicates the best score.}
\label{tab:d3il results main}
\begin{adjustbox}{max width=\linewidth}
\begin{tabular}{llcccc}
\toprule
\textbf{Env}
& \textbf{Method}
& \textbf{SR} $\uparrow$
& \textbf{Task Success} $\uparrow$
& \textbf{Steps (Safe)} $\downarrow$
& \textbf{Time (s)} $\downarrow$ \\
\midrule

\multirow{6}{*}{D3IL Avoiding}
& Unconstrained Diffuser~\citep{janner2022planningdiffusionflexiblebehavior}
& 0.060
& 0.960
& $70.30{\scriptstyle \pm 11.83}$
& $0.224{\scriptstyle \pm 0.002}$ \\
& Gradient Guidance~\citep{chung2023diffusionposterior}
& 0.880
& 0.980
& $\mathbf{62.58{\scriptstyle \pm 7.05}}$
& $6.615{\scriptstyle \pm 0.540}$ \\
& Projected Diffusion~\citep{christopher2024constrained}
& 0.770
& 0.970
& $63.77{\scriptstyle \pm 6.88}$
& $1.367{\scriptstyle \pm 0.057}$ \\
& Augmented Lagrangian~\citep{zhang2025constraineddiffuserssafeplanning}
& 0.220
& 0.840
& $63.45{\scriptstyle \pm 7.22}$
& $4.152{\scriptstyle \pm 0.018}$ \\
& SafeDiffuser-RoS~\citep{xiao2025safediffuser}
& 0.370
& 0.820
& $63.19{\scriptstyle \pm 11.56}$
& $18.890{\scriptstyle \pm 0.877}$ \\
& \textbf{DiRecT (Ours)}
& \textbf{1.000}
& \textbf{1.000}
& $62.86{\scriptstyle \pm 6.90}$
& $0.688{\scriptstyle \pm 0.061}$ \\

\bottomrule
\end{tabular}
\end{adjustbox}
\end{table}

%% file: tables/pusht_results_main.tex
\begin{figure}[t]
\centering
\small
\renewcommand{\arraystretch}{1.08}

\captionof{table}{Results on generation of safe and diverse contact-rich manipulation in \texttt{PushT}, with a velocity limit of $8.4$. Best overall value(s) in boldface; underlining denotes best among constrained samplers.
}
\label{tab:pusht results main}
\begin{adjustbox}{max width=\linewidth}
\begin{tabular}{lccccc}
\toprule
\textbf{Method}
& \textbf{DTW} $\uparrow$
& \textbf{DFD} $\uparrow$
& \textbf{CS} $\uparrow$
& \textbf{TC} $\uparrow$
& \textbf{Time (s)} $\downarrow$ \\
\midrule
Diffusion Policy~\citep{chi2025diffusion}
& $3.26{\scriptstyle \pm 1.280}$
& $0.477{\scriptstyle \pm 0.167}$
& $0.6549{\scriptstyle \pm 0.098}$
& $\mathbf{0.928{\scriptstyle \pm 0.194}}$
& $0.136{\scriptstyle \pm 0.015}$ \\

CD-LB~\citep{luan2026projected}
& $\mathbf{4.242{\scriptstyle \pm 1.292}}$
& $\mathbf{0.610{\scriptstyle \pm 0.166}}$
& $0.592{\scriptstyle \pm 0.136}$
& $0.907{\scriptstyle \pm 0.209}$
& $0.138{\scriptstyle \pm 0.006}$ \\

CD-DPP~\citep{luan2026projected}
& $3.871{\scriptstyle \pm 1.169}$
& $0.560{\scriptstyle \pm 0.150}$
& $0.632{\scriptstyle \pm 0.134}$
& $0.921{\scriptstyle \pm 0.207}$
& $0.142{\scriptstyle \pm 0.010}$ \\

PCD-LB~\citep{luan2026projected}
& $3.928{\scriptstyle \pm 1.238}$
& $0.557{\scriptstyle \pm 0.159}$
& $\underline{\mathbf{1.000{\scriptstyle \pm 0.000}}}$
& $0.890{\scriptstyle \pm 0.226}$
& $0.490{\scriptstyle \pm 0.008}$ \\

PCD-DPP~\citep{luan2026projected}
& $3.530{\scriptstyle \pm 1.236}$
& $0.503{\scriptstyle \pm 0.160}$
& $\underline{\mathbf{1.000{\scriptstyle \pm 0.000}}}$
& $0.898{\scriptstyle \pm 0.224}$
& $0.492{\scriptstyle \pm 0.008}$ \\

\midrule

DiRecT-LB
& $4.046{\scriptstyle \pm 1.139}$
& $0.581{\scriptstyle \pm 0.146}$
& $\underline{\mathbf{1.000{\scriptstyle \pm 0.000}}}$
& $\underline{0.924{\scriptstyle \pm 0.190}}$
& $0.373{\scriptstyle \pm 0.008}$ \\

DiRecT-DPP
& $\underline{4.226{\scriptstyle \pm 1.126}}$
& $\underline{0.609{\scriptstyle \pm 0.142}}$
& $\underline{\mathbf{1.000{\scriptstyle \pm 0.000}}}$
& $0.913{\scriptstyle \pm 0.207}$
& $0.372{\scriptstyle \pm 0.008}$ \\
\bottomrule
\end{tabular}
\end{adjustbox}
\end{figure}

%% file: appendix.tex

\section{Related Work}
\label{sec:Related works - appendix}

\paragraph{Safe diffusion-based planning.}
Existing approaches to diffusion-based planning with test-time constraint enforcement can largely be understood through the lens of projection, differing primarily in \emph{when} projection is applied and how strongly it restricts the denoising process. A simple strategy is \emph{post-sampling} projection, where an unconstrained diffusion model first generates a trajectory and feasibility is enforced only afterward. This direction was explored by~\citet{power2023sampling}. While conceptually straightforward, post-sampling projection can move samples substantially away from the pretrained diffusion distribution, as observed by~\citet{christopher2024constrained}.
Most subsequent methods instead perform projection \emph{during} sampling. For example, \citet{christopher2024constrained} formulate constrained sampling as a maximum-likelihood problem in which all intermediate latents are required to lie in the feasible set, and sample using projected stochastic Langevin dynamics. Building on this per-step projection perspective, \citet{luan2026projected} extend projected diffusion models to multi-robot motion planning with coupled trajectory generation and additional soft costs, while \citet{liang2025simultaneous} introduce a customized Lagrangian-based projector that relaxes nonconvex collision constraints.
Recognizing that enforcing constraints throughout the entire denoising process can be overly restrictive, more recent methods relax early-stage sampling and tighten constraint enforcement near the end of generation. \citet{zhang2025constraineddiffuserssafeplanning} propose a Lagrangian formulation for safe planning, where constraint strictness is progressively increased through primal-dual or augmented-Lagrangian updates. In a related direction, SafeDiffuser~\citep{xiao2025safediffuser} views denoising as a controlled dynamical system and incorporates control barrier functions (CBFs) into the sampling process. This allows constraint violations during early denoising while enforcing terminal feasibility through the CBF condition. However, balancing early-stage relaxation with reliable terminal feasibility can require careful parameter tuning; moreover, even relaxed variants still impose restrictions along the denoising trajectory.
Overall, most existing methods share a common feature: constraint satisfaction is enforced through projections or constraint-driven updates along the denoising path. In contrast, our method adopts a stochastic optimal control perspective that imposes hard constraints only on the terminal clean sample, rather than requiring feasibility of noisy intermediate iterates. It instead steers the sampling process toward feasibility of the final trajectory, reducing unnecessary restrictions on the learned diffusion dynamics. Moreover, beyond hard constraints, practical planning problems often involve soft costs or rewards. While projection-based approaches can incorporate such objectives in specific implementations, they do not naturally provide a unified framework for jointly handling hard feasibility and soft optimality. Our stochastic optimal control formulation addresses both within a single inference-time procedure, enabling efficient and flexible constrained sampling.

\paragraph{Diffusion models and stochastic optimal control.}
A growing body of work has connected diffusion sampling with stochastic optimal control (SOC), viewing the reverse diffusion process as a controlled stochastic dynamics whose drift can be modified while penalizing deviations from a reference process. One line of work uses this perspective for sampling from unnormalized target densities, where a reference diffusion process is controlled to match a desired distribution; representative examples include path-integral sampling~\citep{kappen2005pathintegrals}, denoising diffusion samplers~\citep{vargas2023denoising}, and controlled Monte Carlo diffusions~\citep{vargas2024transport}. Another line of work applies SOC to reward optimization in generative models, either through training-based reward fine-tuning, such as Adjoint Matching~\citep{domingo2024adjoint}, or through training-free guidance toward soft objectives, such as RB-Modulation~\citep{rout2025rbmodulation} and variational diffusion guidance~\citep{pandey2025variational}.
These methods demonstrate the usefulness of the SOC viewpoint for controlling diffusion samplers, but they primarily target distributional sampling or soft reward optimization. They do not directly address hard-constrained planning, where the final trajectory must satisfy feasibility constraints while remaining faithful to the learned diffusion dynamics. Extending SOC-based diffusion control to this setting is nontrivial, and existing reward-guidance methods do not yield an efficient constrained sampler. In contrast, our work develops a tractable surrogate tailored to terminally constrained sampling for diffusion-based planning, enabling efficient inference-time enforcement of hard constraints while also accommodating additional soft costs or rewards.

\paragraph{Training-based diffusion planners.}
A parallel and complementary line of work adapts diffusion-based planners through additional training or fine-tuning. For example, AdaptDiffuser~\citep{liang2023adaptdiffuser} improves diffusion planners via self-evolution, using reward-guided synthetic trajectories to update the model. SODP~\citep{fan2024task} pretrains a task-agnostic diffusion planner on diverse suboptimal data and then fine-tunes it with task-specific rewards. DPPO~\citep{ren2024diffusion} further adapts diffusion policies for continuous control and robot learning using policy-gradient reinforcement learning. These methods improve planning performance by modifying model parameters, whereas our work enforces hard constraints at inference time without retraining, making it orthogonal to training-based adaptation.

\section{Stochastic optimal control and constrained sampling}
\label{sec:stochastic optimal control - Appendix}
We expand on Section \ref{sec:method} and present the connection between stochastic optimal control and safe planning with \emph{test-time} constraints. We recall central ideas from optimal control literature~\citep{fleming2012deterministic} and motivate grounding constrained planning in this line of work.

\paragraph{Stochastic Optimal Control.} For simplicity and coherence with the literature, we consider a \emph{forward-time} Itô process between $0 \le t \le 1$, transporting probability mass from prior distribution $p_0$ toward $p_{base}$. The diffusion-centric formulation in \eqref{eq:soc problem} is equivalent up to time reversal. We parametrize the SDE by a given reference drift $f: \R^d \times [0, 1] \rightarrow \R^d$ and diffusion coefficient $g: [0, 1] \rightarrow \R$:
\begin{equation}
    d X_t = f_t(X_t) dt + g(t) d W_t, \quad X_0 \sim p_0, \quad t \in [0, 1]   
\label{eq:sde-uncontrolled}
\end{equation}
where $\left(W_t\right)_{t \ge 0}$ denotes a standard Wiener process. To control generation, we introduce an additive control drift $u: \R^d \times [0, 1] \rightarrow \mathcal{U}$ producing actions in the admissible set of control inputs $\mathcal{U}$. As optimization objectives, let $C: \R^d \rightarrow \R$ denote the terminal cost, and $\ell: \R^d \times \mathcal{U} \times [0, 1] \rightarrow \R$ be the running cost. Stochastic optimal control formulates the minimization problem of the expected cost functional $J$ over the path measure $\mathbb{P}^u$ induced by the controlled dynamics:
\begin{equation}
\label{eq:soc_problem}
\begin{aligned}
    \min_{u \in \mathcal{A}} \quad J(u) = \quad &
    \EE\left[
        C(X^u_1)
        +
        \int_0^1
        \ell\left(X^u_t, u_t(X^u_t), t\right)\,dt
    \right] \\
    & \text{s.t.}\quad
    dX^u_t
    =
    \left[
        f_t(X^u_t)
        +
        g(t)u_t(X^u_t)
    \right]dt
    +
    g(t)dW_t, \\
    & \qquad X^u_0 \sim p_0, \quad t \in [0,1].
\end{aligned}
\end{equation}

where $\mathcal{A}$ is the admissible class of progressively measurable controls satisfying $u_t(X_t^u)\in\mathcal{U}$ and $\EE_{X^u \sim \mathbb{P}^u} \left[ \int_0^1 \|u_t(X_t^u)\|^2dt \right] < \infty$. Under standard coercivity, convexity, and continuity assumptions on the data that ensure existence of optimal controls for controlled It\^o diffusions~\citep{haussmann1990existence}, we denote by $u^\star \in \arg\min_{u\in\mathcal{A}} J(u)$ an optimal solution of \eqref{eq:soc_problem}. 

\paragraph{Solving the SOC and the Hamilton--Jacobi--Bellman equation.}
Solving \eqref{eq:soc_problem} \emph{as-is} is generally intractable due to the stochastic dynamics and the expectation over all realizable paths. We present a classical dynamic programming formulation that transforms the functional minimization problem into a partial differential equation (PDE). First, for a given admissible control, define the \emph{cost-to-go} $J:\mathcal{A}\times\R^d\times[0,1]\to\R$ by
\begin{equation}
    J(u;x,t)
    =
    \EE_{\mathbb{P}^{u}_{t,x}}
    \left[
        \int_t^1
        \ell\left(X_s^u,u_s(X_s^u),s\right)\,ds
        +
        C(X_1^u)
    \right],
\label{eq:cost_function}
\end{equation}
where $\mathbb{P}^{u}_{t,x}$ denotes the path measure induced by the controlled dynamics initialized at $X_t^u=x$. The cost-to-go coincides with the expected future cost incurred by applying $u$ from state $x$ at time $t$. We then define the \emph{value function} $V:\R^d\times[0,1]\to\R$ as the infimum of this cost over all admissible controls:
\begin{equation}
    V(x,t)
    =
    \inf_{u\in\mathcal{A}} J(u;x,t).
\label{eq:value_function}
\end{equation}

The SOC and value formulations are related as taking expectations over the initial distribution yields:
\begin{equation}
J(u) = \EE_{X_0 \sim p_0} \left[ J(u; X_0, 0) \right], \\
\end{equation}
\begin{equation}
J(u^\star) = \EE_{X_0 \sim p_0}\left[ V(X_0, 0) \right].
\end{equation}
A solution to the stochastic optimal control problem can be achieved by computing the value function, which, under regularity conditions, satisfies the Hamilton--Jacobi--Bellman (HJB)~\citep{bellman1954theory} partial differential equation (Theorem \ref{thm:hjb_soc}).

\begin{theorem}[Hamilton--Jacobi--Bellman equation~\citep{fleming2012deterministic}]
\label{thm:hjb_soc}
Assume the standard conditions for the dynamic programming principle hold, and
suppose that the value function satisfies $V\in C^{1,2}(\R^d\times[0,1])$.
Define the stochastic Hamiltonian
\begin{equation}
\label{eq:soc_hamiltonian}
    \mathcal{H}(x,p,M,u,t)
    :=
    \ell(x,u,t)
    +
    \left(f_t(x)+g(t)u\right)^\top p
    +
    \frac{1}{2}g(t)^2\Tr(M).
\end{equation}
Then $V$ solves the Hamilton--Jacobi--Bellman equation
\begin{equation}
\label{eq:hjb_soc}
    -\partial_t V(x,t)
    =
    \min_{u\in\mathcal{U}}
    \mathcal{H}
    \left(
        x,
        \nabla_x V(x,t),
        \nabla_x^2 V(x,t),
        u,
        t
    \right), \quad V(x, 1) = C(x).
\end{equation} \\
Moreover, whenever the minimum is attained, an optimal feedback control
satisfies
\begin{equation}
\label{eq:hjb_optimal_control}
    u^\star(x,t)
    \in
    \arg\min_{u\in\mathcal{U}}
    \mathcal{H}
    \left(
        x,
        \nabla_x V(x,t),
        \nabla_x^2 V(x,t),
        u,
        t
    \right).
\end{equation}
\begin{proof}
The result follows from Bellman's dynamic programming principle~\citep{bellman1954theory}. 
Applying the principle over $[t,t+\Delta t]$, expanding the value function with It\^o's formula, taking expectations, and sending $\Delta t\to0$ leads to the differential HJB formulation. The terminal condition follows from the definition of the cost-to-go. 
See~\citet{fleming2012deterministic}.
\end{proof}
\end{theorem}

In practice, solving the HJB PDE exactly is rarely tractable beyond low-dimensional systems, because grid-based dynamic programming suffers from the curse of dimensionality; since our focus is on \emph{training-free} guidance rather than learning value functions, we refer the reader to training-based approaches to HJB equations and stochastic optimal control, including approximate dynamic programming, reinforcement learning, and neural PDE methods~\citep{bertsekas1996neuro,sutton2018reinforcement,han2018solving,sirignano2018dgm}.

\paragraph{Relation to KL control.} We now specialize the stochastic optimal control formulation by choosing the quadratic running cost
\begin{equation}
\label{eq:running-cost}
    \ell(x,u,t)=\frac{1}{2}\norm{u}^2.
\end{equation}
This choice gives the control objective a path-space relative-entropy interpretation, connecting stochastic optimal control to Schr\"odinger Bridge problems, path-integral control, and KL-control as outlined in the following theorem:

\begin{theorem}[Path-space KL representation]
\label{thm:path-space-kl-representation}
Let $\{X_t^0\}_{t\in[0,1]}$ be the reference stochastic interpolant
\begin{equation}
\label{eq:reference-stochastic-interpolant}
    dX_t^0=f_t(X_t^0)\,dt+g(t)dW_t,
    \qquad
    X_0^0\sim p_0,
\end{equation}
with induced path measure $\mathbb{P}^0$. Let $\mathbb{P}^u$ be the path measure induced by the controlled stochastic interpolant
\begin{equation}
\label{eq:controlled-stochastic-interpolant}
    dX_t^u=
    \left[f_t(X_t^u)+g(t)u_t(X_t^u)\right]dt
    +g(t)dW_t,
    \qquad
    X_0^u\sim p_0.
\end{equation}
Then, for progressively measurable controls satisfying the standard Girsanov integrability condition, the SOC objective with $\ell(x,u,t)=\frac{1}{2}\norm{u}^2$ is equivalent to the KL-regularized path-space objective
\begin{equation}
\label{eq:kl-terminal-cost-problem}
    \inf_{u\in\mathcal{A}}
    \left\{
        \mathrm{KL}\left(\mathbb{P}^u\,\|\,\mathbb{P}^0\right)
        +
        \EE_{\mathbb{P}^u}\left[C(X_1^u)\right]
    \right\}.
\end{equation}
\end{theorem}
\begin{proof}
By Girsanov's theorem, the controlled drift perturbation $g(t)u_t$ induces an absolutely continuous change of path measure whose Radon--Nikodym derivative yields
\[
    \mathrm{KL}\left(\mathbb{P}^u\,\|\,\mathbb{P}^0\right)
    =
    \EE_{\mathbb{P}^u}
    \left[
        \int_0^1
        \frac{1}{2}\norm{u_t(X_t^u)}^2dt
    \right].
\]
Substituting this identity into the SOC objective with $\ell(x,u,t)=\frac{1}{2}\norm{u}^2$ yields \eqref{eq:kl-terminal-cost-problem}. See~\citet{theodorou2012relative} for the corresponding relative-entropy derivation in path-integral and KL-control stochastic optimal control.
\end{proof}

Theorem~\ref{thm:path-space-kl-representation} shows that the quadratic-control SOC objective selects a guided sampler that reduces terminal cost while remaining close to the reference stochastic interpolant in path-space KL. This principle of remaining maximally close to the reference sampler, and hence avoiding unnecessary over-constraint of the latent dynamics, offers a useful lens for safe planning with pretrained diffusion models. Although the full stochastic optimal control problem is generally intractable, we emphasize that this distributional viewpoint provides a principled idealization from which reliable and efficient training-free surrogate approximations may be derived.

\paragraph{Constrained sampling.} As noted in prior work on constrained sampling~\citep{zhang2025constraineddiffuserssafeplanning,chamon2025constrainedsamplingprimalduallangevin}, the stochastic optimal control formulation naturally generalizes to planning with \emph{hard} constraints. In particular, let $\mathcal{S}\subseteq\R^d$ denote a feasible set and define the extended-valued indicator
\begin{equation}
\label{eq:indicator-function}
    \iota_{\mathcal{S}}(x)
    =
    \begin{cases}
        0, & x\in\mathcal{S},\\
        +\infty, & x\notin\mathcal{S}.
    \end{cases}
\end{equation}
Then hard-constrained sampling can be written by choosing the terminal cost
$C'(x)=C(x) + \iota_{\mathcal{S}}(x)$, yielding the constrained SOC problem
\begin{equation}
\label{eq:hard-constrained-soc}
\begin{aligned}
    \min_{u\in\mathcal{A}}
    \quad &
    \EE\left[
        \int_0^1
        \frac{1}{2}\norm{u_t(X_t^u)}^2\,dt
        +
        C'(X_1^u)
    \right] \\
    \text{s.t.}\quad
    &
    dX_t^u
    =
    \left[
        f_t(X_t^u)
        +
        g(t)u_t(X_t^u)
    \right]dt
    +
    g(t)dW_t,
    \qquad
    X_0^u\sim p_0 .
\end{aligned}
\end{equation}

Now let $C'(x)=C(x)+\iota_{\mathcal{S}}(x)$. Since
\[
    \EE[C'(X_1)]
    =
    \EE[C(X_1)]
    +
    \EE[\iota_{\mathcal{S}}(X_1)],
\]
any finite-cost solution must satisfy $X_1\in\mathcal{S}$ almost surely. Therefore, the extended-valued terminal cost is equivalently represented as an explicit terminal constraint, yielding
\begin{equation}
\label{eq:reverse-hard-constrained-soc}
\begin{aligned}
    \min_{u\in\mathcal{A}}
    \quad &
    \EE_{\mathbb{P}^u}
    \left[
        C(X_1)
        +
        \int_0^1
        \frac{1}{2}\norm{u_t(X_t)}^2\,dt
    \right] \\
    \mathrm{s.t.}\quad
    &
    dX_t^u
    =
    \left[
        f_t(X_t^u)
        +
        g(t)u_t(X_t^u)
    \right]dt
    +
    g(t)dW_t,
    \qquad
    X_0^u\sim p_0 \\
    &
    X_1\in\mathcal{S}
    \qquad
    \mathbb{P}^u\text{-a.s.}
\end{aligned}
\end{equation}

\section{Algorithm derivation}
\label{sec:algorithm derivation - Appendix}
In this section, we provide the full derivation of the algorithm, highlighting all the approximations employed to reach a computationally tractable formulation. For a self-contained derivation, we start by repeating the terminally constrained stochastic optimal control problem.

\begin{problem*}[Continuous-time stochastic optimal control problem with terminal constraints]
\label{prob:soc problem recall}
\textit{Given reference dynamics, constraint set $\mathcal{S}$, terminal cost $C$, and cost weight $\lambda$, solve for the optimal control drift $u^\star$:}
\par\noindent
\begin{equation}
\label{eq:soc problem continuous}
\begin{aligned}
\min_{u}
\quad
& \EE \left[\lambda \, C(X^u_0) + \frac{1}{2} \int_0^1 \norm{u_t(X^u_t)}_2^2 \,dt \right] \\
\text{s.t.}\quad
& d X^u_t = \left[ \tilde{f}_t^\theta(X^u_t) + g(t) u_t(X^u_t) \right]dt + g(t)d\bar{W}_t, \quad X_1 \sim \pprior \ \\
& X^u_0 \in \mathcal S
\qquad
\mathbb{P}^u\text{-a.s.}
\end{aligned}
\end{equation}
\end{problem*}
where the reverse drift is given by the learned score function $ \tilde{f_t}^\theta (X_t) = f_t(X_t) - g(t)^2 s^\theta_t(X_t)$. 

For the numerical integration of the reverse-time dynamics, let $N \in \N$ denote the number of discretization steps, and let $0 = t_0 < t_1 < ... < t_{N-1} < t_{N} = 1$ be the time grid used for sampling from the diffusion model. Since diffusion sampling proceeds backward in time, we define the uncontrolled reverse update from \(t_i\) to \(t_{i-1}\) by
\[
    X_{i-1} = \Phi_i^\theta(X_i,\varepsilon_i),
    \qquad i = 1,\dots,N,
\]
where $ \Phi_i^\theta : \mathbb{R}^d \times \mathbb{R}^d \to \mathbb{R}^d $ is the \(i\)-th denoising transition map, and $\varepsilon_i \sim \mathcal{N}(0,I_d) $ are independent across \(i\). The maps \(\Phi_i^\theta\) are left abstract in order to accommodate different numerical sampling schemes. Moreover, as is commonly done to reduce discretization-induced noise in the
final sample, we assume that the last denoising step is deterministic and
returns the model's data prediction without injecting additional noise. That is,
\[
   X_0 = \Phi^\theta_1(X_1)
       = \hat{x}^\theta_0(X_1,t_1).
\]

For the SOC surrogate, we approximate the controlled dynamics by applying an Euler discretization to the control input over each interval \([t_{i-1},t_i]\). This yields
\begin{equation}
\label{eq:controlled-discrete-dynamics}
    X^u_{i-1}
    =
    \Phi^\theta_i(X^u_i,\varepsilon_i)
    +
    g_i u_i(X^u_i)\Delta t_i,
    \qquad i=1,\dots,N,
\end{equation}
where \(g_i = g(t_i)\), \(\Delta t_i = t_i - t_{i-1}\), and \(\varepsilon_i \sim \mathcal{N}(0,I_d)\). After discretizing the control energy, we obtain the discrete-time formulation of the terminally constrained stochastic optimal control problem, stated in Problem~\ref{prob:soc problem discrete}.

\begin{problem}[Discrete-time stochastic optimal control problem with terminal constraints]
\label{prob:soc problem discrete}
\textit{Given reference drift $\tilde{f}^\theta_t$, constraint set $\mathcal{S}$, and cost weight $\lambda$, solve for the control functions $\{u_i^\star\}_{i=1}^N$:}
\par\noindent
\begin{equation}
\label{eq:soc problem discrete}
\begin{aligned}
\min_{\{u_i\}^N_{i=1}}
\quad
& \EE \left[\lambda \, C(X^u_0) + \frac{1}{2} \sum_{i=1}^N \norm{u_i(X^u_i)}_2^2 \Delta t_i \right] \\
\text{s.t.}\quad
& X^u_{i-1} = \Phi^\theta_i(X^u_i, \eps_i) + g_i\,  u_i(X^u_i) \, \Delta t_i, \quad \eps_i \sim \Normal\left(0, I_d\right), \quad \forall i, \; 1 \le i \le N, \\
& X^u_0 \in \mathcal S
\qquad
\mathbb{P}_d^u\text{-a.s.}
\end{aligned}
\end{equation}
\end{problem}
where $\mathbb{P}_d^u$ is the probability measure over state tuples induced by the controlled dynamics.

Although the preceding discretization yields a finite-horizon control problem, the optimization problem remains largely intractable in its full generality. First, the control variables are defined over the full state space \(\R^d\), while the objective involves an expectation over all realizations of the controlled stochastic dynamics. Moreover, terminal costs and hard terminal constraints induce temporal coupling across the full reverse process: the effect of an early control perturbation on the terminal state is mediated by a long composition of neural-network denoising maps. 

Rearranging the controlled dynamics, we can express the control action at each step in terms of the deviation from the uncontrolled denoising update.\footnote{We also simplify the notation and suppress the superscript \(u\) and write \(X_i\)
instead of \(X_i^u\) for controlled states.} Namely,
\begin{equation}
\label{eq:state-action-relation}
    u_i = \frac{X_{i-1} - \bar{X}^{\varepsilon_i}_{i-1}}{g_i \Delta t_i},
    \qquad
    \bar{X}^{\varepsilon_i}_{i-1} := \Phi^\theta_i(X_i,\varepsilon_i),
    \qquad i=1,\dots,N.
\end{equation}

Second, motivated by model predictive control~\citep{rawlings2017modelpredictive}, we replace the full-horizon optimization by a one-step look-ahead surrogate. Specifically, at step \(i\), the transition from \(X_i\) to \(X_{i-1}\) is evaluated using the original denoising update, while the cost-to-go over the remaining reverse process is approximated using the model's predicted clean sample. This yields the proxy
\begin{equation}
    \mathbb{E}\!\left[C(X_0) \mid X_{i-1} \right]
    \approx
    C\!\left(\mathbb{E}[X_0 \mid X_{i-1}]\right)
    \approx
    C\!\left(\hat{x}_0^\theta(X_{i-1},t_{i-1})\right),
    \qquad i=1,\dots,N.
\end{equation}

The total discrete-time SOC cost-to-go from step \(i\) is therefore approximated by retaining only the one-step control effort and replacing the remaining terminal cost by the Tweedie-based prediction. Namely, we use the surrogate
\begin{equation}
\label{eq:total-cost-estimation}
\mathbb{E}\left[\lambda C(X_0) + \frac{1}{2}\sum_{k=1}^{i} \|u_k(X_k)\|_2^2 \Delta t_k \,\middle|\, X_i \right] \approx\lambda C\!\left(\hat{x}_0^\theta(X_{i-1},t_{i-1})\right) + \frac{1}{2g_i^2\Delta t_i} \norm{X_{i-1} - \bar{X}^{\varepsilon_i}_{i-1}}_2^2 .
\end{equation}

Similarly, we enforce constraint satisfaction on the predicted clean sample \(\hat{x}_0^\theta(X_{i-1},t_{i-1})\) as a proxy for terminal feasibility, thereby replacing the full joint optimization over all control variables with a sequence of one-step subproblems. At each denoising step, we first compute \(\bar{X}^{\varepsilon_i}_{i-1}=\Phi_i^\theta(X_i,\varepsilon_i)\), with \(\varepsilon_i\sim\mathcal{N}(0,I_d)\), and correct the uncontrolled estimate by solving:
\begin{equation}
\label{eq:inner-subproblem}
\begin{aligned}
X^\star_{i-1} \in \operatorname*{arg\,min}_{X_{i-1}} \quad
& \lambda C(\hat{X}_{0|i-1})
+ \frac{1}{2g_i^2\Delta t_i}
\|X_{i-1}-\bar{X}^{\varepsilon_i}_{i-1}\|_2^2 \\
\text{s.t.} \quad
& \hat{X}_{0|i-1}\in\mathcal{S}, \qquad
\hat{X}_{0|i-1}=\hat{x}_0^\theta(X_{i-1},t_{i-1}) .
\end{aligned}
\end{equation}

Each one-step subproblem still contains a nontrivial coupling between the optimization variable \(X_{i-1}\) and the predicted clean sample \(\hat{X}_{0|i-1}:=\hat{x}_0^\theta(X_{i-1},t_{i-1})\). This coupling enters both the terminal cost and the feasibility constraint \(\hat{X}_{0|i-1}\in\mathcal{S}\). Therefore, even when \(C\) and \(\mathcal{S}\) have simple structure in the data domain, their pullback through the neural denoising map \(\hat{x}_0^\theta(\cdot,t_{i-1})\) may define a highly nonlinear and nonconvex optimization problem in the latent variable \(X_{i-1}\). For diffusion policies with tentatively millions of parameters, repeatedly embedding neural-network evaluations within the subproblem can introduce substantial computational overhead and may limit the practicality of direct latent-space optimization with generic off-the-shelf solvers.
To phrase the optimization in the data domain, we approximate the latent displacement \(X_{i-1}-\bar{X}^{\varepsilon_i}_{i-1}\) through the corresponding difference between predicted clean samples. The key step is to approximately invert Tweedie's formula by a fixed-point scheme. Rearranging \eqref{eq:tweedie formula} at time \(t_{i-1}\) gives
\begin{equation}
\label{eq:tweedie-rewritten}
    X_{i-1} = \alpha_{t_{i-1}}\hat{x}_0^\theta(X_{i-1},t_{i-1}) - \sigma_{t_{i-1}}^2 s_{t_{i-1}}^\theta(X_{i-1}).
\end{equation}
Now fix \(y=\hat{x}_0^\theta(X_{i-1},t_{i-1})\). Inverting Tweedie's estimate for this prescribed clean prediction amounts to finding \(X_{i-1}\) such that
\begin{equation}
\label{eq:fixed-point-map}
    X_{i-1}=F_{t_{i-1}}^\theta(X_{i-1};y), \qquad F_{t_{i-1}}^\theta(x;y):=\alpha_{t_{i-1}}y-\sigma_{t_{i-1}}^2s_{t_{i-1}}^\theta(x).
\end{equation}

Assuming \(F_{t_{i-1}}^\theta(\cdot;y)\) is a contraction, the Banach fixed-point theorem guarantees a unique fixed point, which can be obtained by repeated application of \(F_{t_{i-1}}^\theta(\cdot;y)\):
\begin{equation}
\label{eq:fixed-point-limit}
    X^\star_{t_{i-1}}(y) = \lim_{k\to\infty} \left(F_{t_{i-1}}^\theta(\cdot;y)\right)^{(k)}(X^{(0)}).
\end{equation}

In practice, we initialize the fixed-point inversion at the uncontrolled proposal \(\bar{X}^{\varepsilon_i}_{i-1}\), which provides a natural estimate of the latent state before control is applied. Applying one fixed-point iteration gives
\begin{equation}
\label{eq:one-fixed-point-iterate}
    X_{i-1} \approx F^\theta_{t_{i-1}}(\bar{X}^{\varepsilon_i}_{i-1};y)
    = \alpha_{t_{i-1}}y - \sigma_{t_{i-1}}^2 s^\theta_{t_{i-1}}(\bar{X}^{\varepsilon_i}_{i-1}).
\end{equation}
Combining this approximation with \eqref{eq:tweedie-rewritten} applied to
\(\bar{X}^{\varepsilon_i}_{i-1}\) yields
\begin{equation}
\label{eq:state-differences}
\begin{aligned}
X_{i-1}-\bar{X}^{\varepsilon_i}_{i-1}
&\approx
\alpha_{t_{i-1}}
\left(y-\hat{x}_0^\theta(\bar{X}^{\varepsilon_i}_{i-1},t_{i-1}) \right) \\
&=
\alpha_{t_{i-1}}
\left(\hat{x}_0^\theta(X_{i-1},t_{i-1})-\hat{x}_0^\theta(\bar{X}^{\varepsilon_i}_{i-1},t_{i-1})\right).
\end{aligned}
\end{equation}

Thus, the latent displacement can be approximated by a scaled discrepancy between Tweedie clean-sample predictions. This provides a gradient-free correction direction in latent space and allows the control penalty to be expressed in the data domain. Substituting \eqref{eq:state-differences} into the one-step surrogate yields the final receding-horizon formulation in Problem~\ref{prob:final-problem}, where the constrained optimization is carried out in the data domain rather than through repeated neural-network evaluations inside the solver.

\begin{problem}[DiRecT: receding-horizon formulation]
\label{prob:final-problem}
\textit{Given a sampling scheme \(\{\Phi_i^\theta\}_{i=1}^N\), constraint set
\(\mathcal{S}\), cost weight \(\lambda\), and initial sample
\(X_N^\star\sim p_{\mathrm{prior}}\), solve the following recursion for
\(i=N,\ldots,1\):}
\begin{equation}
\label{eq:final-problem}
\left\{
\begin{aligned}
\bar{X}^{\varepsilon_i}_{i-1}
&= \Phi_i^\theta(X_i^\star,\varepsilon_i), \qquad
\varepsilon_i\sim\mathcal{N}(0,I_d), \\
\tilde{X}_{0|i-1}
&= \hat{x}_0^\theta(\bar{X}^{\varepsilon_i}_{i-1},t_{i-1}), \\
\hat{X}^\star_{0|i-1}
&\in \operatorname*{arg\,min}_{\hat{X}_{0|i-1}}
\lambda \, C(\hat{X}_{0|i-1})
+
\frac{\alpha_{t_{i-1}}^2}{2g_i^2\Delta t_i}
\|\hat{X}_{0|i-1}-\tilde{X}_{0|i-1}\|_2^2
\quad \text{s.t.}\quad
\hat{X}_{0|i-1}\in\mathcal{S}, \\
X^\star_{i-1}
&=
\begin{cases}
\bar{X}^{\varepsilon_i}_{i-1}
+
\alpha_{t_{i-1}}
(\hat{X}^\star_{0|i-1}-\tilde{X}_{0|i-1}),
& i>1, \\
\hat{X}^\star_{0|0},
& i=1 .
\end{cases}
\end{aligned}
\right.
\end{equation}
\end{problem}

Although several approximations are introduced to reduce the general stochastic optimal control problem to a tractable form, terminal constraint satisfaction is not relaxed, as formalized in Proposition~\ref{prop:terminal-feasibility}.


\begin{remark}[Terminal convention]
At the final denoising step, we return the optimized data-domain prediction \(\hat{X}^\star_{0|0}\) directly, rather than applying the latent correction formula. This convention ensures exact terminal feasibility whenever the final subproblem is solved feasibly. It is also convenient in practical implementations, since common noise schedules, such as cosine schedules~\citep{nichol2021improved}, need not satisfy \(\alpha_{t_0}=1\) exactly.
\end{remark}

\section{Implementation details}
\label{sec:experimental details - Appendix}
\paragraph{Hardware.} All \texttt{Maze2D} and \texttt{D3IL} experiments were performed on a compute node equipped with two Intel(R) Xeon(R) Gold 5218 @ 2.30GHz CPUs and six NVIDIA Quadro RTX 8000 48GB, \texttt{MRMP} on a node with two Intel Xeon E5-2670 v2 @ 2.50 GHz CPUs and one NVIDIA Tesla V100-SXM2 32 GB GPU, and \texttt{PushT} on a machine equipped with two AMD EPYC 75F3 32-Core Processor CPUs and four NVIDIA A100 GPUs. All computations for training, model evaluation, and gradient calculation were performed on a single GPU, whereas environment rollouts and IPOPT optimizations were performed on a CPU.
\paragraph{Software.} All experiments are based on PyTorch~\citep{paszke2017automatic} for automatic differentiation and JAX~\citep{jax2018github} for our custom multi-agent path-finding optimizer. Moreover, \texttt{MRMP} and \texttt{PushT} are based on \textsc{Projected Coupled Diffusion}~\citep{luan2026projected}~\footnote{\url{https://github.com/EdmundLuan/pcd}} (MIT), which are themselves extensions of \textsc{Diffusion Policy}~\citep{chi2025diffusion}~\footnote{\url{https://github.com/real-stanford/diffusion_policy}}(MIT), \textsc{MMD}~\citep{shaoul2025multirobot}~\footnote{\url{https://github.com/yoraish/mmd}}(MIT), and \textsc{LTLDoG}~\citep{feng2024ltldog}~\footnote{\url{https://github.com/clear-nus/ltldog}}(MIT). We use \textsc{CasADi}~\citep{Andersson2019} (GNU LGPL v3.0) and \textsc{IPOPT}~\citep{wachter2006implementation} (EPL-2.0) for \emph{off-the-shelf} optimization.
\subsection{Safe maze navigation}
\label{sec:constrained navigation in Maze2d - Appendix details}
\paragraph{Training.} We train a \emph{continuous-time} diffusion model with a standard temporal UNet architecture~\citep{jia2024towards} and horizon $H = 384$. The model outputs the state-action pair trajectory \((s_0, a_0, s_1, a_1, \ldots, s_{H-1}, a_{H-1}) \in \R^{6H}\) conditioned on given start and end states. We employ data-prediction parametrization and a cosine noise schedule with offset \(s=0.008\)~\citep{nichol2021improved}. We train the model with Adam~\citep{kingma2014adam} for $500{,}000$ steps with a batch size of $256$ on the D4RL \texttt{maze2d-large-v1}~\citep{fu2020d4rl}~\footnote{\url{https://github.com/Farama-Foundation/D4RL}} dataset (CC BY 4.0), scaled by min-max normalization. We use a learning rate of $0.001$ with a cosine annealing schedule for the first $10{,}000$ steps. We save checkpoints using an exponential moving average of the model weights with decay $0.995$.
\paragraph{Environment details.} We introduce obstacles as ellipses and super-ellipses centered in $(c_x, c_y)$, semi-axes $(r_x, r_y)$, and order $p$:
\begin{equation}
    \left| \frac{x - c_x}{r_x} \right|^p + \left| \frac{y - c_y}{r_y} \right|^p \ge 1
\label{eq:maze2d obstacle definition}
\end{equation}
For both \texttt{Broad} and \texttt{Narrow} variants, we follow the obstacle placement in \citep{xiao2025safediffuser}. To enforce dynamics constraints we fit by ordinary least-squares regression the normalized transitions from the training data for a total of 3,993,503 across 1,062 episodes. For simplicity, although equivalent, we separate the state and action coefficients with the following functional form $s' = A \cdot s + B \cdot a + c$, where each transition is defined by the tuple $(s, a, s')$. When enforcing dynamic constraints, we impose the fitted dynamics as \emph{strict} equality constraints over the planning horizon. Although the per-step least squares residuals are contained, open-loop execution over the full horizon remains impractical. In this regime, predicted and executed trajectories may diverge, hindering proper evaluation of the planner. To mitigate this issue, we perform policy rollout by tracking the generated trajectory with a proportional-derivative (PD) controller with constants $P=5, D=1$. 
\paragraph{Evaluation.} We evaluate each method over $100$ i.i.d. generated samples conditioned on the same fixed endpoints. For each sample, we execute a rollout up to $800$ environment steps, at which the episode is terminated. For all methods we denoise with a DDPM sampler~\citep{ho2020denoisingdiffusionprobabilisticmodels} on a uniform time discretization. We follow standard practice and perform the last denoising iteration by returning the model prediction. To encourage shorter paths, when applicable, the cost function $C(x)$ is chosen to be the squared path length of the generated trajectory. We now briefly describe the implementation details of each method, as well as their adaptation to include equality dynamic constraints:
\begin{itemize}
    \item \textbf{Diffuser}~\citep{janner2022planningdiffusionflexiblebehavior}. We employ $32$ denoising iterations as an unguided reference for the base pretrained model. No obstacles or dynamic constraints are imposed.
    \item \textbf{Gradient Guidance}~\citep{chung2023diffusionposterior}. We employ $32$ denoising iterations and guide the noisy latent $x_t$ by computing the gradient of the cost function using posterior sampling as $\nabla C'(\hat{x}_{0,\theta}(x_t))$, where constraint violations are added as a penalty term in the cost function $C'(x) = C(x) + C_{\mathrm{penalty}}(x)$. To impose dynamic constraints, we penalize the squared residuals of each predicted transition along the prediction horizon by adding them as an additional cost penalty.
    \item \textbf{Projection}~\citep{christopher2024constrained}. We employ $32$ denoising steps and project the second half of the generated latents onto the feasible set. Dynamic constraints are introduced directly in the optimization process, which is solved with IPOPT.
    \item \textbf{Augmented Lagrangian}~\citep{zhang2025constraineddiffuserssafeplanning}. We adapt the original implementation to the continuous-time case by scaling hyperparameters to generate equivalent guided transitions. We use $256$ sampling steps with an additional $200$ iterations for the last denoising step. We use the original implementation for handling dynamic equality constraints.
    \item \textbf{SafeDiffuser-RoS}~\citep{xiao2025safediffuser}. We adapt the original implementation to the continuous-time case by scaling hyperparameters to generate equivalent guided transitions with $256$ sampling steps. We introduce dynamic constraints directly into the quadratic program. The resulting optimization problem is solved with IPOPT.
    \item \textbf{DiRecT}. We employ $32$ denoising steps and optimize over the second half of the generative process. We incorporate dynamic constraints directly into the receding-horizon subproblem, which is solved with IPOPT.
\end{itemize}
\subsection{Safe robotic manipulation}
\label{sec:d3il constrained manipulation in d3il - Appendix details}
We train a \emph{continuous-time} diffusion model with a temporal UNet architecture~\citep{jia2024towards} and horizon $H = 16$. The model outputs the state-action pair trajectory $(s_0, a_0, s_1, a_1, \ldots, s_{H-1}, a_{H-1}) \in \R^{6H}$ over the entire prediction horizon. We employ data-prediction parametrization and a cosine noise schedule with offset $s=0.008$~\citep{nichol2021improved}. We train the model with Adam~\citep{kingma2014adam} for $500{,}000$ steps with a batch size of $32$ on the D3IL \texttt{avoiding} dataset~\citep{jia2024towards}~\footnote{\url{https://github.com/ALRhub/d3il}} (MIT), scaled by min-max normalization. We use a learning rate of $2 \times 10^{-4}$ with a cosine annealing schedule for the first $100{,}000$ steps. We save checkpoints using an exponential moving average of the model weights with decay $0.995$.

\paragraph{Environment details.} The base D3IL \texttt{avoiding} task presents six circular pillars that must be avoided by the end-effector. The unguided diffusion policy learns to dodge these obstacles directly from human demonstrations. To assess performance over \emph{unseen} constraints, we restrict planning by increasing the radius of the first pillar and limiting side-movements by two planar regions. Dynamics constraints are imposed by least-squares fitting of the $7209$ training transitions across $96$ demonstrations, with the same procedure for Maze2D (see Appendix~\ref{sec:constrained navigation in Maze2d - Appendix details}).

\paragraph{Evaluation.} We evaluate each method over $100$ i.i.d. generated samples conditioned on a fixed initial position and with zero starting velocity. For each sample, we execute a rollout up to $100$ environment steps with a planning horizon $H=16$ and execution horizon $M=4$. We terminate the episode when the end-effector collides with a pillar, while we continue rollouts for violations of \emph{test-time} constraints. All other implementation details follow the procedure for  Maze2D (see Appendix~\ref{sec:constrained navigation in Maze2d - Appendix details}).

\subsection{Safe multi-robot motion planning}
\label{sec:safe multi-robot motion planning - Appendix details}
\paragraph{Problem setup.} In this task, multiple agents' trajectories in a common environment are generated in a collision-free and velocity-constrained manner. We test our method on the four environments from the MMD~\citep{shaoul2025multirobot} benchmark. A diffusion model is trained for each environment on single-agent trajectories showing a specific motion pattern. In the unconstrained setting, the objective is to generate trajectories $\traj \in \R^{H \times 2}$ conditioned on given start and target endpoints $(b, e) \in \R^{2 \times 2}$ presenting the desired motion behaviors for each MMD environment (Figure~\ref{fig:mrmp envs}):
\begin{itemize}
\item \texttt{Empty} is the simplest map, where demonstrations consist of straight-line motions and robots are penalized for deviating from a direct path.
\item \texttt{Highways} contains a central block that induces a warehouse-like traffic pattern, where agents are expected to move \emph{counterclockwise} around the obstacle.
\item \texttt{Conveyor} consists of two narrow corridors. Agents are rewarded for moving leftward in the top corridor and rightward in the bottom corridor.
\item \texttt{Drop-Region} features an obstacle-free central space and four drop-off regions. Agents are rewarded for pausing near the midpoint of any of the four square regions, emulating package-delivery tasks.
\end{itemize}
\setlength{\panelheight}{3.5cm}
\begin{figure}[t]
    \centering
    \begin{subfigure}[t]{0.24\linewidth}
        \centering
        \includegraphics[height=\panelheight,keepaspectratio]{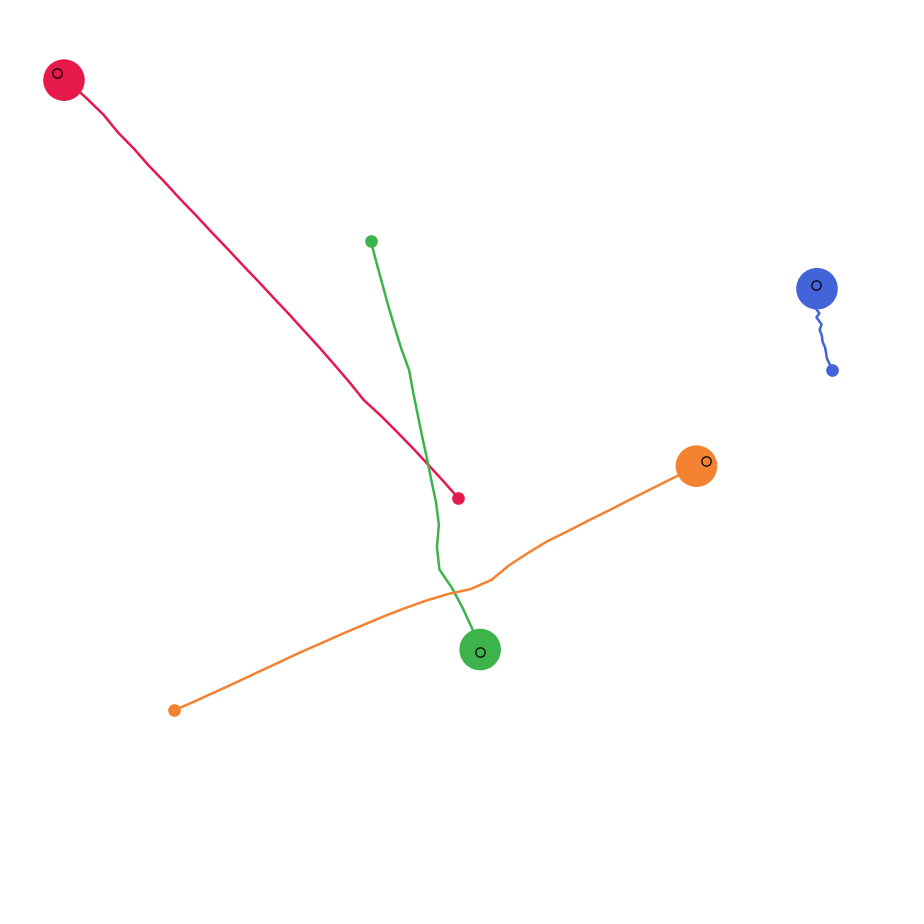}
        \caption{\texttt{Empty}}
        \label{fig:mrmp-empty}
    \end{subfigure}
    \hfill
    \begin{subfigure}[t]{0.24\linewidth}
        \centering
        \includegraphics[height=\panelheight,keepaspectratio]{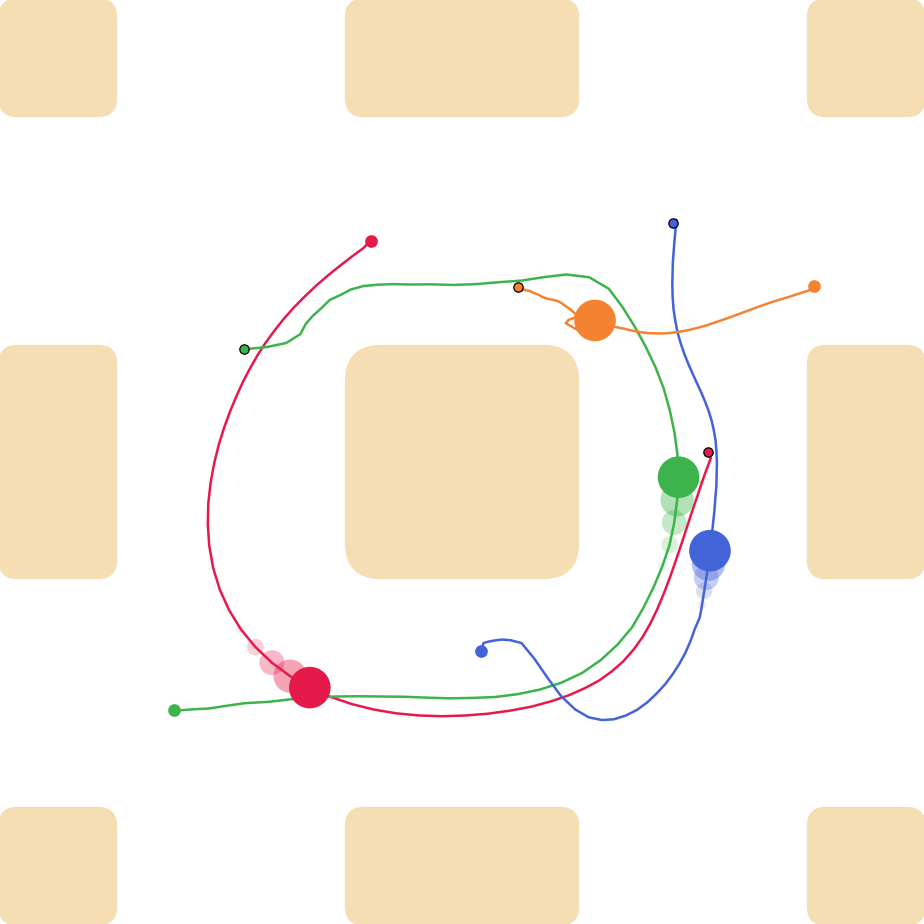}
        \caption{\texttt{Highways}}
        \label{fig:mrmp-highways}
    \end{subfigure}
    \hfill
    \begin{subfigure}[t]{0.24\linewidth}
        \centering
        \includegraphics[height=\panelheight,keepaspectratio]{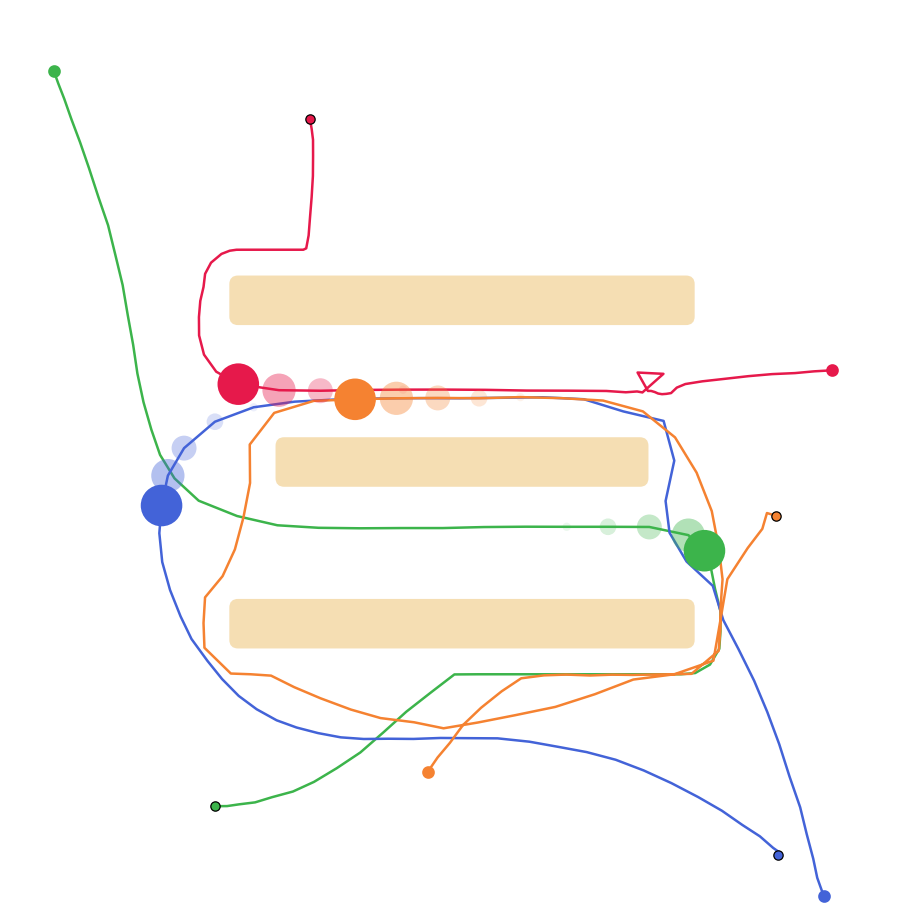}
        \caption{\texttt{Conveyor}}
        \label{fig:mrmp-conveyor}
    \end{subfigure}
    \hfill
    \begin{subfigure}[t]{0.24\linewidth}
        \centering
        \includegraphics[height=\panelheight,keepaspectratio]{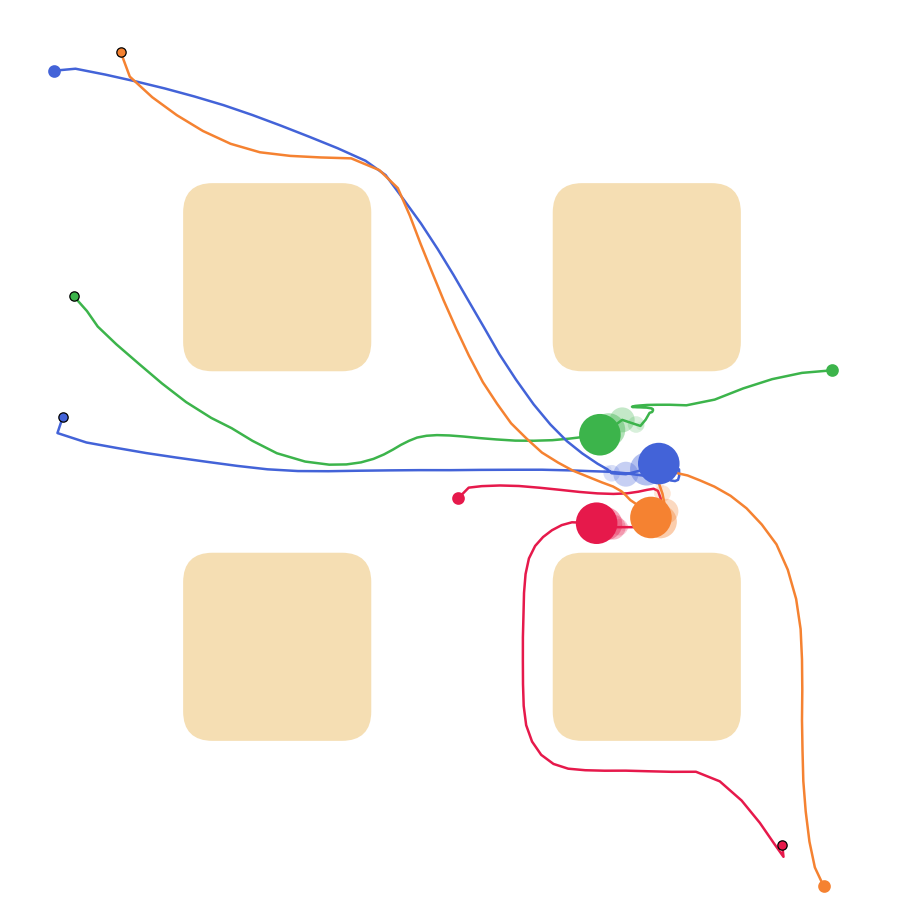}
        \caption{\texttt{Drop-Region}}
        \label{fig:mrmp-dropregion}
    \end{subfigure}
    \caption{Four major MMD environments with feasible optimized solutions. Robots navigate toward their goal positions while avoiding inter-robot and obstacle collisions and satisfying velocity constraints. Each environment exhibits a characteristic motion pattern that should be preserved.}    \label{fig:mrmp envs}
\end{figure}
Following the coupled trajectory-generation setup in~\citep{luan2026projected}, we seek to generate $N_a$ agent trajectories \emph{simultaneously}. The generated trajectories must avoid inter-robot and robot-obstacle collisions, satisfy \emph{test-time} velocity restrictions, and exhibit the desired MMD motion pattern.

Let $[K]$ denote $\{1, ..., K\}$ for any natural number $K \in \N$. We denote a generated sample of trajectories by $\traj = \left(\traj_1, ..., \traj_{N_a} \right) \in \R^{N_a \times H \times 2}$, where $\traj^j_i$ is the $j$-th predicted step for the $i$-th agent. Given an optimization objective $\mathcal{J}: \R^{N_a \times H \times 2} \rightarrow \R$ and starting locations $\{b_i\}_{i \in [N_a]}$ for each robot, the general MRMP problem is formulated as~\citep{liang2025simultaneous, shaoul2025multirobot}
\begin{subequations}
\label{eq:mrmp problem}
\begin{align}
    \min_{\traj \in \mathbb{R}^{N_a \times H \times 2}}
    \quad & \mathcal{J}(\traj) 
    \label{eq:mrmp_obj} \\
    \text{s.t.}\quad
    & \norm{\traj_i^{k+1} - \traj_i^k}_2 \le v_{\max}\Delta t,
    && \forall i \in [N_a],\ \forall k \in [H-1],
    \label{eq:mrmp_velocity} \\
    & \norm{\traj_i^{1} - b_i}_2 \le v_{\max}\Delta t,
    && \forall i \in [N_a],
    \label{eq:mrmp_initial_velocity} \\
    & \norm{\traj_i^k - \traj_j^k}_2 \ge 2\delta,
    && \forall i \neq j,\ \forall k \in [H],
    \label{eq:mrmp_inter_robot} \\
    & d(\traj_i^k,\mathcal{O}) \ge \delta,
    && \forall i \in [N_a],\ \forall k \in [H].
    \label{eq:mrmp_obstacle}
\end{align}
\end{subequations}
where $\delta$ is the robot radius, $v_{\max}$ is the test-time maximum allowable velocity, $\Delta t$ is the environment time increment between consecutive generated steps, and $d(\cdot, \mathcal{O}): \R^2 \rightarrow \R$ is the distance function to the obstacle set. Let $\Omega_V$, $\Omega_{CR}$, and $\Omega_{CO}$ denote the feasible sets associated with the velocity, inter-robot collision-avoidance, and robot-obstacle collision-avoidance constraints, respectively. We define the feasible set of Problem~\ref{eq:mrmp problem} as $\Omega = \Omega_V \cap \Omega_{CR} \cap \Omega_{CO}.$
The data adherence objective is implicitly specified by the single-agent diffusion policy. In this setting, Problem~\ref{eq:mrmp problem} can be naturally interpreted as constrained sampling from a pretrained diffusion model applied to the concatenation of all agent trajectories. More explicitly, given a diffusion model trained for a single agent, we can generate $N_a$ independent trajectories by sampling from the product distribution
\[
    \traj \sim \pdata^\star(\traj) = \prod_{i=1}^{N_a} \pdata(\traj_i).
\]
Therefore, the corresponding concatenated denoiser $\hat{x}^\star_{0,\theta}(\traj) = \bigoplus_{i=1}^{N_a} \hat{x}_{0,\theta}(\traj_i)$
is equivalent to a pretrained diffusion model that interpolates between the prior
$\Normal\left(0, \I_{N_a \times H \times 2}\right)$
and the product distribution $\pdata^\star$.
To impose \emph{test-time} constraints within our framework, we define a \emph{projection} operator onto the feasible set $\Omega$. This differs from the coupled diffusion framework of~\citep{luan2026projected} where only velocity limits are enforced exactly: rather than relying on soft collision penalties, we impose the nonconvex collision-avoidance constraints as hard constraints in the constrained optimization problem.
\paragraph{Evaluation details.}
We use the pretrained checkpoints provided by~\citep{luan2026projected}, which are trained with a \texttt{Diffusion Policy}~\citep{chi2025diffusion} backbone on the four MMD environments. For each environment, we impose a maximum allowable agent velocity, computed from forward differences of trajectory positions. For all methods, we generate samples using $100$ DDPM~\citep{ho2020denoisingdiffusionprobabilisticmodels} denoising steps, applying guidance during the second half of the denoising process.
We evaluate each method over $100$ randomly selected initial configurations. For each configuration, we generate a batch of $128$ candidate trajectories in parallel. Following~\citep{luan2026projected}, we define \emph{Success Rate} as the fraction of trials for which at least one trajectory in the generated batch is collision-free. In contrast, \emph{Constraint Safety} measures the average feasibility rate over \emph{all} generated trajectories in terms of both collision and kinematic constraints.
\subsubsection{Custom MRMP optimizer}
\label{sec:custom MRMP optimizer}
\paragraph{Problem setup.} We explicitly state the projection problem solved by the MRMP optimizer. Let
$X \in \mathbb{R}^{N_a \times H \times 2}$ denote the optimized trajectory tuple,
with $X_i^k \in \mathbb{R}^2$ denoting the position of agent $i$ at waypoint $k$.
Let $\hat X \in \mathbb{R}^{N_a \times H \times 2}$ denote the Tweedie
clean-trajectory estimate produced by the diffusion model. The proximal safe
trajectory is computed as the solution of
{
\setlength{\jot}{2pt}
\begin{subequations}
\label{eq:mrmp_projection}

\begin{equation}
\label{eq:mrmp_projection_obj}
X^\star = \argmin_{X\in\mathbb{R}^{N_a\times H\times 2}}
\quad
\frac{1}{2}\norm{X-\hat X}_F^2
+ \frac{\lambda_{\rm sm}}{2}
\sum_{i=1}^{N_a}\sum_{k=1}^{H-1}
\norm{X_i^{k+1}-X_i^k}_2^2 .
\end{equation}

\vspace{-0.8em}

\begin{alignat}{2}
\text{s.t.}\quad
& \norm{X_i^1-b_i}_2 \le v_{\max}\Delta t,
&\quad& \forall i\in[N_a],
\label{eq:mrmp_projection_start} \\
& \norm{X_i^{k+1}-X_i^k}_2 \le v_{\max}\Delta t,
&& \forall i\in[N_a],\ \forall k\in[H-1],
\label{eq:mrmp_projection_velocity}\\
& \norm{X_i^k-X_j^k}_2 \ge 2\delta,
&& \forall i\neq j,\ \forall k\in[H],
\label{eq:mrmp_projection_inter_robot}\\
& d(X_i^k,\mathcal{O}) \ge \delta,
&& \forall i\in[N_a],\ \forall k\in[H].
\label{eq:mrmp_projection_obstacle}
\end{alignat}

\end{subequations}
}
The first term keeps the projected tuple close to the diffusion model's Tweedie estimate $\hat X$, while the smoothness term with strength $\lambda_{\rm sm}$ optionally encourages straight-line trajectories. The constraints enforce initial-state consistency, inference-time velocity limits, inter-robot separation, and robot-obstacle clearance. Thus, the optimizer defines a (proximal) projection $X^\star=\Pi_\Omega(\hat X)$ onto the feasible set $\Omega=\Omega_V\cap\Omega_{CR}\cap\Omega_{CO}$.
\newline \newline
Different approaches are possible for solving this problem. A first possibility is to employ an off-the-shelf nonlinear solver such as IPOPT. However, general-purpose nonlinear solvers are typically designed around CPU-based sparse linear algebra and do not provide a plug-and-play GPU implementation suitable for extensive sweeps and batched inference. An alternative is to relax the nonconvex collision constraints and optimize the resulting objective with Lagrangian or augmented-Lagrangian methods ~\citep{liang2025simultaneous}. However, since the collision constraints enter the objective through nonlinear penalty or multiplier terms, the resulting optimization can remain nonconvex and still requires a general-purpose CPU-based nonlinear solver. Inspired by the ADMM projector in~\citep{luan2026projected}, given the convexity of the velocity constraints and the spatially local structure of the collision constraints, we derive a custom optimizer by combining successive convexification of the active collision constraints with an ADMM splitting of the resulting convex subproblem~\citep{boyd2011distributed, mao2018successive}. The complete derivation is as follows.

\paragraph{Successive convexification.}
The nonconvexity of~\eqref{eq:mrmp_projection} arises from the inter-robot and robot-obstacle collision constraints. We handle these constraints with an outer successive-convexification loop. Let $X^{(m)}$ denote the current trajectory at SCP iteration $m$. For each active inter-robot constraint, define
\[
    r_{ij}^{k,(m)} = X_i^{k,(m)} - X_j^{k,(m)}, 
    \qquad
    n_{ij}^{k,(m)} =
    \frac{r_{ij}^{k,(m)}}{\norm{r_{ij}^{k,(m)}}_2}.
\]
Using the first-order lower approximation of the norm, we replace
$\norm{X_i^k-X_j^k}_2 \ge 2\delta$ with
\[
    \big(n_{ij}^{k,(m)}\big)^\top (X_i^k-X_j^k) \ge 2\delta .
\]

For robot-obstacle constraints, let $d_\ell(p)$ denote the signed distance from point $p$ to obstacle primitive $\mathcal O_\ell$. For each active constraint
$d_\ell(X_i^k)-\delta \ge 0$, we use the first-order approximation around $X_i^{k,(m)}$:
\[
    d_\ell(X_i^{k,(m)}) - \delta
    +
    \nabla d_\ell(X_i^{k,(m)})^\top
    \big(X_i^k-X_i^{k,(m)}\big)
    \ge 0.
\]
For all MMD environments, obstacles decompose into rectangular and circular primitives, so both signed distances and their gradients can be computed exactly. At SCP iteration $m$, we compute the next trajectory $X^{(m+1)}$ by solving the following convexified, slack-penalized projection problem:
{
\setlength{\jot}{2pt}
\begin{subequations}
\label{eq:mrmp_scp_subproblem}

\begin{flalign}
\label{eq:mrmp_scp_obj}
&\begin{aligned}
(X^{(m+1)},s^{(m+1)}) \in \argmin_{X,s}\quad
& \frac{1}{2}\norm{X-\hat X}_F^2
+ \frac{\lambda_{\rm sm}}{2}
\sum_{i=1}^{N_a}\sum_{k=1}^{H-1}
\norm{X_i^{k+1}-X_i^k}_2^2  \\
& + \frac{\tau_m}{2}\norm{X-X^{(m)}}_F^2
+ \mu_m \mathbf{1}^\top s .
\end{aligned}&&
\end{flalign}

\vspace{-0.8em}

\begin{alignat}{2}
\text{s.t.}\;
& \norm{X_i^1-b_i}_2 \le v_{\max}\Delta t,
&\quad& \forall i\in[N_a],
\label{eq:mrmp_scp_start} \\
& \norm{X_i^{k+1}-X_i^k}_2 \le v_{\max}\Delta t,
&& i\in[N_a],\ k\in[H-1],
\label{eq:mrmp_scp_velocity} \\
& \big(n_{ij}^{k,(m)}\big)^\top(X_i^k-X_j^k)+s_{ij}^k \ge 2\delta,
&& \forall (i,j,k)\in\mathcal A_{\rm RR}^{(m)},
\label{eq:mrmp_scp_rr} \\
& \big(n_{i\ell}^{k,(m)}\big)^\top X_i^k+s_{i\ell}^k
  \ge
  \big(n_{i\ell}^{k,(m)}\big)^\top X_i^{k,(m)}
  -\big(d_\ell(X_i^{k,(m)})-\delta\big),
&& \forall (i,\ell,k)\in\mathcal A_{\rm O}^{(m)},
\label{eq:mrmp_scp_obs} \\
& s\ge 0.
&&
\label{eq:mrmp_scp_slack}
\end{alignat}

\end{subequations}
}
Here $\mathcal A_{\rm RR}^{(m)}$ and $\mathcal A_{\rm O}^{(m)}$ denote the active inter-robot and robot-obstacle constraints at the linearization point $X^{(m)}$. The slack variables $s\ge 0$ keep the convexified subproblem feasible, while the $\ell_1$ penalty $\mathbf{1}^\top s=\norm{s}_1$ encourages sparse violations. We use monotone schedules $\mu_{m+1}=\min\{\gamma_\mu\mu_m,\mu_{\max}\}$ and $\tau_{m+1}=\min\{\gamma_\tau\tau_m,\tau_{\max}\}$, with $\gamma_\mu,\gamma_\tau>1$. Increasing $\mu_m$ progressively enforces collision satisfaction, while increasing $\tau_m$ regularizes each SCP step around its linearization point.
\paragraph{ADMM splitting.}
To solve~\eqref{eq:mrmp_scp_subproblem}, we introduce auxiliary variables
$V,C,Z\in\mathbb{R}^{N_a\times H\times 2}$. The variable $V$ is used to impose velocity feasibility, $Z$ stores the corresponding velocity increments, and $C$ isolates the linearized collision constraints. For notational convenience, set $V_i^0=b_i$ and define the anchored first-difference operator
\[
    (D_bV)_i^k = V_i^k - V_i^{k-1},
    \qquad i\in[N_a],\ k\in[H].
\]
Thus, $(D_bV)_i^1=V_i^1-b_i$ encodes the initial velocity constraint, while the remaining entries encode consecutive displacements. The velocity constraint is equivalently written as $D_bV=Z, \; Z \in \mathcal{K}_Z$ with
\[
    \mathcal K_Z =
    \{Z:\norm{Z_i^k}_2\le v_{\max}\Delta t,\ i\in[N_a],\ k\in[H]\}.
\]

Let $G^{(m)}$ denote the matrix collecting the active linearized collision constraints at SCP iteration $m$, and let $h^{(m)}$ denote the corresponding vector of right-hand sides. The slack variable $s$ has the same dimension as $h^{(m)}$ and keeps the active collision constraints feasible. We write these constraints compactly as $(C,s)\in\mathcal K_C^{(m)}$, where
\[
    \mathcal K_C^{(m)}
    =
    \{(C,s):G^{(m)}C+s\ge h^{(m)},\ s\ge 0\}.
\]
Let $\mathbb I_{\mathcal K}$ denote the indicator function of a set $\mathcal K$, equal to $0$ on $\mathcal K$ and $+\infty$ otherwise. Define
\[
    f_m(X)
    =
    \frac{1}{2}\norm{X-\hat X}_F^2
    +
    \frac{\lambda_{\rm sm}}{2}
    \sum_{i=1}^{N_a}\sum_{k=1}^{H-1}
    \norm{X_i^{k+1}-X_i^k}_2^2
    +
    \frac{\tau_m}{2}\norm{X-X^{(m)}}_F^2 .
\]
The SCP subproblem can then be written in ADMM form as
\begin{equation}
\label{eq:mrmp_admm_split}
\begin{aligned}
\min_{X,V,C,Z,s}\quad
& f_m(X)+\mu_m\mathbf{1}^\top s
+\mathbb{I}_{\mathcal K_Z}(Z)
+\mathbb{I}_{\mathcal K_C^{(m)}}(C,s) \\
\text{s.t.}\quad
& X=V,\qquad X=C,\qquad D_bV=Z, \qquad s \ge 0 .
\end{aligned}
\end{equation}
Using scaled dual variables $U_{XV}$, $U_{XC}$, and $U_Z$ with penalties
$\rho_{XV},\rho_{XC},\rho_Z>0$, the corresponding scaled augmented Lagrangian formulation of \eqref{eq:mrmp_admm_split} is:
\begin{equation}
\label{eq:mrmp_admm_lagrangian}
\begin{aligned}
\mathcal L_\rho
=&\ f_m(X)+\mu_m\mathbf{1}^\top s
+\mathbb{I}_{\mathcal K_Z}(Z)
+\mathbb{I}_{\mathcal K_C^{(m)}}(C,s) \\
&+\frac{\rho_{XV}}{2}\norm{X-V+U_{XV}}_F^2
-\frac{\rho_{XV}}{2}\norm{U_{XV}}_F^2 \\
&+\frac{\rho_{XC}}{2}\norm{X-C+U_{XC}}_F^2
-\frac{\rho_{XC}}{2}\norm{U_{XC}}_F^2 \\
&+\frac{\rho_Z}{2}\norm{D_bV-Z+U_Z}_F^2
-\frac{\rho_Z}{2}\norm{U_Z}_F^2 .
\end{aligned}
\end{equation}
ADMM then alternates between the following updates at inner iteration $q$ of SCP iteration $m$. For the linear solves, write the anchored difference as
$D_bV=AV-B$, where $A$ is the first-difference matrix with
$(AV)_i^1=V_i^1$ and $(AV)_i^k=V_i^k-V_i^{k-1}$ for $k\ge 2$, while
$B_i^1=b_i$ and $B_i^k=0$ for $k\ge 2$.

\begin{itemize}

\item \textbf{$X$-update.}
The $X$-update minimizes the scaled augmented Lagrangian with respect to $X$ while keeping the remaining variables fixed:
\[
X^{q+1}
=
\argmin_X
\left\{
f_m(X)
+
\frac{\rho_{XV}}{2}\norm{X-V^q+U_{XV}^q}_F^2
+
\frac{\rho_{XC}}{2}\norm{X-C^q+U_{XC}^q}_F^2
\right\}.
\]
This is an unconstrained quadratic program. Setting its gradient to zero gives the linear system
\[
\Big((1+\tau_m+\rho_{XV}+\rho_{XC})I+\lambda_{\rm sm}D^\top D\Big)X^{q+1}
=
\hat X+\tau_m X^{(m)}
+\rho_{XV}(V^q-U_{XV}^q)
+\rho_{XC}(C^q-U_{XC}^q),
\]
where $D$ is the unanchored first-difference operator appearing in the smoothness term.

\item \textbf{$V$-update.}
The $V$-update collects the terms involving the velocity split:
\[
V^{q+1}
=
\argmin_V
\left\{
\frac{\rho_{XV}}{2}\norm{X^{q+1}-V+U_{XV}^q}_F^2
+
\frac{\rho_Z}{2}\norm{D_bV-Z^q+U_Z^q}_F^2
\right\}.
\]
This is also an unconstrained quadratic program. Using $D_bV=AV-B$, its optimality condition yields
\[
\big(\rho_{XV}I+\rho_ZA^\top A\big)V^{q+1}
=
\rho_{XV}(X^{q+1}+U_{XV}^q)
+
\rho_ZA^\top(Z^q-U_Z^q+B).
\]

\item \textbf{$Z$-update.}
The $Z$-update enforces the velocity bound by projecting the current anchored displacement estimate onto $\mathcal K_Z$:
\[
Z^{q+1}
=
\Pi_{\mathcal K_Z}\big(D_bV^{q+1}+U_Z^q\big).
\]
This projection has a closed-form solution as it is applied \emph{row-wise}. For
$\nu_i^k=(D_bV^{q+1})_i^k+(U_Z^q)_i^k$,
\[
Z_i^{k,q+1}
=
\begin{cases}
(v_{\max}\Delta t)\dfrac{\nu_i^k}{\norm{\nu_i^k}_2},
& \norm{\nu_i^k}_2>v_{\max}\Delta t,\\[0.8em]
\nu_i^k,
& \text{otherwise},
\end{cases}
\qquad i\in[N_a],\ k\in[H].
\]

\item \textbf{$C,s$-update.}
The collision update projects onto the linearized collision constraints while penalizing slack:
\[
(C^{q+1},s^{q+1})
=
\argmin_{(C,s)\in\mathcal K_C^{(m)}}
\left\{
\frac{\rho_{XC}}{2}\norm{C-(X^{q+1}+U_{XC}^q)}_F^2
+
\mu_m\mathbf{1}^\top s
\right\}.
\]
This is a convex quadratic program with linear inequality constraints. Rather than calling a generic QP solver, we solve its box-constrained dual. Let $Q^q=X^{q+1}+U_{XC}^q$. The dual problem is
\[
\lambda^\star
=
\argmax_{0\le \lambda\le \mu_m}
\left\{
-\frac{1}{2\rho_{XC}}\norm{(G^{(m)})^\top\lambda}_2^2
+
\lambda^\top\big(h^{(m)}-G^{(m)}Q^q\big)
\right\}.
\]
We solve this dual by projected gradient ascent on $\lambda$. The primal variables are then recovered as
\[
C^{q+1}
=
Q^q+\frac{1}{\rho_{XC}}(G^{(m)})^\top\lambda^\star,
\qquad
s^{q+1}
=
\max\{0,h^{(m)}-G^{(m)}C^{q+1}\}.
\]

\item \textbf{Dual updates.}
Finally, the scaled dual variables are updated by accumulating the primal residuals of the three splitting constraints:
\[
U_{XV}^{q+1}=U_{XV}^q+X^{q+1}-V^{q+1},\qquad
U_{XC}^{q+1}=U_{XC}^q+X^{q+1}-C^{q+1},
\]
\[
U_Z^{q+1}=U_Z^q+D_bV^{q+1}-Z^{q+1}.
\]
\end{itemize}
The $X$- and $V$-updates require solving linear systems with matrices of the form $\alpha I+\beta D^\top D$ or $\alpha I+\beta A^\top A$, which are tridiagonal along the time dimension. Their factorizations can therefore be cached and reused across ADMM iterations within each SCP iteration. We implement the projector in JAX~\citep{jax2018github}, using just-in-time compilation and GPU parallelism for batched inference.

For simplicity in implementation, we use fixed iteration budgets for all nested loops: $K_{\rm cvx}$ successive-convexification iterations, $K_{\rm ADMM}$ ADMM iterations, and $K_{\rm QP}$ projected-gradient steps for the local collision QP. Developing principled termination criteria for these loops is an interesting direction for future work and could further reduce inference-time computational overhead. We summarize the full procedure in Algorithm~\ref{alg:mrmp_scp_admm} and default parameters in Table~\ref{tab:mrmp_scp_admm_main_params}.

\begin{algorithm}[t]
\caption{Batched SCP--ADMM Projection for MRMP}
\label{alg:mrmp_scp_admm}
\small

\KwIn{Tweedie reference \(\hat X\), starts \(b\), obstacle primitives \(\mathcal O\), robot radius \(\delta\), step bound \(d_{\max}=v_{\max}\Delta t\); iterations \(K_{\rm cvx},K_{\rm ADMM},K_{\rm QP}\); penalties \(\rho_{XV},\rho_{XC},\rho_Z\); schedules \(\mu_m,\tau_m\); dual step size \(\alpha\)}
\KwOut{Projected trajectory \(X\)}

Pre-compute \(D\), \(A\), and \(B\) such that \(D_bV=AV-B\), with \(B_i^1=b_i\) and \(B_i^k=0\) for \(k\ge 2\)\;
\(X\gets \hat X\), \(V\gets X\), \(C\gets X\)\tcp*{Initialize from diffusion estimate}
\(Z\gets \Pi_{\mathcal K_Z}(AV-B)\), \quad \(U_{XV},U_{XC},U_Z\gets 0\)\tcp*{Initialize velocity split and duals}

\BlankLine
\textbf{SCP convexification}\;
\For{\(m=0,\ldots,K_{\rm cvx}-1\)}{
    \(X_{\rm lin}\gets X\)\;
    Select active constraints \(\mathcal A_{\rm RR}^{(m)},\mathcal A_{\rm O}^{(m)}\) at \(X_{\rm lin}\)\;
    Build the linearized collision constraints \(G^{(m)}C+s\ge h^{(m)}\), \(s\ge 0\)\;
    \(H_X\gets(1+\tau_m+\rho_{XV}+\rho_{XC})I+\lambda_{\rm sm}D^\top D\)\tcp*{\(X\)-update system}
    \(H_V\gets\rho_{XV}I+\rho_ZA^\top A\)\tcp*{\(V\)-update system}
    Cache factorizations of \(H_X\) and \(H_V\)\tcp*{Tridiagonal along time}
    \(C\gets X\), \(s\gets 0\), \(U_{XC}\gets 0\)\;

    \BlankLine
    \textbf{ADMM splitting}\;
    \For{\(q=0,\ldots,K_{\rm ADMM}-1\)}{
        \textbf{\(X,V,Z\) updates}\;
        \(R_X\gets \hat X+\tau_m X_{\rm lin}+\rho_{XV}(V-U_{XV})+\rho_{XC}(C-U_{XC})\)\;
        \(X\gets \textsc{SolveBatch}(H_X X=R_X)\)\tcp*{Batched linear solve on GPU}
        \(R_V\gets \rho_{XV}(X+U_{XV})+\rho_ZA^\top(Z-U_Z+B)\)\;
        \(V\gets \textsc{SolveBatch}(H_V V=R_V)\)\tcp*{Batched linear solve on GPU}
        \(W\gets AV-B+U_Z\)\;
        \(Z_i^k\gets \min\left\{1,\dfrac{d_{\max}}{\norm{W_i^k}_2+\varepsilon_{min}}\right\}W_i^k,\quad i\in[N_a],\ k\in[H]\)\tcp*{Project onto velocity balls}

        \BlankLine
        \textbf{Local collision-QP dual}\;
        \(Q\gets X+U_{XC}\), \quad \(\lambda\gets 0\)\;
        \For{\(r=0,\ldots,K_{\rm QP}-1\)}{
            \(g_\lambda\gets -\rho_{XC}^{-1}G^{(m)}(G^{(m)})^\top\lambda+\big(h^{(m)}-G^{(m)}Q\big)\)\tcp*{Dual gradient}
            \(\lambda\gets \operatorname{clip}(\lambda+\alpha g_\lambda,\ 0,\ \mu_m)\)\tcp*{Box-constrained dual step}
        }

        \BlankLine
        \textbf{\(C,s\) recovery and dual updates}\;
        \(C\gets Q+\rho_{XC}^{-1}(G^{(m)})^\top\lambda\)\;
        \(s\gets \max\{0,\ h^{(m)}-G^{(m)}C\}\)\;
        \(U_{XV}\gets U_{XV}+X-V\)\tcp*{Position--velocity residual}
        \(U_{XC}\gets U_{XC}+X-C\)\tcp*{Position--collision residual}
        \(U_Z\gets U_Z+AV-B-Z\)\tcp*{Velocity-projection residual}
    }
}
\Return{\(X\)}
\end{algorithm}

\begin{table}[t]
\centering
\small
\setlength{\tabcolsep}{4.5pt}
\renewcommand{\arraystretch}{1.10}
\caption{Main hyperparameters used by the batched SCP--ADMM MRMP optimizer.}
\label{tab:mrmp_scp_admm_main_params}
\begin{tabular}{lcccccccccccc}
\toprule
\textbf{Symbol}
& $K_{\rm cvx}$
& $K_{\rm ADMM}$
& $K_{\rm QP}$
& $\rho_{XV}$
& $\rho_{XC}$
& $\rho_Z$
& $\mu_{\rm init}$
& $\mu_{\max}$
& $\gamma_\mu$
& $\tau_{\rm init}$
& $\tau_{\max}$
& $\gamma_\tau$ \\
\midrule
\textbf{Value}
& $5$
& $50$
& $100$
& $1.0$
& $1.0$
& $1.0$
& $1.0$
& $1000.0$
& $2.0$
& $0.1$
& $10.0$
& $1.5$ \\
\bottomrule
\end{tabular}
\end{table}

\subsection{Safe and diverse contact-rich manipulation}
\label{sec:diverse constrained planning in PushT - Appendix Details}
\paragraph{Training.} We employ a diffusion model with noise-prediction parametrization on the augmented expert demonstrations dataset using default parameters in~\citep{feng2024ltldog}. 
\paragraph{Evaluation.} We follow the same setup as in~\citep{luan2026projected} and generate $100$ pairs of trajectories for each of $50$ random starting configurations. We constrain each agent by imposing \emph{test-time} velocity limits corresponding to the $90\%$ quantile in the original work. Sampling is performed with a planning horizon of $H_{p}=16$, execution horizon $H_{a}=8$, and for a maximum of $360$ environment steps. 
\paragraph{Cost functions.} For clarity of exposition, we repeat the definition of the cost functions used to encourage non-intersecting trajectories during coupled sampling. \\

\textit{Determinantal Point Process (DPP)}~\citep{feng2024ltldog}. Given a generated pair $(X, Y) \in \R^{H \times 2 \times 2}$, the DPP cost directly penalizes high cosine similarity between the flattened trajectories $(\bar{X}, \bar{Y}) \in \R^{2H \times 2}$:
\begin{equation}
    c_{\text{DPP}}(X, Y) = \log \left( \cos \angle (\bar{X}, \bar{Y}) + \varepsilon \right)
\label{eq:dpp cost function}    
\end{equation}
where $\varepsilon > 1$ regulates cost sharpness. \\

\textit{Log-Barrier (LB)}. Given a trajectory tuple $(X, Y) \in \R^{H \times 2 \times 2}$, log-barrier cost discourages proximity of trajectories, as measured by the Euclidean distance: 
\begin{equation}
    c_{\text{LB}}(X, Y) = - \sum_{i=1}^H \log ( \norm{X_i - Y_i} + \alpha)
\label{eq:lb cost function}
\end{equation}
where $\alpha > 0$ is an offset parameter. 
\paragraph{Projected Gradient Descent.} For the PCD-DPP and PCD-LB baselines, we use one gradient iteration per denoising step with guidance strengths of $\gamma_{DPP} = 0.2$ and $\gamma_{LB} = 0.02$ respectively. For DiRecT-DPP and DiRecT-LB we introduce diversity costs directly in the optimization problem \eqref{eq:inner-subproblem} scaled by an appropriate weight $\lambda_c > 0$. For a comparison under similar computational budgets, the optimization problem is not solved to convergence, instead we employ one iteration of projected gradient descent~\citep{levitin1966constrained}. Given a pair of denoised estimates $(\bar{X}, \bar{Y})$, the optimized pair $(X^\star, Y^\star)$ becomes\footnote{We include the step size for the gradient update in $\lambda_c$, instead of defining an additional hyperparameter. Moreover, the gradient of the quadratic regularization from Tweedie's estimate vanishes at $(\bar{X}, \bar{Y})$.}:
\begin{equation}
\begin{aligned}
X^\star &= \Pi_{\mathcal{K}} \left( \bar X-\lambda_c\nabla_X c_{\text{DPP}/\text{LB}}(X,Y)\big|_{\bar X,\bar Y} \right),\\
Y^\star &= \Pi_{\mathcal{K}} \left(\bar Y-\lambda_c\nabla_Y c_{\text{DPP}/\text{LB}}(X,Y)\big|_{\bar X,\bar Y} \right).
\end{aligned}
\label{eq:projected gradient descent pusht}
\end{equation}
Here, $\Pi_K$ is the same ADMM-based projector as the PCD baselines over the set $\mathcal{K}$ of velocity-restricted trajectories. We sweep multiple weights for $\lambda_c$ and choose $\lambda_c = 0.005$ for DiRecT-LB and $\lambda_c = 0.05$ for DiRecT-DPP. Additional details are provided in Appendix~\ref{sec:pareto-optimal analysis for PushT}.
\section{Additional results}
\label{sec:additional results - Appendix}
\subsection{Safe maze navigation}
\label{sec:constrained navigation in Maze2d - Appendix}
We report in Table~\ref{tab: maze2d results additional} a detailed comparison of our method against additional baselines, with a focus on the difference between \emph{generated} and \emph{rollout} safety. We consider these additional methods for comparison: \textbf{Classifier Guidance}~\citep{dhariwal2021diffusionmodelsbeatgans} which employs gradient guidance through \emph{soft} violation costs without posterior sampling, \textbf{Primal-Dual}~\citep{zhang2025constraineddiffuserssafeplanning} as an additional Lagrangian-based method, and the two additional SafeDiffuser variations \textbf{SafeDiffuser-ReS}, and \textbf{SafeDiffuser-TVS}~\citep{xiao2025safediffuser} for CBF-based constraining. We also report additional metrics on the \emph{generated} trajectory: (i) \emph{Safety Rate}, measuring the fraction of samples that correctly avoid obstacles; (ii) mean \emph{Violations}, measuring the number of steps the predicted trajectory remains infeasible; (iii) \emph{Acceleration Smoothness} (AS), quantifying the change in velocity across the trajectory, and computed as the magnitude of the second-order position differences averaged across all planned transitions: $AS = (H - 2)^{-1} \ \sum_{i=1}^{H-2} \norm{s_{i+1} - 2 s_i + s_{i-1}}$; (iv) \emph{Curvature Smoothness} (CS), measuring the mean curvature of the generated trajectory. Following~\citep{dai2025safeflowsaferobotmotion}, we define the turning angle $\theta_i = \cos^{-1} (w_i^T w_{i+1} / (\norm{w_i} \norm{w_{i+1}})), \;\; w_i = s_{i+1} - s_i, \;\; i=0, ..., H-3$, and compute the metric as the average across the planned horizon $CS = (H - 2)^{-1} \sum_{i=0}^{H-3}(1 - \cos \theta_i)$.
\input{tables/maze2d_results_additional}

In both variants DiRecT outperforms all other baselines in terms of \emph{rollout} safety. Even though Projected Diffusion, Primal-Dual and the three SafeDiffuser variations generate safe trajectories at planning time on the \texttt{broad} instance of the problem, their rollout performance remains limited. This can be explained by noting that these methods find feasible solutions to
the optimization problem by departing substantially from the data distribution.
As shown in Figures~\ref{fig:maze2d broad visualization}--\ref{fig:maze2d narrow visualization}, prior methods in highly constrained dynamic settings often produce trajectories that depart from the learned maze structure and pass through maze walls. In contrast, our method performs optimization near the data manifold, thereby remaining proximal to the data distribution. In the narrow setting, we observe that IPOPT struggles to find feasible solutions within the allowed \(200\) iterations, sometimes leading to infeasible generated solutions for projection-based methods. Moreover, primal-dual methods and some SafeDiffuser variants exhibit instabilities on certain samples under dynamic constraints. 
\subsection{Safe robotic manipulation}
\label{sec:constrained manipulation in D3IL - Appendix}
We compare the same extended baselines as in Table~\ref{tab: maze2d results additional} and report safety and success rates of these additional policies in Table~\ref{tab:d3il results additional}. Our method is the only one to obtain perfect task success rate.

For the primal-dual implementation, we tuned the reference hyperparameters to the best of our ability, including the dual-ascent learning rate. Nevertheless, in the highly constrained setting with dynamics constraints, the method exhibited increasing instability, occasionally producing trajectories that triggered errors in the MuJoCo simulator during execution. 
\input{tables/d3il_results_additional.tex}
\FloatBarrier
\subsection{Safe multi-robot motion planning}
\label{sec:safe multi-robot motion planning - Appendix}
We present sweep results for each of the three velocity limitations and number of agents in Tables~\ref{tab:empty results additional},~\ref{tab:highways results additional},~\ref{tab:conveyor results additional},~\ref{tab:dropregion results additional}. We report computational times as well as the mean number of collisions. DiRecT consistently improves \emph{Success} and \emph{Constraint safety}. On \texttt{Empty} and \texttt{Highways}, our method matches the  \emph{data adherence} of the pretrained model for all velocity limitations. We note that on \texttt{Conveyor} and \texttt{Drop-Region} comparing data adherence between DiRecT and the PCD variants is not straightforward for $N \in \{12, 16, 20 \}$, as the latter are not capable of generating feasible trajectories, thereby showcasing high task success while neglecting safety. 
\input{tables/mrmp_results_additional}
\subsection{Safe and diverse contact-rich manipulation}
\label{sec:pareto-optimal analysis for PushT}
As outlined in the ablation studies of~\citep{luan2026projected}, task success and trajectory diversity (DTW, DFD) are conflicting objectives in the \texttt{PushT} task. Therefore, there exists a Pareto-optimality front that can be traced by varying coupling strength. We sweep the coupling strength coefficient for gradient guidance in PCD and the step size of the Projected Gradient Descent update $\lambda_c$ for DiRecT. Figure~\ref{fig:pusht pareto fronts} shows that Pareto fronts obtained with our method dominate those derived from PCD for both DPP and LB cost functions. Furthermore, DiRecT improves over the baselines while requiring only \emph{half} the projection and gradient evaluations, as we start guidance in the second half of denoising. This highlights the importance of optimizing near the data distribution where constraints and costs are more naturally defined. 

\begin{figure}[h]
    \centering
    \newcommand{\paretoheight}{4.2cm}
    \begin{subfigure}[t]{0.495\linewidth}
        \centering
        \includegraphics[height=\paretoheight,keepaspectratio]{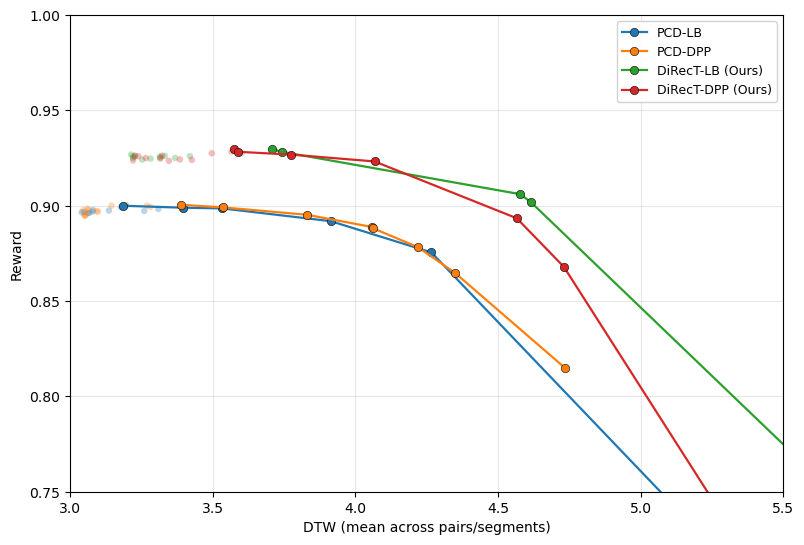}
        \caption{\textbf{TC} vs \textbf{DTW}. Top-right is better.}
        \label{fig:pusht-pareto-dtw}
    \end{subfigure}%
    \hspace{0.01\linewidth}%
    \begin{subfigure}[t]{0.495\linewidth}
        \centering
        \includegraphics[height=\paretoheight,keepaspectratio]{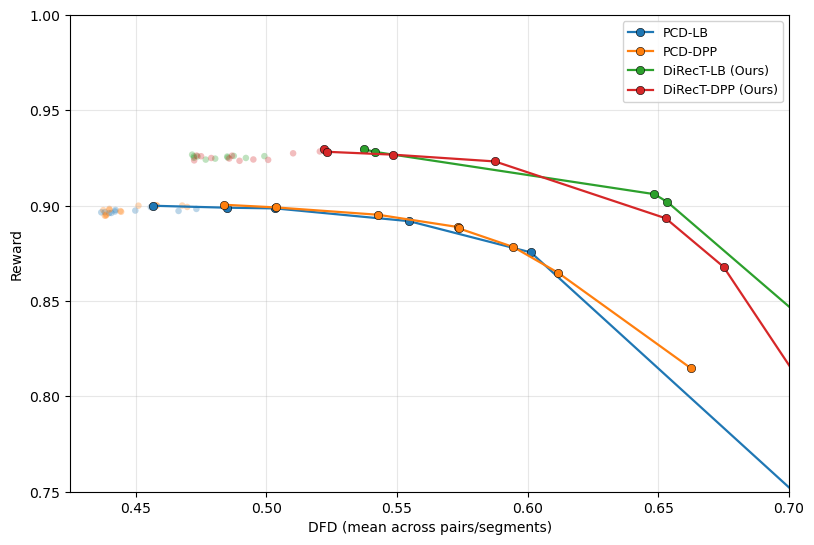}
        \caption{\textbf{TC} vs \textbf{DFD}. Top-right is better.}
        \label{fig:pusht-pareto-dfd}
    \end{subfigure}
    \caption{Pareto fronts for the PCD and DiRecT variants, obtained by sweeping the gradient coupling strength $\gamma$ for PCD and the cost weight $\lambda_c$ for DiRecT. Fronts are drawn as lines connecting Pareto-optimal points for each method, while dominated points are shaded.}
\label{fig:pusht pareto fronts}
\end{figure}
\FloatBarrier
\pagebreak
\clearpage
\section{Visualizations}
\label{sec:visualizations - Appendix}
\begin{figure}[h]
    \centering
    \includegraphics[width=0.89\linewidth]{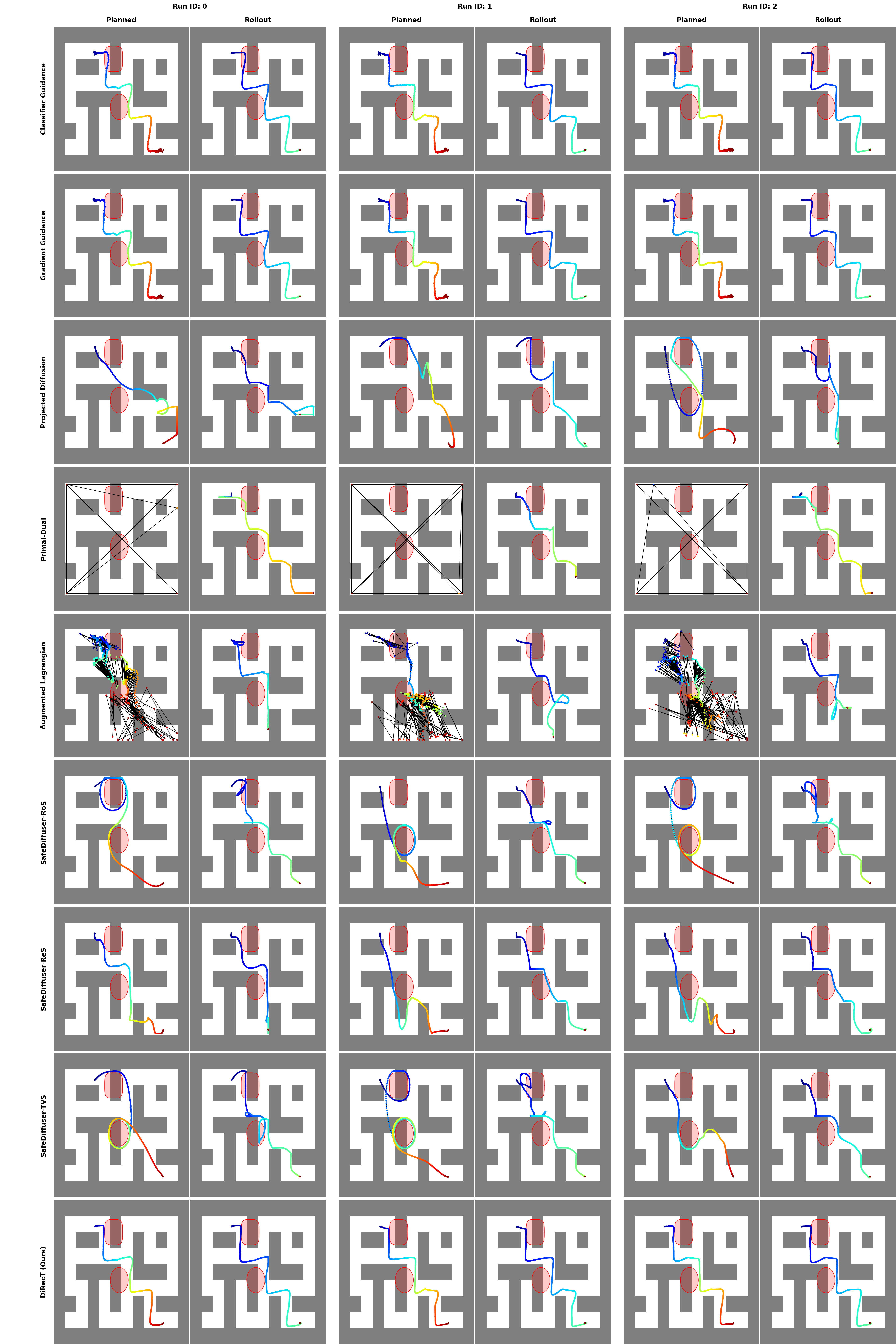}
    \caption{Comparison of \emph{planned} and \emph{executed} trajectories for safe navigation in \texttt{Maze2D-broad}. Each row corresponds to a planning method.}
\label{fig:maze2d broad visualization}
\end{figure}

\begin{figure}[h]
    \centering
    \includegraphics[width=0.89\linewidth]{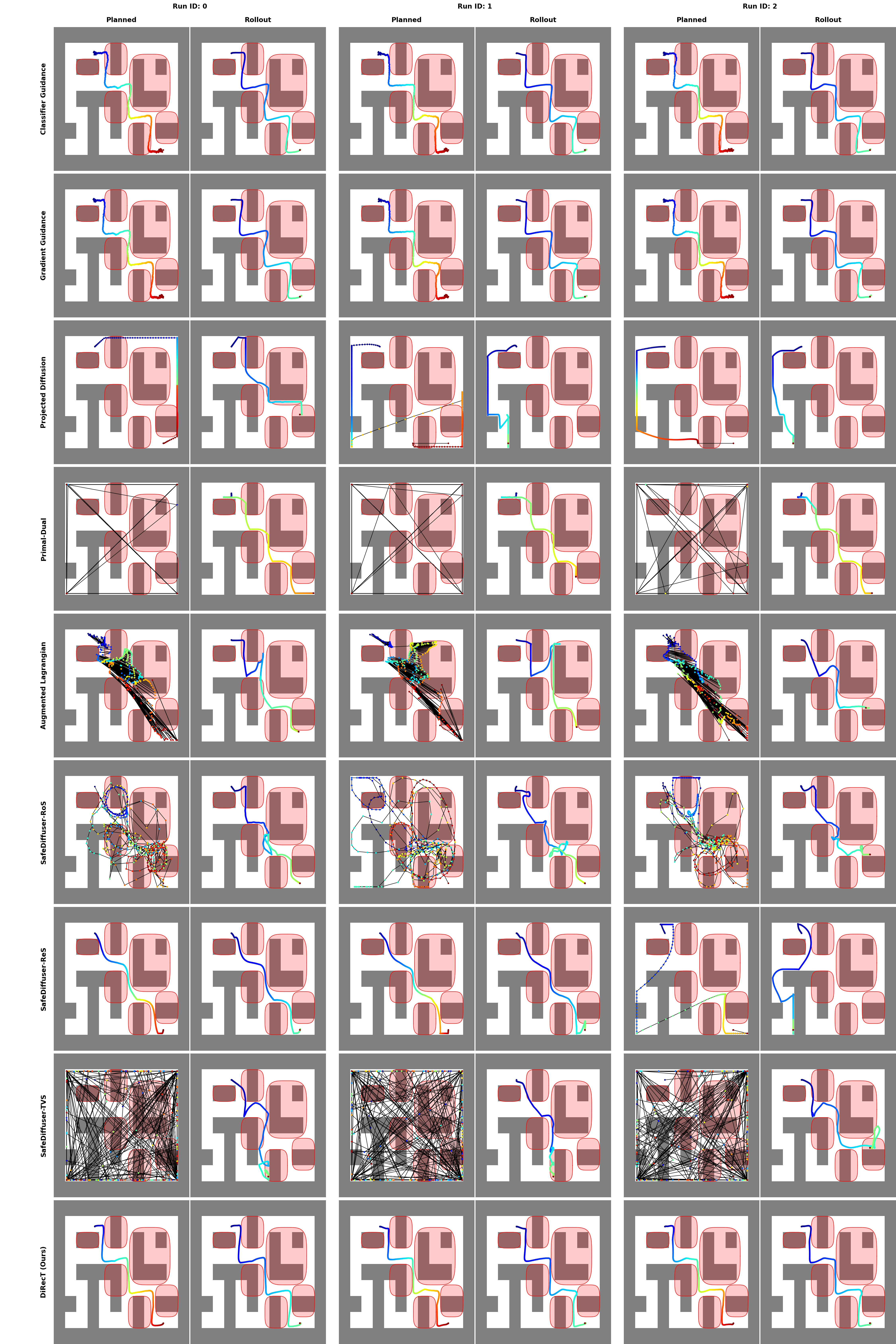}
    \caption{Comparison of \emph{planned} and \emph{executed} trajectories for safe navigation in \texttt{Maze2D-narrow}. Each row corresponds to a planning method.}
    \label{fig:maze2d narrow visualization}
\end{figure}
\begin{figure}[h]
    \centering
    \includegraphics[width=0.95\linewidth]{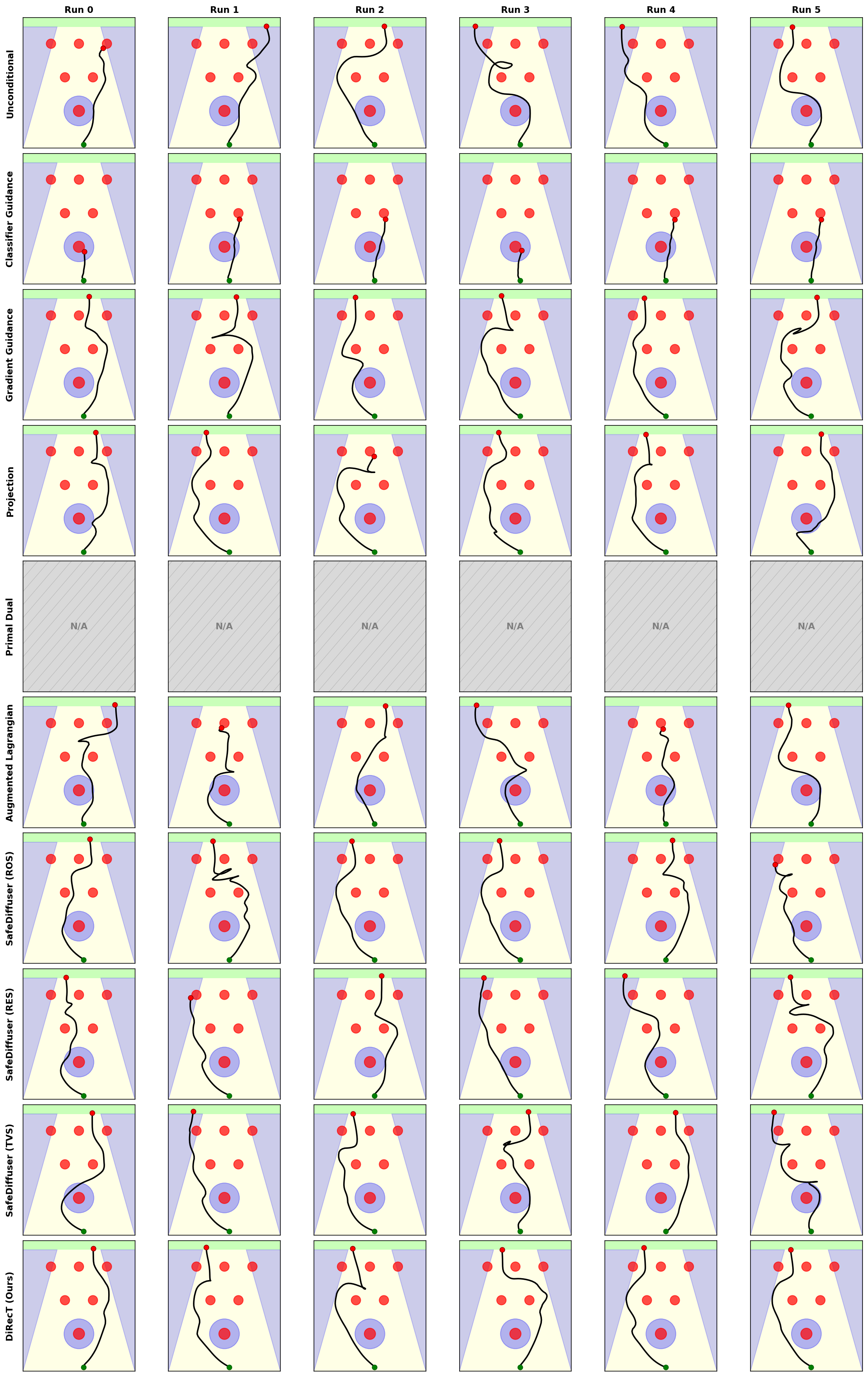}
    \caption{Comparison of environment rollouts for different sampling schemes for safe manipulation in D3IL \texttt{avoiding}.}
    \label{fig:d3il-dyn}
\end{figure}
\begin{figure}[h]
    \centering
    \includegraphics[width=\linewidth]{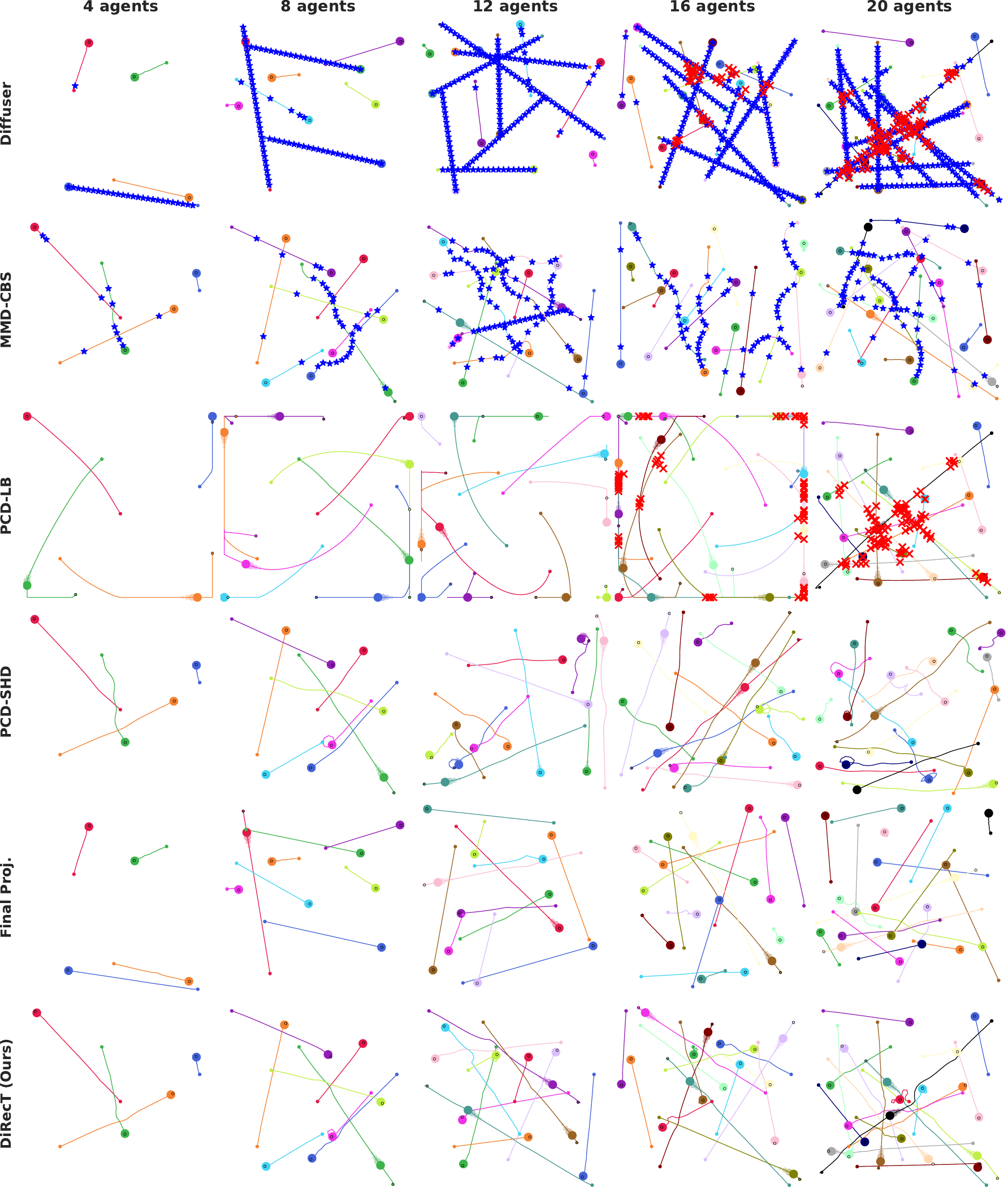}
    \caption{Visualizations of trajectories produced by different algorithms for the \texttt{Empty} MMD environment. Velocity is limited to $80\%$ of the maximum action. Robots are rewarded for moving in a straight line, while avoiding collisions (red crosses) and velocity violations (blue stars).}
    \label{fig: mrmp-empty-grid}
\end{figure}
\begin{figure}[h]
    \centering
    \includegraphics[width=\linewidth]{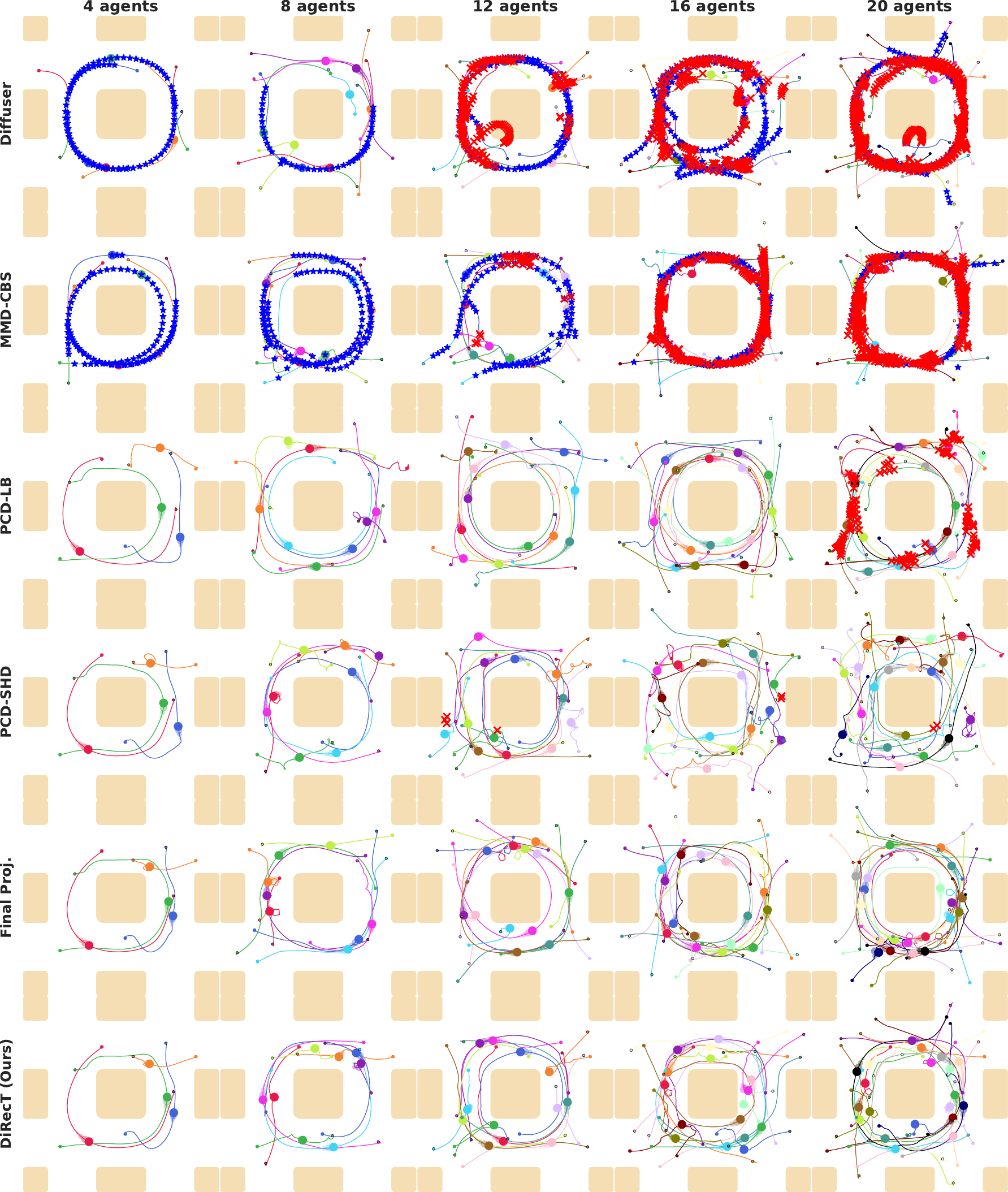}
    \caption{Visualizations of trajectories produced by different algorithms for the \texttt{Highways} MMD environment. Velocity is limited to $75\%$ of the maximum action. Robots are rewarded for moving \emph{counterclockwise} around the central obstacle, while avoiding collisions (red crosses) and velocity violations (blue stars).}
    \label{fig: mrmp-highways-grid}
\end{figure}
\begin{figure}[h]
    \centering
    \includegraphics[width=\linewidth]{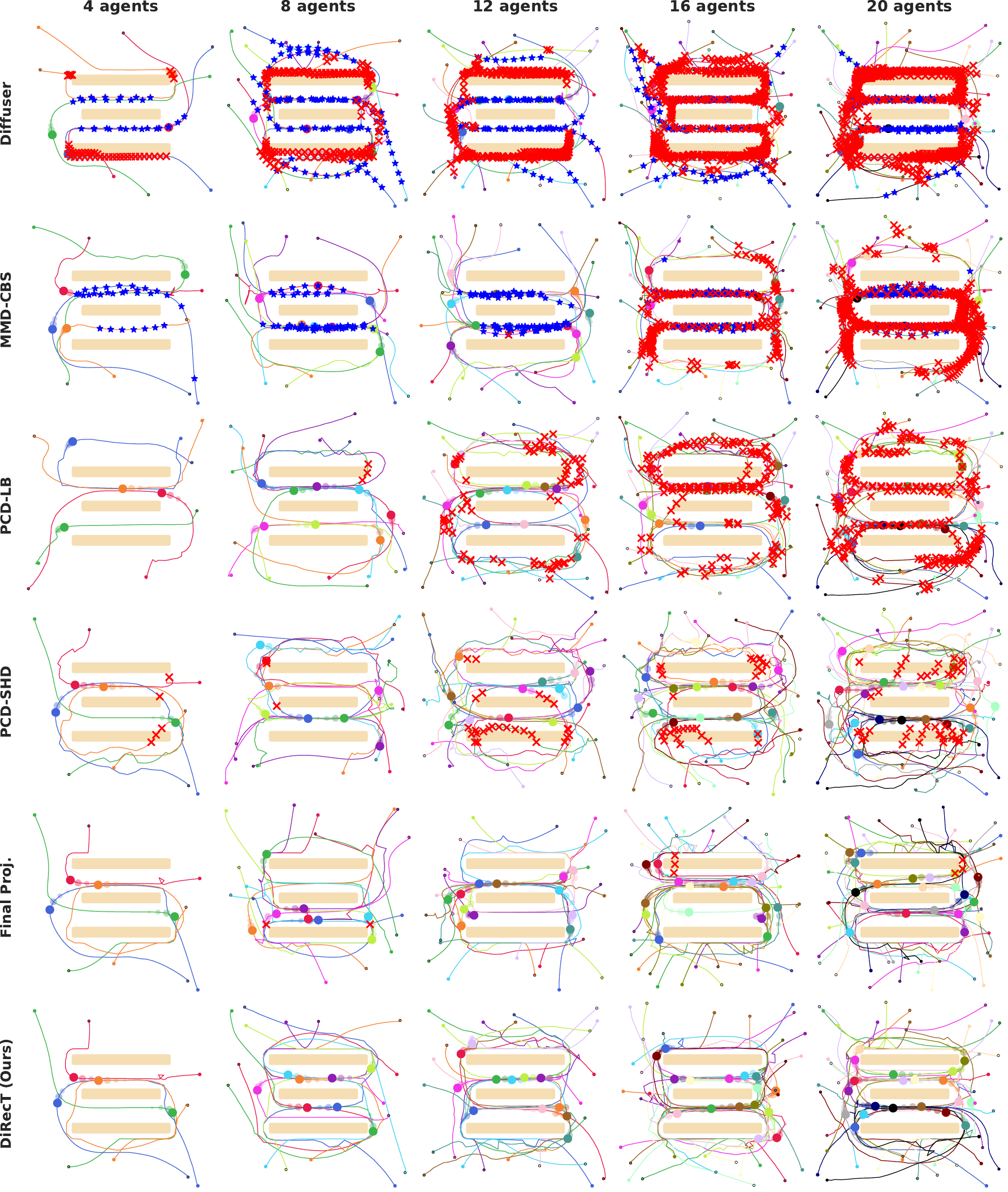}
    \caption{Visualizations of trajectories produced by different algorithms for the \texttt{Conveyor} MMD environment. Velocity is limited to $85\%$ of the maximum action. Robots are rewarded for moving \emph{leftward} in the top corridor and \emph{rightward} in the bottom one, while avoiding collisions (red crosses) and velocity violations (blue stars).}
    \label{fig: mrmp-conveyor-grid}
\end{figure}
\begin{figure}[h]
    \centering
    \includegraphics[width=\linewidth]{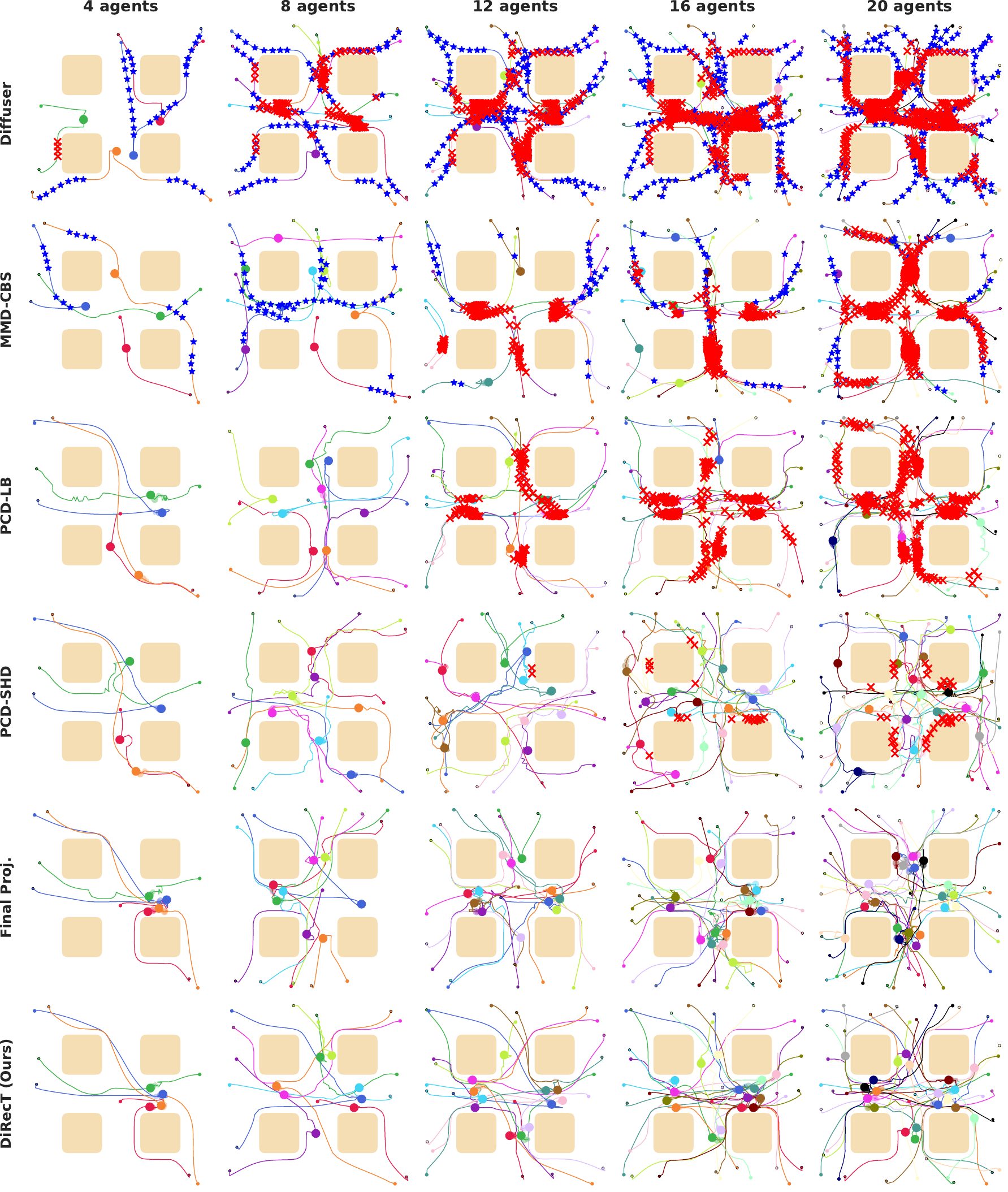}
    \caption{Visualizations of trajectories produced by different algorithms for the \texttt{Drop-Region} MMD environment. Velocity is limited to $85\%$ of the maximum action. Robots are rewarded for \emph{stopping} at the drop-off regions, while avoiding collisions (red crosses) and velocity violations (blue stars).}
    \label{fig: mrmp-dropregion-grid}
\end{figure}
\begin{figure}[h]
    \centering
    \includegraphics[width=\linewidth]{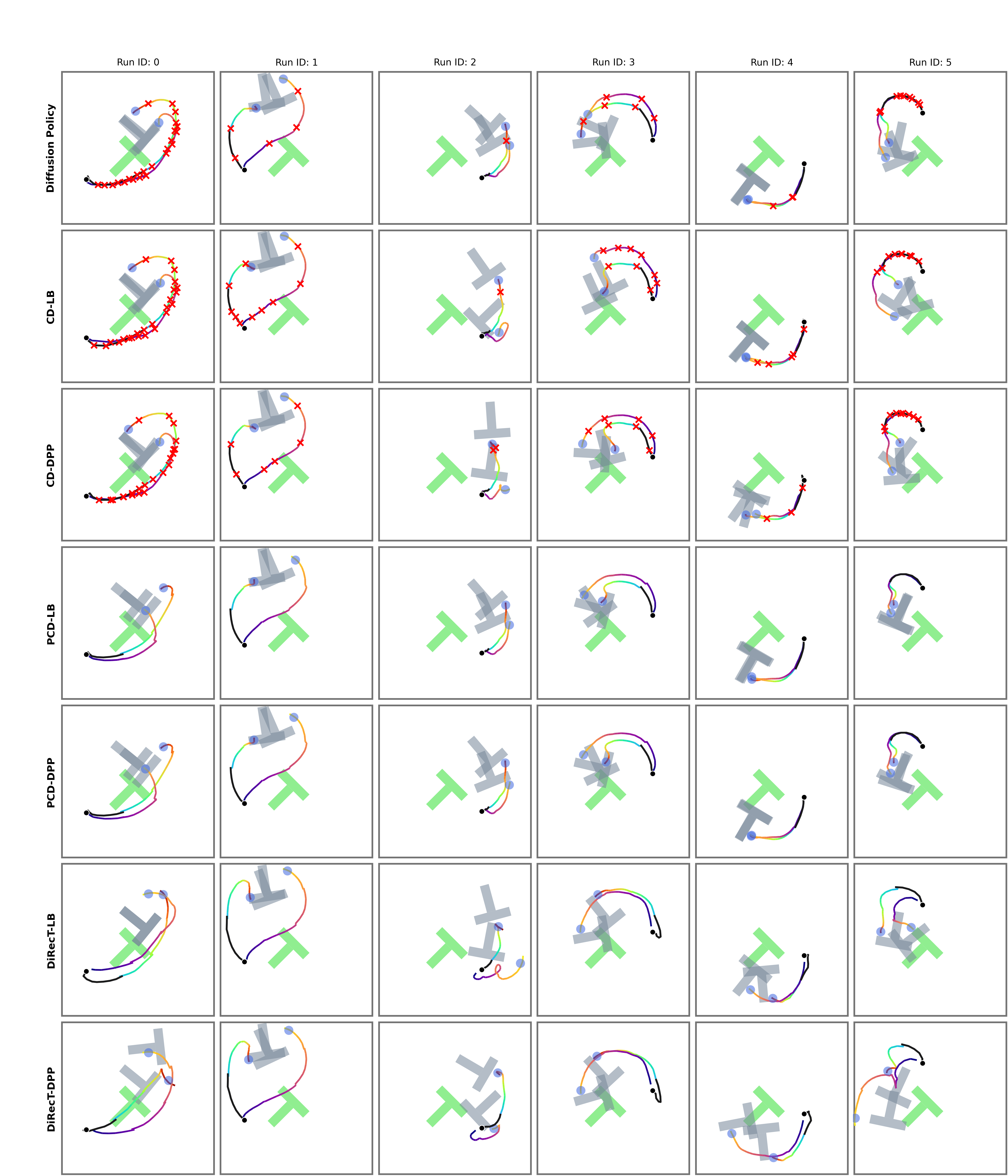}
    \caption{Visual comparison of different methods on coupled \texttt{PushT} generation with velocity constraints. Each column corresponds to the first sample of different initializations with red crosses to denote constraint violations. Only the first 64 executed steps are depicted to avoid clutter.}
    \label{fig: pusht-grid}
\end{figure}

%% file: tables/maze2d_results_additional.tex
\begin{table}[t!]
\centering
\scriptsize
\setlength{\tabcolsep}{3.2pt}
\renewcommand{\arraystretch}{1.10}
\caption{Results on the Maze2D constrained navigation task. Mean values are computed over $100$ i.i.d. samples and reported alongside their standard deviation. Gray columns report metrics evaluated on the planned trajectories before execution. Best value(s) in boldface.}
\label{tab: maze2d results additional}
\begin{adjustbox}{max width=\linewidth}
\begin{tabular}{llcccc>{\color{gray}}c>{\color{gray}}c>{\color{gray}}c>{\color{gray}}c}
\toprule
\textbf{Env}
& \textbf{Method}
& \textbf{SR} $\uparrow$
& \textbf{Viol.} $\downarrow$
& \textbf{Score} $\uparrow$
& \textbf{Time (s)} $\downarrow$
& \textbf{Gen. SR} $\uparrow$
& \textbf{Gen. Viol.} $\downarrow$
& \textbf{Gen. AS} $\downarrow$
& \textbf{Gen. CS} $\downarrow$ \\
\midrule

\multirow{10}{*}{Maze2D Broad}
& Diffuser~\citep{janner2022planningdiffusionflexiblebehavior}
& $0.000$
& $45.920{\scriptstyle \pm 11.863}$
& $\mathbf{1.601{\scriptstyle \pm 0.025}}$
& $0.316{\scriptstyle \pm 0.078}$
& $0.000{\scriptstyle \pm 0.000}$
& $49.040{\scriptstyle \pm 17.426}$
& $0.057{\scriptstyle \pm 0.005}$
& $0.636{\scriptstyle \pm 0.077}$ \\

& Classifier Guidance~\citep{dhariwal2021diffusionmodelsbeatgans}
& $0.030$
& $37.480{\scriptstyle \pm 13.260}$
& $1.587{\scriptstyle \pm 0.164}$
& $0.654{\scriptstyle \pm 0.093}$
& $0.040{\scriptstyle \pm 0.197}$
& $35.760{\scriptstyle \pm 19.482}$
& $0.055{\scriptstyle \pm 0.005}$
& $0.619{\scriptstyle \pm 0.077}$ \\

& Gradient Guidance~\citep{chung2023diffusionposterior}
& $0.040$
& $25.520{\scriptstyle \pm 10.178}$
& $1.587{\scriptstyle \pm 0.164}$
& $3.637{\scriptstyle \pm 0.076}$
& $0.040{\scriptstyle \pm 0.197}$
& $10.220{\scriptstyle \pm 7.882}$
& $0.056{\scriptstyle \pm 0.005}$
& $0.636{\scriptstyle \pm 0.073}$ \\

& Projected Diffusion~\citep{christopher2024constrained}
& $0.220$
& $42.000{\scriptstyle \pm 47.282}$
& $1.119{\scriptstyle \pm 0.708}$
& $46.398{\scriptstyle \pm 11.096}$
& $\mathbf{1.000{\scriptstyle \pm 0.000}}$
& $\mathbf{0.000{\scriptstyle \pm 0.000}}$
& $\mathbf{0.002{\scriptstyle \pm 0.000}}$
& $0.011{\scriptstyle \pm 0.004}$ \\

& Primal-Dual~\citep{zhang2025constraineddiffuserssafeplanning}
& $0.170$
& $26.610{\scriptstyle \pm 31.734}$
& $0.267{\scriptstyle \pm 0.417}$
& $5.957{\scriptstyle \pm 0.083}$
& $\mathbf{1.000{\scriptstyle \pm 0.000}}$
& $\mathbf{0.000{\scriptstyle \pm 0.000}}$
& $21.299{\scriptstyle \pm 0.219}$
& $1.861{\scriptstyle \pm 0.015}$ \\

& Augmented Lagrangian~\citep{zhang2025constraineddiffuserssafeplanning}
& $0.080$
& $72.920{\scriptstyle \pm 92.751}$
& $-0.010{\scriptstyle \pm 0.145}$
& $5.676{\scriptstyle \pm 0.566}$
& $0.000{\scriptstyle \pm 0.000}$
& $91.010{\scriptstyle \pm 55.079}$
& $2.152{\scriptstyle \pm 0.974}$
& $1.407{\scriptstyle \pm 0.452}$ \\

& SafeDiffuser-RoS~\citep{xiao2025safediffuser}
& $0.040$
& $66.260{\scriptstyle \pm 41.720}$
& $1.186{\scriptstyle \pm 0.533}$
& $148.967{\scriptstyle \pm 3.831}$
& $\mathbf{1.000{\scriptstyle \pm 0.000}}$
& $\mathbf{0.000{\scriptstyle \pm 0.000}}$
& $0.008{\scriptstyle \pm 0.007}$
& $0.010{\scriptstyle \pm 0.004}$ \\

& SafeDiffuser-ReS~\citep{xiao2025safediffuser}
& $0.520$
& $9.690{\scriptstyle \pm 11.948}$
& $0.907{\scriptstyle \pm 0.815}$
& $486.568{\scriptstyle \pm 61.073}$
& $\mathbf{1.000{\scriptstyle \pm 0.000}}$
& $\mathbf{0.000{\scriptstyle \pm 0.000}}$
& $0.003{\scriptstyle \pm 0.003}$
& $0.015{\scriptstyle \pm 0.006}$ \\

& SafeDiffuser-TVS~\citep{xiao2025safediffuser}
& $0.040$
& $62.240{\scriptstyle \pm 38.460}$
& $1.219{\scriptstyle \pm 0.561}$
& $128.758{\scriptstyle \pm 2.988}$
& $\mathbf{1.000{\scriptstyle \pm 0.000}}$
& $\mathbf{0.000{\scriptstyle \pm 0.000}}$
& $0.005{\scriptstyle \pm 0.004}$
& $0.010{\scriptstyle \pm 0.004}$ \\

& \textbf{DiRecT (Ours)}
& $\mathbf{0.970}$
& $\mathbf{0.450{\scriptstyle \pm 3.806}}$
& $1.500{\scriptstyle \pm 0.419}$
& $17.657{\scriptstyle \pm 0.629}$
& $\mathbf{1.000{\scriptstyle \pm 0.000}}$
& $\mathbf{0.000{\scriptstyle \pm 0.000}}$
& $\mathbf{0.002{\scriptstyle \pm 0.000}}$
& $\mathbf{0.008{\scriptstyle \pm 0.002}}$ \\

\midrule

\multirow{10}{*}{Maze2D Narrow}
& Diffuser~\citep{janner2022planningdiffusionflexiblebehavior}
& $0.000$
& $103.940{\scriptstyle \pm 21.881}$
& $1.601{\scriptstyle \pm 0.025}$
& $0.321{\scriptstyle \pm 0.086}$
& $0.000{\scriptstyle \pm 0.000}$
& $101.230{\scriptstyle \pm 24.551}$
& $0.057{\scriptstyle \pm 0.005}$
& $0.636{\scriptstyle \pm 0.077}$ \\

& Classifier Guidance~\citep{dhariwal2021diffusionmodelsbeatgans}
& $0.000$
& $85.280{\scriptstyle \pm 24.224}$
& $1.608{\scriptstyle \pm 0.022}$
& $0.695{\scriptstyle \pm 0.079}$
& $0.000{\scriptstyle \pm 0.000}$
& $76.120{\scriptstyle \pm 26.527}$
& $0.056{\scriptstyle \pm 0.005}$
& $0.629{\scriptstyle \pm 0.077}$ \\

& Gradient Guidance~\citep{chung2023diffusionposterior}
& $0.000$
& $51.300{\scriptstyle \pm 19.221}$
& $1.619{\scriptstyle \pm 0.029}$
& $3.608{\scriptstyle \pm 0.023}$
& $0.000{\scriptstyle \pm 0.000}$
& $23.390{\scriptstyle \pm 14.123}$
& $0.056{\scriptstyle \pm 0.005}$
& $0.647{\scriptstyle \pm 0.072}$ \\

& Projected Diffusion~\citep{christopher2024constrained}
& $0.780$
& $37.290{\scriptstyle \pm 129.876}$
& $0.163{\scriptstyle \pm 0.459}$
& $112.918{\scriptstyle \pm 19.517}$
& $0.420{\scriptstyle \pm 0.496}$
& $17.630{\scriptstyle \pm 30.680}$
& $0.018{\scriptstyle \pm 0.039}$
& $0.033{\scriptstyle \pm 0.077}$ \\

& Primal-Dual~\citep{zhang2025constraineddiffuserssafeplanning}
& $0.000$
& $55.970{\scriptstyle \pm 49.478}$
& $0.268{\scriptstyle \pm 0.418}$
& $5.154{\scriptstyle \pm 0.464}$
& $0.910{\scriptstyle \pm 0.288}$
& $0.110{\scriptstyle \pm 0.373}$
& $21.306{\scriptstyle \pm 0.226}$
& $1.861{\scriptstyle \pm 0.014}$ \\

& Augmented Lagrangian~\citep{zhang2025constraineddiffuserssafeplanning}
& $0.000$
& $217.510{\scriptstyle \pm 159.431}$
& $0.156{\scriptstyle \pm 0.449}$
& $6.022{\scriptstyle \pm 0.122}$
& $0.000{\scriptstyle \pm 0.000}$
& $197.520{\scriptstyle \pm 45.739}$
& $5.516{\scriptstyle \pm 1.301}$
& $1.765{\scriptstyle \pm 0.040}$ \\

& SafeDiffuser-RoS~\citep{xiao2025safediffuser}
& $0.000$
& $278.870{\scriptstyle \pm 209.870}$
& $0.735{\scriptstyle \pm 0.681}$
& $509.850{\scriptstyle \pm 19.929}$
& $0.000{\scriptstyle \pm 0.000}$
& $136.250{\scriptstyle \pm 39.029}$
& $0.419{\scriptstyle \pm 0.168}$
& $0.350{\scriptstyle \pm 0.148}$ \\

& SafeDiffuser-ReS~\citep{xiao2025safediffuser}
& $0.370$
& $79.090{\scriptstyle \pm 161.395}$
& $1.155{\scriptstyle \pm 0.764}$
& $1423.534{\scriptstyle \pm 41.319}$
& $0.760{\scriptstyle \pm 0.429}$
& $6.320{\scriptstyle \pm 16.578}$
& $0.031{\scriptstyle \pm 0.084}$
& $0.078{\scriptstyle \pm 0.164}$ \\

& SafeDiffuser-TVS~\citep{xiao2025safediffuser}
& $0.000$
& $315.040{\scriptstyle \pm 223.000}$
& $0.729{\scriptstyle \pm 0.694}$
& $384.860{\scriptstyle \pm 9.681}$
& $0.000{\scriptstyle \pm 0.000}$
& $35.460{\scriptstyle \pm 12.962}$
& $0.059{\scriptstyle \pm 0.026}$
& $0.038{\scriptstyle \pm 0.011}$ \\

& \textbf{DiRecT (Ours)}
& $\mathbf{0.940}$
& $\mathbf{0.230{\scriptstyle \pm 1.038}}$
& $\mathbf{1.624{\scriptstyle \pm 0.027}}$
& $98.093{\scriptstyle \pm 5.047}$
& $\mathbf{1.000{\scriptstyle \pm 0.000}}$
& $\mathbf{0.000{\scriptstyle \pm 0.000}}$
& $\mathbf{0.002{\scriptstyle \pm 0.000}}$
& $\mathbf{0.010{\scriptstyle \pm 0.006}}$ \\

\bottomrule
\end{tabular}
\end{adjustbox}
\end{table}

%% file: tables/d3il_results_additional.tex
\begin{table}[h]
\centering
\scriptsize
\setlength{\tabcolsep}{3.5pt}
\renewcommand{\arraystretch}{1.10}
\caption{Results on the D3IL avoiding constrained manipulation task. Mean values are computed over $100$ i.i.d. trials and reported with standard deviations where applicable. Best value in boldface.}
\label{tab:d3il results additional}
\begin{adjustbox}{max width=\linewidth}
\begin{tabular}{lcccc}
\toprule
\textbf{Method}
& \textbf{SR} $\uparrow$
& \textbf{Task} $\uparrow$
& \textbf{Steps (Safe)} $\downarrow$
& \textbf{Time (s)} $\downarrow$ \\
\midrule

Diffuser~\citep{janner2022planningdiffusionflexiblebehavior}
& $0.060$
& $0.960$
& $70.30{\scriptstyle \pm 11.83}$
& $0.224{\scriptstyle \pm 0.002}$ \\

Classifier Guidance~\citep{dhariwal2021diffusionmodelsbeatgans}
& $0.000$
& $0.000$
& --
& $0.627{\scriptstyle \pm 0.030}$ \\

Gradient Guidance~\citep{chung2023diffusionposterior}
& $0.880$
& $0.980$
& $62.58{\scriptstyle \pm 7.05}$
& $6.615{\scriptstyle \pm 0.540}$ \\

Projected Diffusion~\citep{christopher2024constrained}
& $0.770$
& $0.970$
& $63.77{\scriptstyle \pm 6.88}$
& $1.367{\scriptstyle \pm 0.057}$ \\

Primal-Dual~\citep{zhang2025constraineddiffuserssafeplanning}
& --
& --
& --
& -- \\

Augmented Lagrangian~\citep{zhang2025constraineddiffuserssafeplanning}
& $0.220$
& $0.840$
& $63.45{\scriptstyle \pm 7.22}$
& $4.152{\scriptstyle \pm 0.018}$ \\

SafeDiffuser-RoS~\citep{xiao2025safediffuser}
& $0.370$
& $0.820$
& $63.19{\scriptstyle \pm 11.56}$
& $18.890{\scriptstyle \pm 0.877}$ \\

SafeDiffuser-ReS~\citep{xiao2025safediffuser}
& $0.220$
& $0.860$
& $62.09{\scriptstyle \pm 9.62}$
& $12.949{\scriptstyle \pm 0.839}$ \\

SafeDiffuser-TVS~\citep{xiao2025safediffuser}
& $0.300$
& $0.850$
& $\mathbf{60.57{\scriptstyle \pm 8.49}}$
& $13.278{\scriptstyle \pm 0.488}$ \\

\textbf{DiRecT (Ours)}
& $\mathbf{1.000}$
& $\mathbf{1.000}$
& $62.86{\scriptstyle \pm 6.90}$
& $0.688{\scriptstyle \pm 0.061}$ \\

\bottomrule
\end{tabular}
\end{adjustbox}
\end{table}

%% file: tables/mrmp_results_additional.tex
\begin{table}[ht]
\centering
\scriptsize
\setlength{\tabcolsep}{2.0pt}
\renewcommand{\arraystretch}{1.05}
\caption{MRMP results on \texttt{Empty}. SR: success rate (\%), CS: constraint safety (\%), Coll: mean collisions, DA: data adherence, Time: computation time (s). Standard deviations are over different initializations and boldface represents best among the methods actively enforcing constraints.}
\label{tab:empty results additional}
\begin{adjustbox}{max width=\linewidth}
\begin{tabular}{clccccccccccccccc}
\toprule
\textbf{\# Ag.} & \textbf{Method} & \multicolumn{5}{c}{\textbf{Low ($v_{\max}$=0.675)}} & \multicolumn{5}{c}{\textbf{Medium ($v_{\max}$=0.692)}} & \multicolumn{5}{c}{\textbf{High ($v_{\max}$=0.703)}} \\
\cmidrule(lr){3-7}\cmidrule(lr){8-12}\cmidrule(lr){13-17}
 &  & \textbf{SR}$\uparrow$ & \textbf{CS}$\uparrow$ & \textbf{Coll}$\downarrow$ & \textbf{DA}$\uparrow$ & \textbf{Time (s)}$\downarrow$ & \textbf{SR}$\uparrow$ & \textbf{CS}$\uparrow$ & \textbf{Coll}$\downarrow$ & \textbf{DA}$\uparrow$ & \textbf{Time (s)}$\downarrow$ & \textbf{SR}$\uparrow$ & \textbf{CS}$\uparrow$ & \textbf{Coll}$\downarrow$ & \textbf{DA}$\uparrow$ & \textbf{Time (s)}$\downarrow$ \\
\midrule
\multirow{6}{*}{4} & Diffuser & 65.0 & 61.7${\scriptstyle \pm 47.4}$ & 1.15${\scriptstyle \pm 1.74}$ & 0.997${\scriptstyle \pm 0.018}$ & 1.0${\scriptstyle \pm 0.4}$ & 65.0 & 61.7${\scriptstyle \pm 47.4}$ & 1.15${\scriptstyle \pm 1.74}$ & 0.997${\scriptstyle \pm 0.018}$ & 1.0${\scriptstyle \pm 0.4}$ & 65.0 & 61.7${\scriptstyle \pm 47.4}$ & 1.15${\scriptstyle \pm 1.74}$ & 0.997${\scriptstyle \pm 0.018}$ & 1.0${\scriptstyle \pm 0.4}$ \\
 & MMD-CBS & 100.0 & 100.0 & 0.00 & -- & 5.5${\scriptstyle \pm 0.7}$ & 100.0 & 100.0 & 0.00 & -- & 5.5${\scriptstyle \pm 0.7}$ & 100.0 & 100.0 & 0.00 & -- & 5.5${\scriptstyle \pm 0.7}$ \\
 & PCD-LB & 85.0 & 81.4${\scriptstyle \pm 34.5}$ & 0.41${\scriptstyle \pm 0.90}$ & 0.739${\scriptstyle \pm 0.102}$ & 3.0${\scriptstyle \pm 1.5}$ & 75.0 & 73.3${\scriptstyle \pm 44.0}$ & 0.60${\scriptstyle \pm 1.08}$ & 0.984${\scriptstyle \pm 0.052}$ & 2.0 & 87.0 & 82.6${\scriptstyle \pm 33.5}$ & 0.39${\scriptstyle \pm 0.85}$ & 0.736${\scriptstyle \pm 0.102}$ & 2.8${\scriptstyle \pm 1.3}$ \\
 & PCD-SHD & 90.0 & 86.4${\scriptstyle \pm 31.8}$ & 0.27${\scriptstyle \pm 0.68}$ & 0.994${\scriptstyle \pm 0.020}$ & 2.0${\scriptstyle \pm 0.3}$ & 93.0 & 87.2${\scriptstyle \pm 30.8}$ & 0.22${\scriptstyle \pm 0.53}$ & 0.994${\scriptstyle \pm 0.020}$ & 1.9 & 95.0 & 90.0${\scriptstyle \pm 25.1}$ & 0.16${\scriptstyle \pm 0.42}$ & 0.994${\scriptstyle \pm 0.020}$ & 2.0${\scriptstyle \pm 0.1}$ \\
 & Final Proj. & 89.0 & 84.7${\scriptstyle \pm 35.1}$ & 0.26${\scriptstyle \pm 0.80}$ & 0.997${\scriptstyle \pm 0.017}$ & 2.9${\scriptstyle \pm 0.4}$ & 97.0 & 94.8${\scriptstyle \pm 20.1}$ & 0.14${\scriptstyle \pm 0.68}$ & \textbf{0.998${\scriptstyle \pm 0.015}$} & 2.3 & 97.0 & 96.1${\scriptstyle \pm 18.1}$ & 0.11${\scriptstyle \pm 0.57}$ & \textbf{0.998${\scriptstyle \pm 0.017}$} & 3.7${\scriptstyle \pm 2.3}$ \\
 & DiRecT (Ours) & \textbf{98.0} & \textbf{98.0${\scriptstyle \pm 14.0}$} & \textbf{0.05${\scriptstyle \pm 0.35}$} & \textbf{0.998${\scriptstyle \pm 0.016}$} & 31.1${\scriptstyle \pm 0.9}$ & \textbf{98.0} & \textbf{98.0${\scriptstyle \pm 14.0}$} & \textbf{0.05${\scriptstyle \pm 0.34}$} & \textbf{0.998${\scriptstyle \pm 0.015}$} & 19.8 & \textbf{99.0} & \textbf{98.2${\scriptstyle \pm 12.6}$} & \textbf{0.04${\scriptstyle \pm 0.31}$} & \textbf{0.998${\scriptstyle \pm 0.016}$} & 29.4${\scriptstyle \pm 11.9}$ \\
\midrule
\multirow{6}{*}{8} & Diffuser & 15.0 & 9.4${\scriptstyle \pm 28.2}$ & 6.20${\scriptstyle \pm 4.44}$ & 0.996${\scriptstyle \pm 0.015}$ & 0.9${\scriptstyle \pm 0.4}$ & 15.0 & 9.4${\scriptstyle \pm 28.2}$ & 6.20${\scriptstyle \pm 4.44}$ & 0.996${\scriptstyle \pm 0.015}$ & 0.9${\scriptstyle \pm 0.4}$ & 15.0 & 9.4${\scriptstyle \pm 28.2}$ & 6.20${\scriptstyle \pm 4.44}$ & 0.996${\scriptstyle \pm 0.015}$ & 0.9${\scriptstyle \pm 0.4}$ \\
 & MMD-CBS & 100.0 & 100.0 & 0.00 & -- & 13.3${\scriptstyle \pm 0.2}$ & 100.0 & 100.0 & 0.00 & -- & 13.3${\scriptstyle \pm 0.2}$ & 100.0 & 100.0 & 0.00 & -- & 13.3${\scriptstyle \pm 0.2}$ \\
 & PCD-LB & 40.5 & 22.1${\scriptstyle \pm 36.7}$ & 3.35${\scriptstyle \pm 2.85}$ & 0.744${\scriptstyle \pm 0.073}$ & 6.0${\scriptstyle \pm 2.0}$ & 25.0 & 19.7${\scriptstyle \pm 38.1}$ & 3.91${\scriptstyle \pm 3.55}$ & 0.977${\scriptstyle \pm 0.041}$ & 4.1 & 40.5 & 23.4${\scriptstyle \pm 37.8}$ & 3.06${\scriptstyle \pm 2.58}$ & 0.743${\scriptstyle \pm 0.074}$ & 5.3${\scriptstyle \pm 2.0}$ \\
 & PCD-SHD & 74.0 & 58.9${\scriptstyle \pm 44.2}$ & 1.15${\scriptstyle \pm 1.64}$ & 0.985${\scriptstyle \pm 0.022}$ & 3.9${\scriptstyle \pm 0.2}$ & 77.0 & 66.3${\scriptstyle \pm 41.8}$ & 0.89${\scriptstyle \pm 1.36}$ & 0.986${\scriptstyle \pm 0.022}$ & 3.6 & 82.0 & 68.8${\scriptstyle \pm 40.7}$ & 0.74${\scriptstyle \pm 1.12}$ & 0.986${\scriptstyle \pm 0.022}$ & 2.8${\scriptstyle \pm 0.2}$ \\
 & Final Proj. & 45.0 & 31.6${\scriptstyle \pm 43.2}$ & 2.11${\scriptstyle \pm 2.57}$ & \textbf{0.995${\scriptstyle \pm 0.016}$} & 2.7${\scriptstyle \pm 0.3}$ & 73.0 & 57.7${\scriptstyle \pm 45.0}$ & 1.22${\scriptstyle \pm 2.06}$ & 0.995${\scriptstyle \pm 0.015}$ & 2.1 & 79.0 & 66.1${\scriptstyle \pm 43.3}$ & 0.93${\scriptstyle \pm 1.69}$ & \textbf{0.995${\scriptstyle \pm 0.015}$} & 3.3${\scriptstyle \pm 1.6}$ \\
 & DiRecT (Ours) & \textbf{95.0} & \textbf{94.0${\scriptstyle \pm 22.8}$} & \textbf{0.13${\scriptstyle \pm 0.55}$} & \textbf{0.995${\scriptstyle \pm 0.015}$} & 26.5${\scriptstyle \pm 7.4}$ & \textbf{96.0} & \textbf{93.1${\scriptstyle \pm 24.1}$} & \textbf{0.13${\scriptstyle \pm 0.52}$} & \textbf{0.996${\scriptstyle \pm 0.015}$} & 23.9 & \textbf{97.0} & \textbf{94.3${\scriptstyle \pm 22.2}$} & \textbf{0.12${\scriptstyle \pm 0.53}$} & \textbf{0.995${\scriptstyle \pm 0.015}$} & 29.8${\scriptstyle \pm 12.1}$ \\
\midrule
\multirow{6}{*}{12} & Diffuser & 1.0 & 1.0${\scriptstyle \pm 9.9}$ & 14.52${\scriptstyle \pm 7.73}$ & 0.998${\scriptstyle \pm 0.006}$ & 1.0${\scriptstyle \pm 0.2}$ & 1.0 & 1.0${\scriptstyle \pm 9.9}$ & 14.52${\scriptstyle \pm 7.73}$ & 0.998${\scriptstyle \pm 0.006}$ & 1.0${\scriptstyle \pm 0.2}$ & 1.0 & 1.0${\scriptstyle \pm 9.9}$ & 14.52${\scriptstyle \pm 7.73}$ & 0.998${\scriptstyle \pm 0.006}$ & 1.0${\scriptstyle \pm 0.2}$ \\
 & MMD-CBS & 100.0 & 100.0 & 0.00 & -- & 56.9${\scriptstyle \pm 2.4}$ & 100.0 & 100.0 & 0.00 & -- & 56.9${\scriptstyle \pm 2.4}$ & 100.0 & 100.0 & 0.00 & -- & 56.9${\scriptstyle \pm 2.4}$ \\
 & PCD-LB & 3.0 & 1.3${\scriptstyle \pm 9.3}$ & 11.43${\scriptstyle \pm 5.65}$ & 0.748${\scriptstyle \pm 0.050}$ & 6.2${\scriptstyle \pm 2.6}$ & 2.0 & 2.0${\scriptstyle \pm 14.0}$ & 11.04${\scriptstyle \pm 6.40}$ & 0.987${\scriptstyle \pm 0.022}$ & 3.7 & 2.5 & 1.3${\scriptstyle \pm 9.4}$ & 10.78${\scriptstyle \pm 5.36}$ & 0.748${\scriptstyle \pm 0.050}$ & 5.1${\scriptstyle \pm 1.9}$ \\
 & PCD-SHD & 34.0 & 18.8${\scriptstyle \pm 31.9}$ & 2.97${\scriptstyle \pm 2.66}$ & 0.984${\scriptstyle \pm 0.013}$ & 3.8${\scriptstyle \pm 0.7}$ & 48.0 & 27.8${\scriptstyle \pm 35.7}$ & 2.25${\scriptstyle \pm 2.10}$ & 0.984${\scriptstyle \pm 0.013}$ & 3.9 & 62.0 & 32.6${\scriptstyle \pm 37.6}$ & 1.90${\scriptstyle \pm 1.79}$ & 0.984${\scriptstyle \pm 0.013}$ & 3.5${\scriptstyle \pm 0.2}$ \\
 & Final Proj. & 25.0 & 15.7${\scriptstyle \pm 33.8}$ & 4.80${\scriptstyle \pm 4.40}$ & \textbf{0.997${\scriptstyle \pm 0.006}$} & 4.4${\scriptstyle \pm 1.0}$ & 51.0 & 32.9${\scriptstyle \pm 44.0}$ & 2.98${\scriptstyle \pm 3.65}$ & \textbf{0.997${\scriptstyle \pm 0.006}$} & 3.6 & 62.0 & 48.9${\scriptstyle \pm 46.0}$ & 2.18${\scriptstyle \pm 3.09}$ & \textbf{0.997${\scriptstyle \pm 0.006}$} & 4.4${\scriptstyle \pm 0.9}$ \\
 & DiRecT (Ours) & \textbf{95.5} & \textbf{90.5${\scriptstyle \pm 27.7}$} & \textbf{0.24${\scriptstyle \pm 0.78}$} & \textbf{0.997${\scriptstyle \pm 0.007}$} & 42.6${\scriptstyle \pm 2.8}$ & \textbf{95.0} & \textbf{87.0${\scriptstyle \pm 30.4}$} & \textbf{0.30${\scriptstyle \pm 0.83}$} & \textbf{0.997${\scriptstyle \pm 0.007}$} & 37.0 & \textbf{95.5} & \textbf{87.3${\scriptstyle \pm 29.7}$} & \textbf{0.29${\scriptstyle \pm 0.76}$} & \textbf{0.997${\scriptstyle \pm 0.007}$} & 39.5${\scriptstyle \pm 1.4}$ \\
\midrule
\multirow{6}{*}{16} & Diffuser & 0.0 & 0.0${\scriptstyle \pm 0.0}$ & 26.17${\scriptstyle \pm 9.98}$ & 0.998${\scriptstyle \pm 0.006}$ & 1.4${\scriptstyle \pm 0.4}$ & 0.0 & 0.0${\scriptstyle \pm 0.0}$ & 26.17${\scriptstyle \pm 9.98}$ & 0.998${\scriptstyle \pm 0.006}$ & 1.4${\scriptstyle \pm 0.4}$ & 0.0 & 0.0${\scriptstyle \pm 0.0}$ & 26.17${\scriptstyle \pm 9.98}$ & 0.998${\scriptstyle \pm 0.006}$ & 1.4${\scriptstyle \pm 0.4}$ \\
 & MMD-CBS & 100.0 & 100.0 & 0.00 & -- & 48.0${\scriptstyle \pm 3.3}$ & 100.0 & 100.0 & 0.00 & -- & 48.0${\scriptstyle \pm 3.3}$ & 100.0 & 100.0 & 0.00 & -- & 48.0${\scriptstyle \pm 3.3}$ \\
 & PCD-LB & 0.0 & 0.0${\scriptstyle \pm 0.0}$ & 30.57${\scriptstyle \pm 9.83}$ & 0.740${\scriptstyle \pm 0.044}$ & 8.7${\scriptstyle \pm 3.8}$ & 0.0 & 0.0${\scriptstyle \pm 0.0}$ & 21.22${\scriptstyle \pm 8.61}$ & 0.986${\scriptstyle \pm 0.023}$ & 5.4 & 0.0 & 0.0${\scriptstyle \pm 0.0}$ & 29.78${\scriptstyle \pm 9.75}$ & 0.740${\scriptstyle \pm 0.044}$ & 6.7${\scriptstyle \pm 2.0}$ \\
 & PCD-SHD & 19.0 & 7.0${\scriptstyle \pm 19.0}$ & 6.18${\scriptstyle \pm 3.71}$ & 0.979${\scriptstyle \pm 0.018}$ & 6.7${\scriptstyle \pm 0.9}$ & 29.0 & 10.4${\scriptstyle \pm 23.9}$ & 4.85${\scriptstyle \pm 3.13}$ & 0.979${\scriptstyle \pm 0.017}$ & 5.2 & 29.0 & 12.4${\scriptstyle \pm 25.1}$ & 4.12${\scriptstyle \pm 2.76}$ & 0.979${\scriptstyle \pm 0.017}$ & 4.9${\scriptstyle \pm 0.3}$ \\
 & Final Proj. & 5.0 & 2.6${\scriptstyle \pm 11.9}$ & 9.16${\scriptstyle \pm 5.09}$ & \textbf{0.995${\scriptstyle \pm 0.007}$} & 6.4${\scriptstyle \pm 1.1}$ & 22.0 & 12.1${\scriptstyle \pm 28.6}$ & 5.66${\scriptstyle \pm 4.07}$ & 0.995${\scriptstyle \pm 0.007}$ & 5.7 & 36.5 & 20.1${\scriptstyle \pm 35.2}$ & 4.23${\scriptstyle \pm 3.51}$ & 0.995${\scriptstyle \pm 0.007}$ & 5.7${\scriptstyle \pm 0.3}$ \\
 & DiRecT (Ours) & \textbf{90.0} & \textbf{80.4${\scriptstyle \pm 35.4}$} & \textbf{0.57${\scriptstyle \pm 1.30}$} & \textbf{0.995${\scriptstyle \pm 0.007}$} & 58.2${\scriptstyle \pm 3.7}$ & \textbf{93.0} & \textbf{78.7${\scriptstyle \pm 37.4}$} & \textbf{0.54${\scriptstyle \pm 1.16}$} & \textbf{0.996${\scriptstyle \pm 0.007}$} & 59.1 & \textbf{93.0} & \textbf{81.8${\scriptstyle \pm 34.5}$} & \textbf{0.42${\scriptstyle \pm 0.95}$} & \textbf{0.996${\scriptstyle \pm 0.007}$} & 60.9${\scriptstyle \pm 0.9}$ \\
\midrule
\multirow{6}{*}{20} & Diffuser & 0.0 & 0.0${\scriptstyle \pm 0.0}$ & 40.39${\scriptstyle \pm 14.13}$ & 0.998${\scriptstyle \pm 0.006}$ & 1.7${\scriptstyle \pm 0.4}$ & 0.0 & 0.0${\scriptstyle \pm 0.0}$ & 40.39${\scriptstyle \pm 14.13}$ & 0.998${\scriptstyle \pm 0.006}$ & 1.7${\scriptstyle \pm 0.4}$ & 0.0 & 0.0${\scriptstyle \pm 0.0}$ & 40.39${\scriptstyle \pm 14.13}$ & 0.998${\scriptstyle \pm 0.006}$ & 1.7${\scriptstyle \pm 0.4}$ \\
 & MMD-CBS & 33.3 & 46.7 & 14.67 & -- & 69.2${\scriptstyle \pm 8.5}$ & 33.3 & 46.7 & 14.67 & -- & 69.2${\scriptstyle \pm 8.5}$ & 33.3 & 46.7 & 14.67 & -- & 69.2${\scriptstyle \pm 8.5}$ \\
 & PCD-LB & 0.0 & 0.0${\scriptstyle \pm 0.0}$ & 34.43${\scriptstyle \pm 12.76}$ & 0.988${\scriptstyle \pm 0.018}$ & 8.7 & 0.0 & 0.0${\scriptstyle \pm 0.0}$ & 33.09${\scriptstyle \pm 12.23}$ & 0.988${\scriptstyle \pm 0.018}$ & 8.0 & 0.0 & 0.0${\scriptstyle \pm 0.0}$ & 32.03${\scriptstyle \pm 11.73}$ & 0.988${\scriptstyle \pm 0.018}$ & 7.6 \\
 & PCD-SHD & 7.0 & 2.7${\scriptstyle \pm 12.5}$ & 9.07${\scriptstyle \pm 4.94}$ & 0.973${\scriptstyle \pm 0.018}$ & 7.1 & 11.0 & 4.1${\scriptstyle \pm 16.3}$ & 7.20${\scriptstyle \pm 4.10}$ & 0.973${\scriptstyle \pm 0.018}$ & 8.1 & 19.0 & 5.5${\scriptstyle \pm 18.4}$ & 6.07${\scriptstyle \pm 3.45}$ & 0.973${\scriptstyle \pm 0.018}$ & 8.0 \\
 & Final Proj. & 1.0 & 0.0${\scriptstyle \pm 0.4}$ & 13.54${\scriptstyle \pm 7.95}$ & 0.995${\scriptstyle \pm 0.007}$ & 9.3${\scriptstyle \pm 2.8}$ & 13.0 & 4.7${\scriptstyle \pm 16.5}$ & 8.24${\scriptstyle \pm 6.38}$ & 0.995${\scriptstyle \pm 0.007}$ & 8.1 & 23.0 & 11.3${\scriptstyle \pm 25.5}$ & 6.37${\scriptstyle \pm 5.59}$ & 0.995${\scriptstyle \pm 0.007}$ & 8.1${\scriptstyle \pm 0.2}$ \\
 & DiRecT (Ours) & \textbf{90.0} & \textbf{76.1${\scriptstyle \pm 39.0}$} & \textbf{0.82${\scriptstyle \pm 1.88}$} & \textbf{0.996${\scriptstyle \pm 0.007}$} & 80.7 & \textbf{87.0} & \textbf{77.5${\scriptstyle \pm 37.5}$} & \textbf{0.76${\scriptstyle \pm 1.71}$} & \textbf{0.996${\scriptstyle \pm 0.007}$} & 87.2 & \textbf{90.0} & \textbf{80.9${\scriptstyle \pm 35.3}$} & \textbf{0.62${\scriptstyle \pm 1.48}$} & \textbf{0.996${\scriptstyle \pm 0.007}$} & 89.6 \\
\bottomrule
\end{tabular}
\end{adjustbox}
\end{table}

\begin{table}[ht]
\centering
\scriptsize
\setlength{\tabcolsep}{2.0pt}
\renewcommand{\arraystretch}{1.05}
\caption{MRMP results on \texttt{Highways}. SR: success rate (\%), CS: constraint safety (\%), Coll: mean collisions, DA: data adherence, Time: computation time (s). Standard deviations are over different initializations and boldface represents best among the methods actively enforcing constraints.}
\label{tab:highways results additional}
\begin{adjustbox}{max width=\linewidth}
\begin{tabular}{clccccccccccccccc}
\toprule
\textbf{\# Ag.} & \textbf{Method} & \multicolumn{5}{c}{\textbf{Low ($v_{\max}$=0.647)}} & \multicolumn{5}{c}{\textbf{Medium ($v_{\max}$=0.781)}} & \multicolumn{5}{c}{\textbf{High ($v_{\max}$=0.878)}} \\
\cmidrule(lr){3-7}\cmidrule(lr){8-12}\cmidrule(lr){13-17}
 &  & \textbf{SR}$\uparrow$ & \textbf{CS}$\uparrow$ & \textbf{Coll}$\downarrow$ & \textbf{DA}$\uparrow$ & \textbf{Time (s)}$\downarrow$ & \textbf{SR}$\uparrow$ & \textbf{CS}$\uparrow$ & \textbf{Coll}$\downarrow$ & \textbf{DA}$\uparrow$ & \textbf{Time (s)}$\downarrow$ & \textbf{SR}$\uparrow$ & \textbf{CS}$\uparrow$ & \textbf{Coll}$\downarrow$ & \textbf{DA}$\uparrow$ & \textbf{Time (s)}$\downarrow$ \\
\midrule
\multirow{6}{*}{4} & Diffuser & 54.0 & 16.2${\scriptstyle \pm 30.5}$ & 8.91${\scriptstyle \pm 7.89}$ & 0.994${\scriptstyle \pm 0.034}$ & 0.5${\scriptstyle \pm 0.2}$ & 54.0 & 16.2${\scriptstyle \pm 30.5}$ & 8.91${\scriptstyle \pm 7.89}$ & 0.994${\scriptstyle \pm 0.034}$ & 0.5${\scriptstyle \pm 0.2}$ & 54.0 & 16.2${\scriptstyle \pm 30.5}$ & 8.91${\scriptstyle \pm 7.89}$ & 0.994${\scriptstyle \pm 0.034}$ & 0.5${\scriptstyle \pm 0.2}$ \\
 & MMD-CBS & 100.0 & 100.0 & 0.00 & -- & 5.2${\scriptstyle \pm 0.6}$ & 100.0 & 100.0 & 0.00 & -- & 5.2${\scriptstyle \pm 0.6}$ & 100.0 & 100.0 & 0.00 & -- & 5.2${\scriptstyle \pm 0.6}$ \\
 & PCD-LB & 94.0 & 60.4${\scriptstyle \pm 38.8}$ & 0.82${\scriptstyle \pm 1.20}$ & 0.977${\scriptstyle \pm 0.066}$ & 1.9 & 98.0 & 69.6${\scriptstyle \pm 32.3}$ & 0.82${\scriptstyle \pm 1.12}$ & 0.980${\scriptstyle \pm 0.062}$ & 1.9 & 48.5 & 34.2${\scriptstyle \pm 15.8}$ & 0.44${\scriptstyle \pm 0.59}$ & 0.938${\scriptstyle \pm 0.085}$ & 1.8${\scriptstyle \pm 0.1}$ \\
 & PCD-SHD & 96.0 & 78.1${\scriptstyle \pm 31.9}$ & 0.07${\scriptstyle \pm 0.20}$ & 0.978${\scriptstyle \pm 0.064}$ & 1.8 & \textbf{100.0} & 93.6${\scriptstyle \pm 9.6}$ & 0.03${\scriptstyle \pm 0.04}$ & 0.980${\scriptstyle \pm 0.060}$ & 2.0 & \textbf{100.0} & 94.9${\scriptstyle \pm 8.3}$ & 0.02${\scriptstyle \pm 0.05}$ & 0.981${\scriptstyle \pm 0.058}$ & 1.8${\scriptstyle \pm 0.0}$ \\
 & Final Proj. & 99.0 & 76.7${\scriptstyle \pm 30.4}$ & 0.72${\scriptstyle \pm 1.52}$ & \textbf{0.994${\scriptstyle \pm 0.034}$} & 2.5 & \textbf{100.0} & 99.5${\scriptstyle \pm 3.8}$ & 0.01${\scriptstyle \pm 0.08}$ & \textbf{0.994${\scriptstyle \pm 0.034}$} & 2.5 & \textbf{100.0} & \textbf{100.0${\scriptstyle \pm 0.0}$} & \textbf{0.00${\scriptstyle \pm 0.00}$} & \textbf{0.994${\scriptstyle \pm 0.034}$} & 4.4${\scriptstyle \pm 2.5}$ \\
 & DiRecT (Ours) & \textbf{100.0} & \textbf{100.0${\scriptstyle \pm 0.0}$} & \textbf{0.00${\scriptstyle \pm 0.00}$} & 0.993${\scriptstyle \pm 0.035}$ & 28.5 & \textbf{100.0} & \textbf{100.0${\scriptstyle \pm 0.0}$} & \textbf{0.00${\scriptstyle \pm 0.00}$} & 0.993${\scriptstyle \pm 0.035}$ & 28.6 & \textbf{100.0} & \textbf{100.0${\scriptstyle \pm 0.0}$} & \textbf{0.00${\scriptstyle \pm 0.00}$} & 0.993${\scriptstyle \pm 0.035}$ & 49.2${\scriptstyle \pm 27.6}$ \\
\midrule
\multirow{6}{*}{8} & Diffuser & 2.0 & 0.1${\scriptstyle \pm 0.9}$ & 39.77${\scriptstyle \pm 18.43}$ & 0.992${\scriptstyle \pm 0.025}$ & 0.5${\scriptstyle \pm 0.2}$ & 2.0 & 0.1${\scriptstyle \pm 0.9}$ & 39.77${\scriptstyle \pm 18.43}$ & 0.992${\scriptstyle \pm 0.025}$ & 0.5${\scriptstyle \pm 0.2}$ & 2.0 & 0.1${\scriptstyle \pm 0.9}$ & 39.77${\scriptstyle \pm 18.43}$ & 0.992${\scriptstyle \pm 0.025}$ & 0.5${\scriptstyle \pm 0.2}$ \\
 & MMD-CBS & 100.0 & 100.0 & 0.00 & -- & 42.7${\scriptstyle \pm 2.0}$ & 100.0 & 100.0 & 0.00 & -- & 42.7${\scriptstyle \pm 2.0}$ & 100.0 & 100.0 & 0.00 & -- & 42.7${\scriptstyle \pm 2.0}$ \\
 & PCD-LB & 63.0 & 11.1${\scriptstyle \pm 18.6}$ & 5.12${\scriptstyle \pm 3.42}$ & 0.972${\scriptstyle \pm 0.047}$ & 2.5 & 72.0 & 12.3${\scriptstyle \pm 20.5}$ & 6.16${\scriptstyle \pm 3.28}$ & 0.975${\scriptstyle \pm 0.043}$ & 2.5 & 27.0 & 3.3${\scriptstyle \pm 5.9}$ & 3.83${\scriptstyle \pm 2.00}$ & 0.949${\scriptstyle \pm 0.056}$ & 2.5${\scriptstyle \pm 0.4}$ \\
 & PCD-SHD & 91.0 & 38.5${\scriptstyle \pm 34.3}$ & 0.37${\scriptstyle \pm 0.47}$ & 0.973${\scriptstyle \pm 0.045}$ & 2.4 & \textbf{100.0} & 64.3${\scriptstyle \pm 26.5}$ & 0.23${\scriptstyle \pm 0.29}$ & 0.976${\scriptstyle \pm 0.041}$ & 2.4 & \textbf{100.0} & 69.2${\scriptstyle \pm 23.0}$ & 0.15${\scriptstyle \pm 0.17}$ & 0.977${\scriptstyle \pm 0.040}$ & 2.3${\scriptstyle \pm 0.1}$ \\
 & Final Proj. & 88.0 & 30.9${\scriptstyle \pm 32.1}$ & 4.22${\scriptstyle \pm 4.10}$ & \textbf{0.992${\scriptstyle \pm 0.025}$} & 3.5 & \textbf{100.0} & 98.1${\scriptstyle \pm 5.6}$ & 0.07${\scriptstyle \pm 0.27}$ & \textbf{0.992${\scriptstyle \pm 0.025}$} & 3.5 & \textbf{100.0} & 99.8${\scriptstyle \pm 1.1}$ & 0.01${\scriptstyle \pm 0.07}$ & \textbf{0.992${\scriptstyle \pm 0.025}$} & 4.9${\scriptstyle \pm 1.9}$ \\
 & DiRecT (Ours) & \textbf{100.0} & \textbf{100.0${\scriptstyle \pm 0.1}$} & \textbf{0.00${\scriptstyle \pm 0.01}$} & 0.991${\scriptstyle \pm 0.026}$ & 40.4 & \textbf{100.0} & \textbf{100.0${\scriptstyle \pm 0.0}$} & \textbf{0.00${\scriptstyle \pm 0.00}$} & \textbf{0.992${\scriptstyle \pm 0.026}$} & 40.3 & \textbf{100.0} & \textbf{100.0${\scriptstyle \pm 0.0}$} & \textbf{0.00${\scriptstyle \pm 0.00}$} & \textbf{0.992${\scriptstyle \pm 0.026}$} & 54.3${\scriptstyle \pm 19.9}$ \\
\midrule
\multirow{6}{*}{12} & Diffuser & 0.0 & 0.0${\scriptstyle \pm 0.0}$ & 91.38${\scriptstyle \pm 23.06}$ & 0.991${\scriptstyle \pm 0.022}$ & 0.6${\scriptstyle \pm 0.3}$ & 0.0 & 0.0${\scriptstyle \pm 0.0}$ & 91.38${\scriptstyle \pm 23.06}$ & 0.991${\scriptstyle \pm 0.022}$ & 0.6${\scriptstyle \pm 0.3}$ & 0.0 & 0.0${\scriptstyle \pm 0.0}$ & 91.38${\scriptstyle \pm 23.06}$ & 0.991${\scriptstyle \pm 0.022}$ & 0.6${\scriptstyle \pm 0.3}$ \\
 & MMD-CBS & 33.3 & 63.9 & 3.67 & -- & 59.5${\scriptstyle \pm 2.4}$ & 33.3 & 63.9 & 3.67 & -- & 59.5${\scriptstyle \pm 2.4}$ & 33.3 & 63.9 & 3.67 & -- & 59.5${\scriptstyle \pm 2.4}$ \\
 & PCD-LB & 9.0 & 0.5${\scriptstyle \pm 3.1}$ & 19.54${\scriptstyle \pm 8.23}$ & 0.971${\scriptstyle \pm 0.044}$ & 3.6 & 1.0 & 0.0${\scriptstyle \pm 0.1}$ & 25.21${\scriptstyle \pm 8.56}$ & 0.974${\scriptstyle \pm 0.041}$ & 2.8 & 0.0 & 0.0${\scriptstyle \pm 0.0}$ & 17.40${\scriptstyle \pm 5.68}$ & 0.952${\scriptstyle \pm 0.044}$ & 2.9${\scriptstyle \pm 0.4}$ \\
 & PCD-SHD & 78.0 & 11.9${\scriptstyle \pm 19.2}$ & 1.05${\scriptstyle \pm 0.90}$ & 0.972${\scriptstyle \pm 0.042}$ & 3.1 & 99.0 & 26.1${\scriptstyle \pm 21.2}$ & 0.63${\scriptstyle \pm 0.51}$ & 0.975${\scriptstyle \pm 0.039}$ & 2.9 & \textbf{100.0} & 30.3${\scriptstyle \pm 20.3}$ & 0.44${\scriptstyle \pm 0.33}$ & 0.976${\scriptstyle \pm 0.038}$ & 2.6${\scriptstyle \pm 0.1}$ \\
 & Final Proj. & 55.0 & 7.5${\scriptstyle \pm 15.7}$ & 11.73${\scriptstyle \pm 6.72}$ & \textbf{0.991${\scriptstyle \pm 0.022}$} & 5.4 & \textbf{100.0} & 92.4${\scriptstyle \pm 10.7}$ & 0.31${\scriptstyle \pm 0.51}$ & \textbf{0.991${\scriptstyle \pm 0.022}$} & 5.4 & \textbf{100.0} & 99.7${\scriptstyle \pm 0.6}$ & 0.01${\scriptstyle \pm 0.02}$ & \textbf{0.991${\scriptstyle \pm 0.022}$} & 5.8${\scriptstyle \pm 0.6}$ \\
 & DiRecT (Ours) & \textbf{100.0} & \textbf{99.7${\scriptstyle \pm 1.4}$} & \textbf{0.01${\scriptstyle \pm 0.03}$} & 0.990${\scriptstyle \pm 0.023}$ & 65.0 & \textbf{100.0} & \textbf{100.0${\scriptstyle \pm 0.0}$} & \textbf{0.00${\scriptstyle \pm 0.00}$} & 0.990${\scriptstyle \pm 0.023}$ & 64.7 & \textbf{100.0} & \textbf{100.0${\scriptstyle \pm 0.0}$} & \textbf{0.00${\scriptstyle \pm 0.00}$} & 0.990${\scriptstyle \pm 0.023}$ & 65.8${\scriptstyle \pm 2.0}$ \\
\midrule
\multirow{6}{*}{16} & Diffuser & 0.0 & 0.0${\scriptstyle \pm 0.0}$ & 165.86${\scriptstyle \pm 31.12}$ & 0.995${\scriptstyle \pm 0.014}$ & 0.8${\scriptstyle \pm 0.3}$ & 0.0 & 0.0${\scriptstyle \pm 0.0}$ & 165.86${\scriptstyle \pm 31.12}$ & 0.995${\scriptstyle \pm 0.014}$ & 0.8${\scriptstyle \pm 0.3}$ & 0.0 & 0.0${\scriptstyle \pm 0.0}$ & 165.86${\scriptstyle \pm 31.12}$ & 0.995${\scriptstyle \pm 0.014}$ & 0.8${\scriptstyle \pm 0.3}$ \\
 & MMD-CBS & 0.0 & 6.2 & 27.00 & -- & 72.4${\scriptstyle \pm 4.1}$ & 0.0 & 6.2 & 27.00 & -- & 72.4${\scriptstyle \pm 4.1}$ & 0.0 & 6.2 & 27.00 & -- & 72.4${\scriptstyle \pm 4.1}$ \\
 & PCD-LB & 1.0 & 0.1${\scriptstyle \pm 0.6}$ & 45.62${\scriptstyle \pm 12.48}$ & 0.973${\scriptstyle \pm 0.033}$ & 3.8 & 0.0 & 0.0${\scriptstyle \pm 0.0}$ & 56.59${\scriptstyle \pm 12.29}$ & 0.976${\scriptstyle \pm 0.031}$ & 3.7 & 0.0 & 0.0${\scriptstyle \pm 0.0}$ & 43.82${\scriptstyle \pm 9.83}$ & 0.956${\scriptstyle \pm 0.039}$ & 3.8${\scriptstyle \pm 0.7}$ \\
 & PCD-SHD & 41.0 & 3.9${\scriptstyle \pm 10.0}$ & 2.16${\scriptstyle \pm 1.43}$ & 0.975${\scriptstyle \pm 0.031}$ & 3.8 & 80.0 & 8.6${\scriptstyle \pm 12.3}$ & 1.49${\scriptstyle \pm 1.01}$ & 0.978${\scriptstyle \pm 0.029}$ & 3.6 & 87.0 & 9.1${\scriptstyle \pm 11.5}$ & 1.06${\scriptstyle \pm 0.66}$ & 0.979${\scriptstyle \pm 0.028}$ & 3.4${\scriptstyle \pm 0.1}$ \\
 & Final Proj. & 19.0 & 2.0${\scriptstyle \pm 9.2}$ & 27.99${\scriptstyle \pm 14.18}$ & \textbf{0.994${\scriptstyle \pm 0.014}$} & 7.7 & \textbf{100.0} & 74.5${\scriptstyle \pm 23.2}$ & 1.34${\scriptstyle \pm 1.96}$ & \textbf{0.994${\scriptstyle \pm 0.014}$} & 7.8 & \textbf{100.0} & 98.7${\scriptstyle \pm 4.0}$ & 0.04${\scriptstyle \pm 0.14}$ & \textbf{0.994${\scriptstyle \pm 0.014}$} & 8.2${\scriptstyle \pm 0.6}$ \\
 & DiRecT (Ours) & \textbf{100.0} & \textbf{97.7${\scriptstyle \pm 7.7}$} & \textbf{0.07${\scriptstyle \pm 0.32}$} & \textbf{0.994${\scriptstyle \pm 0.015}$} & 93.6 & \textbf{100.0} & \textbf{99.9${\scriptstyle \pm 1.2}$} & \textbf{0.00${\scriptstyle \pm 0.04}$} & \textbf{0.994${\scriptstyle \pm 0.015}$} & 93.8 & \textbf{100.0} & \textbf{100.0${\scriptstyle \pm 0.0}$} & \textbf{0.00${\scriptstyle \pm 0.00}$} & \textbf{0.994${\scriptstyle \pm 0.015}$} & 94.4${\scriptstyle \pm 1.7}$ \\
\midrule
\multirow{6}{*}{20} & Diffuser & 0.0 & 0.0${\scriptstyle \pm 0.0}$ & 272.99${\scriptstyle \pm 46.77}$ & 0.992${\scriptstyle \pm 0.017}$ & 0.8${\scriptstyle \pm 0.3}$ & 0.0 & 0.0${\scriptstyle \pm 0.0}$ & 272.99${\scriptstyle \pm 46.77}$ & 0.992${\scriptstyle \pm 0.017}$ & 0.8${\scriptstyle \pm 0.3}$ & 0.0 & 0.0${\scriptstyle \pm 0.0}$ & 272.99${\scriptstyle \pm 46.77}$ & 0.992${\scriptstyle \pm 0.017}$ & 0.8${\scriptstyle \pm 0.3}$ \\
 & MMD-CBS & 0.0 & 0.0 & 35.00 & -- & 73.8${\scriptstyle \pm 5.2}$ & 0.0 & 0.0 & 35.00 & -- & 73.8${\scriptstyle \pm 5.2}$ & 0.0 & 0.0 & 35.00 & -- & 73.8${\scriptstyle \pm 5.2}$ \\
 & PCD-LB & 0.0 & 0.0${\scriptstyle \pm 0.0}$ & 85.17${\scriptstyle \pm 18.69}$ & 0.973${\scriptstyle \pm 0.032}$ & 5.2 & 0.0 & 0.0${\scriptstyle \pm 0.0}$ & 105.16${\scriptstyle \pm 19.39}$ & 0.976${\scriptstyle \pm 0.030}$ & 4.3 & 0.0 & 0.0${\scriptstyle \pm 0.0}$ & 120.10${\scriptstyle \pm 19.24}$ & 0.978${\scriptstyle \pm 0.028}$ & 3.8 \\
 & PCD-SHD & 11.0 & 0.2${\scriptstyle \pm 0.7}$ & 4.07${\scriptstyle \pm 1.96}$ & 0.974${\scriptstyle \pm 0.030}$ & 4.8 & 39.0 & 1.0${\scriptstyle \pm 2.3}$ & 2.63${\scriptstyle \pm 1.22}$ & 0.977${\scriptstyle \pm 0.028}$ & 4.2 & 43.0 & 1.2${\scriptstyle \pm 2.5}$ & 1.94${\scriptstyle \pm 0.85}$ & 0.978${\scriptstyle \pm 0.027}$ & 3.8 \\
 & Final Proj. & 1.0 & 0.0${\scriptstyle \pm 0.2}$ & 52.60${\scriptstyle \pm 22.60}$ & \textbf{0.992${\scriptstyle \pm 0.018}$} & 10.4 & \textbf{100.0} & 49.9${\scriptstyle \pm 30.6}$ & 4.15${\scriptstyle \pm 4.45}$ & \textbf{0.992${\scriptstyle \pm 0.018}$} & 10.4 & \textbf{100.0} & 96.1${\scriptstyle \pm 5.8}$ & 0.14${\scriptstyle \pm 0.25}$ & \textbf{0.992${\scriptstyle \pm 0.018}$} & 10.8${\scriptstyle \pm 0.6}$ \\
 & DiRecT (Ours) & \textbf{100.0} & \textbf{95.2${\scriptstyle \pm 8.1}$} & \textbf{0.17${\scriptstyle \pm 0.37}$} & \textbf{0.992${\scriptstyle \pm 0.018}$} & 126.8 & \textbf{100.0} & \textbf{100.0${\scriptstyle \pm 0.1}$} & \textbf{0.00${\scriptstyle \pm 0.00}$} & \textbf{0.992${\scriptstyle \pm 0.018}$} & 126.8 & \textbf{100.0} & \textbf{100.0${\scriptstyle \pm 0.0}$} & \textbf{0.00${\scriptstyle \pm 0.00}$} & \textbf{0.992${\scriptstyle \pm 0.018}$} & 128.9 \\
\bottomrule
\end{tabular}
\end{adjustbox}
\end{table}

\begin{table}[ht]
\centering
\scriptsize
\setlength{\tabcolsep}{2.0pt}
\renewcommand{\arraystretch}{1.05}
\caption{MRMP results on \texttt{Conveyor}. SR: success rate (\%), CS: constraint safety (\%), Coll: mean collisions, DA: data adherence, Time: computation time (s). Standard deviations are over different initializations and boldface represents best among the methods actively enforcing constraints.}
\label{tab:conveyor results additional}
\begin{adjustbox}{max width=\linewidth}
\begin{tabular}{clccccccccccccccc}
\toprule
\textbf{\# Ag.} & \textbf{Method} & \multicolumn{5}{c}{\textbf{Low ($v_{\max}$=1.21)}} & \multicolumn{5}{c}{\textbf{Medium ($v_{\max}$=1.46)}} & \multicolumn{5}{c}{\textbf{High ($v_{\max}$=1.76)}} \\
\cmidrule(lr){3-7}\cmidrule(lr){8-12}\cmidrule(lr){13-17}
 &  & \textbf{SR}$\uparrow$ & \textbf{CS}$\uparrow$ & \textbf{Coll}$\downarrow$ & \textbf{DA}$\uparrow$ & \textbf{Time (s)}$\downarrow$ & \textbf{SR}$\uparrow$ & \textbf{CS}$\uparrow$ & \textbf{Coll}$\downarrow$ & \textbf{DA}$\uparrow$ & \textbf{Time (s)}$\downarrow$ & \textbf{SR}$\uparrow$ & \textbf{CS}$\uparrow$ & \textbf{Coll}$\downarrow$ & \textbf{DA}$\uparrow$ & \textbf{Time (s)}$\downarrow$ \\
\midrule
\multirow{6}{*}{4} & Diffuser & 1.0 & 0.0${\scriptstyle \pm 0.1}$ & 8.78${\scriptstyle \pm 4.23}$ & 0.960${\scriptstyle \pm 0.039}$ & 1.1${\scriptstyle \pm 0.3}$ & 1.0 & 0.0${\scriptstyle \pm 0.1}$ & 8.78${\scriptstyle \pm 4.23}$ & 0.960${\scriptstyle \pm 0.039}$ & 1.1${\scriptstyle \pm 0.3}$ & 1.0 & 0.0${\scriptstyle \pm 0.1}$ & 8.78${\scriptstyle \pm 4.23}$ & 0.960${\scriptstyle \pm 0.039}$ & 1.1${\scriptstyle \pm 0.3}$ \\
 & MMD-CBS & 100.0 & 100.0 & 0.00 & -- & 6.9${\scriptstyle \pm 0.8}$ & 100.0 & 100.0 & 0.00 & -- & 6.9${\scriptstyle \pm 0.8}$ & 100.0 & 100.0 & 0.00 & -- & 6.9${\scriptstyle \pm 0.8}$ \\
 & PCD-LB & \textbf{100.0} & 24.9${\scriptstyle \pm 14.1}$ & 0.94${\scriptstyle \pm 0.42}$ & 0.896${\scriptstyle \pm 0.077}$ & 2.0 & \textbf{100.0} & 39.6${\scriptstyle \pm 15.1}$ & 0.91${\scriptstyle \pm 0.39}$ & \textbf{0.968${\scriptstyle \pm 0.032}$} & 1.9 & \textbf{100.0} & 41.8${\scriptstyle \pm 15.3}$ & 0.96${\scriptstyle \pm 0.42}$ & \textbf{0.986${\scriptstyle \pm 0.022}$} & 1.8 \\
 & PCD-SHD & \textbf{100.0} & 28.6${\scriptstyle \pm 15.9}$ & 0.09${\scriptstyle \pm 0.08}$ & \textbf{0.902${\scriptstyle \pm 0.069}$} & 2.1 & \textbf{100.0} & 42.0${\scriptstyle \pm 17.5}$ & 0.04${\scriptstyle \pm 0.05}$ & 0.964${\scriptstyle \pm 0.030}$ & 1.9 & \textbf{100.0} & 44.7${\scriptstyle \pm 17.0}$ & 0.01${\scriptstyle \pm 0.02}$ & 0.980${\scriptstyle \pm 0.021}$ & 1.8 \\
 & Final Proj. & \textbf{100.0} & 82.7${\scriptstyle \pm 15.9}$ & 0.15${\scriptstyle \pm 0.19}$ & 0.864${\scriptstyle \pm 0.065}$ & 5.8 & \textbf{100.0} & 94.3${\scriptstyle \pm 12.3}$ & \textbf{0.00${\scriptstyle \pm 0.01}$} & 0.953${\scriptstyle \pm 0.032}$ & 3.6 & \textbf{100.0} & 94.5${\scriptstyle \pm 12.5}$ & \textbf{0.00${\scriptstyle \pm 0.00}$} & 0.980${\scriptstyle \pm 0.019}$ & 5.2 \\
 & DiRecT (Ours) & \textbf{100.0} & \textbf{96.6${\scriptstyle \pm 10.5}$} & \textbf{0.00${\scriptstyle \pm 0.00}$} & 0.835${\scriptstyle \pm 0.072}$ & 50.5 & \textbf{100.0} & \textbf{96.9${\scriptstyle \pm 10.3}$} & \textbf{0.00${\scriptstyle \pm 0.00}$} & 0.937${\scriptstyle \pm 0.043}$ & 27.7 & \textbf{100.0} & \textbf{96.9${\scriptstyle \pm 10.2}$} & \textbf{0.00${\scriptstyle \pm 0.00}$} & 0.981${\scriptstyle \pm 0.017}$ & 61.2 \\
\midrule
\multirow{6}{*}{8} & Diffuser & 0.0 & 0.0${\scriptstyle \pm 0.0}$ & 41.62${\scriptstyle \pm 9.14}$ & 0.960${\scriptstyle \pm 0.025}$ & 1.2${\scriptstyle \pm 0.3}$ & 0.0 & 0.0${\scriptstyle \pm 0.0}$ & 41.62${\scriptstyle \pm 9.14}$ & 0.960${\scriptstyle \pm 0.025}$ & 1.2${\scriptstyle \pm 0.3}$ & 0.0 & 0.0${\scriptstyle \pm 0.0}$ & 41.62${\scriptstyle \pm 9.14}$ & 0.960${\scriptstyle \pm 0.025}$ & 1.2${\scriptstyle \pm 0.3}$ \\
 & MMD-CBS & 100.0 & 100.0 & 0.00 & -- & 22.3${\scriptstyle \pm 2.3}$ & 100.0 & 100.0 & 0.00 & -- & 22.3${\scriptstyle \pm 2.3}$ & 100.0 & 100.0 & 0.00 & -- & 22.3${\scriptstyle \pm 2.3}$ \\
 & PCD-LB & 36.0 & 0.7${\scriptstyle \pm 1.4}$ & 7.87${\scriptstyle \pm 2.00}$ & \textbf{0.895${\scriptstyle \pm 0.052}$} & 2.4 & 48.0 & 1.1${\scriptstyle \pm 2.0}$ & 7.87${\scriptstyle \pm 1.94}$ & \textbf{0.969${\scriptstyle \pm 0.026}$} & 2.7 & 43.0 & 0.7${\scriptstyle \pm 1.2}$ & 8.32${\scriptstyle \pm 1.95}$ & \textbf{0.988${\scriptstyle \pm 0.014}$} & 2.3 \\
 & PCD-SHD & 75.0 & 1.8${\scriptstyle \pm 2.7}$ & 0.39${\scriptstyle \pm 0.13}$ & 0.859${\scriptstyle \pm 0.050}$ & 2.4 & 89.0 & 3.8${\scriptstyle \pm 4.0}$ & 0.15${\scriptstyle \pm 0.06}$ & 0.940${\scriptstyle \pm 0.030}$ & 2.6 & 94.0 & 4.8${\scriptstyle \pm 4.2}$ & 0.04${\scriptstyle \pm 0.02}$ & 0.964${\scriptstyle \pm 0.018}$ & 2.3 \\
 & Final Proj. & \textbf{100.0} & 50.8${\scriptstyle \pm 16.8}$ & 0.96${\scriptstyle \pm 0.70}$ & 0.845${\scriptstyle \pm 0.049}$ & 6.2 & \textbf{100.0} & 87.0${\scriptstyle \pm 16.5}$ & 0.02${\scriptstyle \pm 0.02}$ & 0.949${\scriptstyle \pm 0.023}$ & 4.2 & \textbf{100.0} & 88.5${\scriptstyle \pm 17.1}$ & \textbf{0.00${\scriptstyle \pm 0.01}$} & 0.976${\scriptstyle \pm 0.012}$ & 6.4 \\
 & DiRecT (Ours) & \textbf{100.0} & \textbf{92.7${\scriptstyle \pm 13.4}$} & \textbf{0.01${\scriptstyle \pm 0.01}$} & 0.820${\scriptstyle \pm 0.053}$ & 66.0 & \textbf{100.0} & \textbf{93.9${\scriptstyle \pm 13.4}$} & \textbf{0.00${\scriptstyle \pm 0.00}$} & 0.933${\scriptstyle \pm 0.031}$ & 35.3 & \textbf{100.0} & \textbf{93.8${\scriptstyle \pm 13.5}$} & \textbf{0.00${\scriptstyle \pm 0.00}$} & 0.976${\scriptstyle \pm 0.013}$ & 67.7 \\
\midrule
\multirow{6}{*}{12} & Diffuser & 0.0 & 0.0${\scriptstyle \pm 0.0}$ & 99.55${\scriptstyle \pm 11.98}$ & 0.961${\scriptstyle \pm 0.023}$ & 1.2${\scriptstyle \pm 0.4}$ & 0.0 & 0.0${\scriptstyle \pm 0.0}$ & 99.55${\scriptstyle \pm 11.98}$ & 0.961${\scriptstyle \pm 0.023}$ & 1.2${\scriptstyle \pm 0.4}$ & 0.0 & 0.0${\scriptstyle \pm 0.0}$ & 99.55${\scriptstyle \pm 11.98}$ & 0.961${\scriptstyle \pm 0.023}$ & 1.2${\scriptstyle \pm 0.4}$ \\
 & MMD-CBS & 0.0 & 79.2 & 1.33 & -- & 50.0${\scriptstyle \pm 7.9}$ & 0.0 & 79.2 & 1.33 & -- & 50.0${\scriptstyle \pm 7.9}$ & 0.0 & 79.2 & 1.33 & -- & 50.0${\scriptstyle \pm 7.9}$ \\
 & PCD-LB & 0.0 & 0.0${\scriptstyle \pm 0.0}$ & 27.50${\scriptstyle \pm 4.01}$ & \textbf{0.896${\scriptstyle \pm 0.038}$} & 3.1 & 0.0 & 0.0${\scriptstyle \pm 0.0}$ & 28.07${\scriptstyle \pm 3.59}$ & \textbf{0.973${\scriptstyle \pm 0.017}$} & 3.3 & 0.0 & 0.0${\scriptstyle \pm 0.0}$ & 29.75${\scriptstyle \pm 3.57}$ & \textbf{0.988${\scriptstyle \pm 0.012}$} & 2.8 \\
 & PCD-SHD & 1.0 & 0.0${\scriptstyle \pm 0.1}$ & 0.86${\scriptstyle \pm 0.20}$ & 0.829${\scriptstyle \pm 0.039}$ & 3.1 & 9.0 & 0.1${\scriptstyle \pm 0.3}$ & 0.34${\scriptstyle \pm 0.10}$ & 0.922${\scriptstyle \pm 0.023}$ & 3.4 & 10.0 & 0.1${\scriptstyle \pm 0.2}$ & 0.09${\scriptstyle \pm 0.04}$ & 0.952${\scriptstyle \pm 0.017}$ & 2.8 \\
 & Final Proj. & \textbf{100.0} & 24.7${\scriptstyle \pm 10.0}$ & 2.65${\scriptstyle \pm 1.13}$ & 0.837${\scriptstyle \pm 0.034}$ & 6.5 & \textbf{100.0} & 82.1${\scriptstyle \pm 13.9}$ & 0.09${\scriptstyle \pm 0.07}$ & 0.944${\scriptstyle \pm 0.016}$ & 5.3 & \textbf{100.0} & 87.1${\scriptstyle \pm 15.4}$ & \textbf{0.00${\scriptstyle \pm 0.01}$} & 0.972${\scriptstyle \pm 0.011}$ & 6.6 \\
 & DiRecT (Ours) & \textbf{100.0} & \textbf{89.9${\scriptstyle \pm 11.6}$} & \textbf{0.04${\scriptstyle \pm 0.04}$} & 0.823${\scriptstyle \pm 0.035}$ & 66.7 & \textbf{100.0} & \textbf{93.0${\scriptstyle \pm 11.8}$} & \textbf{0.00${\scriptstyle \pm 0.01}$} & 0.931${\scriptstyle \pm 0.021}$ & 53.5 & \textbf{100.0} & \textbf{93.3${\scriptstyle \pm 12.0}$} & \textbf{0.00${\scriptstyle \pm 0.00}$} & 0.973${\scriptstyle \pm 0.013}$ & 67.3 \\
\midrule
\multirow{6}{*}{16} & Diffuser & 0.0 & 0.0${\scriptstyle \pm 0.0}$ & 181.11${\scriptstyle \pm 19.59}$ & 0.962${\scriptstyle \pm 0.019}$ & 1.6${\scriptstyle \pm 0.2}$ & 0.0 & 0.0${\scriptstyle \pm 0.0}$ & 181.11${\scriptstyle \pm 19.59}$ & 0.962${\scriptstyle \pm 0.019}$ & 1.6${\scriptstyle \pm 0.2}$ & 0.0 & 0.0${\scriptstyle \pm 0.0}$ & 181.11${\scriptstyle \pm 19.59}$ & 0.962${\scriptstyle \pm 0.019}$ & 1.6${\scriptstyle \pm 0.2}$ \\
 & MMD-CBS & 0.0 & 6.2 & 17.33 & -- & 63.2${\scriptstyle \pm 2.3}$ & 0.0 & 6.2 & 17.33 & -- & 63.2${\scriptstyle \pm 2.3}$ & 0.0 & 6.2 & 17.33 & -- & 63.2${\scriptstyle \pm 2.3}$ \\
 & PCD-LB & 0.0 & 0.0${\scriptstyle \pm 0.0}$ & 64.74${\scriptstyle \pm 8.28}$ & \textbf{0.902${\scriptstyle \pm 0.036}$} & 3.9 & 0.0 & 0.0${\scriptstyle \pm 0.0}$ & 65.15${\scriptstyle \pm 7.61}$ & \textbf{0.976${\scriptstyle \pm 0.014}$} & 4.3 & 0.0 & 0.0${\scriptstyle \pm 0.0}$ & 68.06${\scriptstyle \pm 7.30}$ & \textbf{0.988${\scriptstyle \pm 0.010}$} & 3.6 \\
 & PCD-SHD & 0.0 & 0.0${\scriptstyle \pm 0.0}$ & 1.58${\scriptstyle \pm 0.29}$ & 0.809${\scriptstyle \pm 0.035}$ & 4.0 & 0.0 & 0.0${\scriptstyle \pm 0.0}$ & 0.64${\scriptstyle \pm 0.13}$ & 0.910${\scriptstyle \pm 0.024}$ & 4.4 & 1.0 & 0.0${\scriptstyle \pm 0.1}$ & 0.17${\scriptstyle \pm 0.06}$ & 0.943${\scriptstyle \pm 0.017}$ & 3.6 \\
 & Final Proj. & 98.0 & 8.4${\scriptstyle \pm 5.9}$ & 5.49${\scriptstyle \pm 2.23}$ & 0.822${\scriptstyle \pm 0.035}$ & 6.5 & \textbf{100.0} & 67.9${\scriptstyle \pm 18.8}$ & 0.26${\scriptstyle \pm 0.18}$ & 0.937${\scriptstyle \pm 0.017}$ & 9.8 & \textbf{100.0} & 79.8${\scriptstyle \pm 20.7}$ & 0.01${\scriptstyle \pm 0.02}$ & 0.967${\scriptstyle \pm 0.012}$ & 6.8 \\
 & DiRecT (Ours) & \textbf{100.0} & \textbf{81.4${\scriptstyle \pm 16.0}$} & \textbf{0.08${\scriptstyle \pm 0.06}$} & 0.820${\scriptstyle \pm 0.030}$ & 68.9 & \textbf{100.0} & \textbf{87.2${\scriptstyle \pm 17.2}$} & \textbf{0.00${\scriptstyle \pm 0.01}$} & 0.927${\scriptstyle \pm 0.020}$ & 99.1 & \textbf{100.0} & \textbf{87.3${\scriptstyle \pm 17.3}$} & \textbf{0.00${\scriptstyle \pm 0.00}$} & 0.970${\scriptstyle \pm 0.011}$ & 69.2 \\
\midrule
\multirow{6}{*}{20} & Diffuser & 0.0 & 0.0${\scriptstyle \pm 0.0}$ & 293.67${\scriptstyle \pm 20.60}$ & 0.961${\scriptstyle \pm 0.016}$ & 1.8${\scriptstyle \pm 0.1}$ & 0.0 & 0.0${\scriptstyle \pm 0.0}$ & 293.67${\scriptstyle \pm 20.60}$ & 0.961${\scriptstyle \pm 0.016}$ & 1.8${\scriptstyle \pm 0.1}$ & 0.0 & 0.0${\scriptstyle \pm 0.0}$ & 293.67${\scriptstyle \pm 20.60}$ & 0.961${\scriptstyle \pm 0.016}$ & 1.8${\scriptstyle \pm 0.1}$ \\
 & MMD-CBS & 0.0 & 0.0 & 35.83 & -- & 66.2${\scriptstyle \pm 5.7}$ & 0.0 & 0.0 & 35.83 & -- & 66.2${\scriptstyle \pm 5.7}$ & 0.0 & 0.0 & 35.83 & -- & 66.2${\scriptstyle \pm 5.7}$ \\
 & PCD-LB & 0.0 & 0.0${\scriptstyle \pm 0.0}$ & 121.49${\scriptstyle \pm 10.18}$ & \textbf{0.896${\scriptstyle \pm 0.031}$} & 4.4 & 0.0 & 0.0${\scriptstyle \pm 0.0}$ & 121.45${\scriptstyle \pm 9.39}$ & \textbf{0.973${\scriptstyle \pm 0.013}$} & 5.5 & 0.0 & 0.0${\scriptstyle \pm 0.0}$ & 125.76${\scriptstyle \pm 8.99}$ & \textbf{0.987${\scriptstyle \pm 0.009}$} & 4.2 \\
 & PCD-SHD & 0.0 & 0.0${\scriptstyle \pm 0.0}$ & 2.74${\scriptstyle \pm 0.42}$ & 0.778${\scriptstyle \pm 0.032}$ & 4.5 & 0.0 & 0.0${\scriptstyle \pm 0.0}$ & 1.16${\scriptstyle \pm 0.18}$ & 0.886${\scriptstyle \pm 0.024}$ & 5.5 & 0.0 & 0.0${\scriptstyle \pm 0.0}$ & 0.30${\scriptstyle \pm 0.08}$ & 0.926${\scriptstyle \pm 0.018}$ & 4.2 \\
 & Final Proj. & 65.0 & 1.2${\scriptstyle \pm 1.3}$ & 11.43${\scriptstyle \pm 3.99}$ & 0.806${\scriptstyle \pm 0.033}$ & 9.1 & \textbf{100.0} & 50.9${\scriptstyle \pm 15.8}$ & 0.74${\scriptstyle \pm 0.50}$ & 0.927${\scriptstyle \pm 0.017}$ & 12.5 & \textbf{100.0} & 75.0${\scriptstyle \pm 19.8}$ & 0.04${\scriptstyle \pm 0.06}$ & 0.960${\scriptstyle \pm 0.011}$ & 9.2 \\
 & DiRecT (Ours) & \textbf{100.0} & \textbf{72.4${\scriptstyle \pm 16.2}$} & \textbf{0.25${\scriptstyle \pm 0.18}$} & 0.815${\scriptstyle \pm 0.028}$ & 97.8 & \textbf{100.0} & \textbf{86.5${\scriptstyle \pm 17.0}$} & \textbf{0.01${\scriptstyle \pm 0.02}$} & 0.921${\scriptstyle \pm 0.020}$ & 125.4 & \textbf{100.0} & \textbf{87.3${\scriptstyle \pm 17.3}$} & \textbf{0.00${\scriptstyle \pm 0.00}$} & 0.965${\scriptstyle \pm 0.010}$ & 98.1 \\
\bottomrule
\end{tabular}
\end{adjustbox}
\end{table}

\clearpage
\begin{table}[H]
\centering
\scriptsize
\setlength{\tabcolsep}{2.0pt}
\renewcommand{\arraystretch}{1.05}
\caption{MRMP results on \texttt{Drop-Region}. SR: success rate (\%), CS: constraint safety (\%), Coll: mean collisions, DA: data adherence, Time: computation time (s). Standard deviations are over different initializations and boldface represents best among the methods actively enforcing constraints.}
\label{tab:dropregion results additional}
\begin{adjustbox}{max width=\linewidth}
\begin{tabular}{clccccccccccccccc}
\toprule
\textbf{\# Ag.} & \textbf{Method} & \multicolumn{5}{c}{\textbf{Low ($v_{\max}$=0.928)}} & \multicolumn{5}{c}{\textbf{Medium ($v_{\max}$=1.13)}} & \multicolumn{5}{c}{\textbf{High ($v_{\max}$=1.34)}} \\
\cmidrule(lr){3-7}\cmidrule(lr){8-12}\cmidrule(lr){13-17}
 &  & \textbf{SR}$\uparrow$ & \textbf{CS}$\uparrow$ & \textbf{Coll}$\downarrow$ & \textbf{DA}$\uparrow$ & \textbf{Time (s)}$\downarrow$ & \textbf{SR}$\uparrow$ & \textbf{CS}$\uparrow$ & \textbf{Coll}$\downarrow$ & \textbf{DA}$\uparrow$ & \textbf{Time (s)}$\downarrow$ & \textbf{SR}$\uparrow$ & \textbf{CS}$\uparrow$ & \textbf{Coll}$\downarrow$ & \textbf{DA}$\uparrow$ & \textbf{Time (s)}$\downarrow$ \\
\midrule
\multirow{6}{*}{4} & Diffuser & 40.0 & 1.1${\scriptstyle \pm 1.8}$ & 17.17${\scriptstyle \pm 3.75}$ & 0.991${\scriptstyle \pm 0.007}$ & 0.6${\scriptstyle \pm 0.4}$ & 40.0 & 1.1${\scriptstyle \pm 1.8}$ & 17.17${\scriptstyle \pm 3.75}$ & 0.991${\scriptstyle \pm 0.007}$ & 0.6${\scriptstyle \pm 0.4}$ & 40.0 & 1.1${\scriptstyle \pm 1.8}$ & 17.17${\scriptstyle \pm 3.75}$ & 0.991${\scriptstyle \pm 0.007}$ & 0.6${\scriptstyle \pm 0.4}$ \\
 & MMD-CBS & 100.0 & 100.0 & 0.00 & -- & 9.2${\scriptstyle \pm 1.0}$ & 100.0 & 100.0 & 0.00 & -- & 9.2${\scriptstyle \pm 1.0}$ & 100.0 & 100.0 & 0.00 & -- & 9.2${\scriptstyle \pm 1.0}$ \\
 & PCD-LB & \textbf{100.0} & 38.6${\scriptstyle \pm 10.2}$ & 2.74${\scriptstyle \pm 0.85}$ & \textbf{0.980${\scriptstyle \pm 0.017}$} & 1.9 & \textbf{100.0} & 36.6${\scriptstyle \pm 9.8}$ & 2.89${\scriptstyle \pm 0.86}$ & 0.986${\scriptstyle \pm 0.013}$ & 1.9 & \textbf{100.0} & 36.1${\scriptstyle \pm 9.0}$ & 2.92${\scriptstyle \pm 0.80}$ & 0.987${\scriptstyle \pm 0.010}$ & 1.8 \\
 & PCD-SHD & \textbf{100.0} & 88.0${\scriptstyle \pm 6.2}$ & 0.05${\scriptstyle \pm 0.03}$ & 0.965${\scriptstyle \pm 0.018}$ & 1.9 & \textbf{100.0} & 88.0${\scriptstyle \pm 5.1}$ & 0.04${\scriptstyle \pm 0.03}$ & 0.965${\scriptstyle \pm 0.017}$ & 1.7 & \textbf{100.0} & 86.9${\scriptstyle \pm 5.2}$ & 0.02${\scriptstyle \pm 0.02}$ & 0.965${\scriptstyle \pm 0.016}$ & 1.7 \\
 & Final Proj. & \textbf{100.0} & 84.6${\scriptstyle \pm 11.9}$ & 0.35${\scriptstyle \pm 0.39}$ & \textbf{0.980${\scriptstyle \pm 0.015}$} & 2.2 & \textbf{100.0} & 98.4${\scriptstyle \pm 2.1}$ & 0.03${\scriptstyle \pm 0.05}$ & \textbf{0.989${\scriptstyle \pm 0.007}$} & 2.1 & \textbf{100.0} & 99.8${\scriptstyle \pm 0.5}$ & \textbf{0.00${\scriptstyle \pm 0.01}$} & \textbf{0.990${\scriptstyle \pm 0.007}$} & 5.9 \\
 & DiRecT (Ours) & \textbf{100.0} & \textbf{100.0${\scriptstyle \pm 0.0}$} & \textbf{0.00${\scriptstyle \pm 0.00}$} & 0.971${\scriptstyle \pm 0.021}$ & 25.4 & \textbf{100.0} & \textbf{100.0${\scriptstyle \pm 0.0}$} & \textbf{0.00${\scriptstyle \pm 0.00}$} & 0.985${\scriptstyle \pm 0.011}$ & 24.2 & \textbf{100.0} & \textbf{100.0${\scriptstyle \pm 0.0}$} & \textbf{0.00${\scriptstyle \pm 0.00}$} & 0.989${\scriptstyle \pm 0.008}$ & 64.0 \\
\midrule
\multirow{6}{*}{8} & Diffuser & 0.0 & 0.0${\scriptstyle \pm 0.0}$ & 82.64${\scriptstyle \pm 9.51}$ & 0.991${\scriptstyle \pm 0.005}$ & 0.6${\scriptstyle \pm 0.4}$ & 0.0 & 0.0${\scriptstyle \pm 0.0}$ & 82.64${\scriptstyle \pm 9.51}$ & 0.991${\scriptstyle \pm 0.005}$ & 0.6${\scriptstyle \pm 0.4}$ & 0.0 & 0.0${\scriptstyle \pm 0.0}$ & 82.64${\scriptstyle \pm 9.51}$ & 0.991${\scriptstyle \pm 0.005}$ & 0.6${\scriptstyle \pm 0.4}$ \\
 & MMD-CBS & 100.0 & 100.0 & 0.00 & -- & 55.5${\scriptstyle \pm 1.9}$ & 100.0 & 100.0 & 0.00 & -- & 55.5${\scriptstyle \pm 1.9}$ & 100.0 & 100.0 & 0.00 & -- & 55.5${\scriptstyle \pm 1.9}$ \\
 & PCD-LB & 11.0 & 0.1${\scriptstyle \pm 0.3}$ & 29.94${\scriptstyle \pm 3.91}$ & \textbf{0.981${\scriptstyle \pm 0.014}$} & 3.5 & 7.0 & 0.1${\scriptstyle \pm 0.2}$ & 31.40${\scriptstyle \pm 3.80}$ & \textbf{0.988${\scriptstyle \pm 0.009}$} & 2.3 & 2.0 & 0.0${\scriptstyle \pm 0.1}$ & 32.13${\scriptstyle \pm 3.91}$ & \textbf{0.989${\scriptstyle \pm 0.008}$} & 2.1 \\
 & PCD-SHD & \textbf{100.0} & 26.3${\scriptstyle \pm 7.1}$ & 0.45${\scriptstyle \pm 0.12}$ & 0.883${\scriptstyle \pm 0.021}$ & 3.7 & \textbf{100.0} & 24.6${\scriptstyle \pm 6.5}$ & 0.34${\scriptstyle \pm 0.09}$ & 0.883${\scriptstyle \pm 0.020}$ & 2.3 & \textbf{100.0} & 23.1${\scriptstyle \pm 6.1}$ & 0.18${\scriptstyle \pm 0.05}$ & 0.880${\scriptstyle \pm 0.021}$ & 2.1 \\
 & Final Proj. & \textbf{100.0} & 59.7${\scriptstyle \pm 16.6}$ & 1.76${\scriptstyle \pm 1.10}$ & 0.971${\scriptstyle \pm 0.012}$ & 2.6 & \textbf{100.0} & 94.6${\scriptstyle \pm 4.1}$ & 0.14${\scriptstyle \pm 0.15}$ & 0.983${\scriptstyle \pm 0.007}$ & 2.6 & \textbf{100.0} & 97.9${\scriptstyle \pm 1.5}$ & 0.02${\scriptstyle \pm 0.02}$ & 0.985${\scriptstyle \pm 0.007}$ & 5.3 \\
 & DiRecT (Ours) & \textbf{100.0} & \textbf{100.0${\scriptstyle \pm 0.1}$} & \textbf{0.00${\scriptstyle \pm 0.01}$} & 0.965${\scriptstyle \pm 0.015}$ & 29.2 & \textbf{100.0} & \textbf{100.0${\scriptstyle \pm 0.1}$} & \textbf{0.00${\scriptstyle \pm 0.00}$} & 0.982${\scriptstyle \pm 0.009}$ & 28.9 & \textbf{100.0} & \textbf{100.0${\scriptstyle \pm 0.2}$} & \textbf{0.00${\scriptstyle \pm 0.00}$} & 0.986${\scriptstyle \pm 0.007}$ & 65.6 \\
\midrule
\multirow{6}{*}{12} & Diffuser & 0.0 & 0.0${\scriptstyle \pm 0.0}$ & 194.40${\scriptstyle \pm 14.29}$ & 0.991${\scriptstyle \pm 0.004}$ & 0.7${\scriptstyle \pm 0.5}$ & 0.0 & 0.0${\scriptstyle \pm 0.0}$ & 194.40${\scriptstyle \pm 14.29}$ & 0.991${\scriptstyle \pm 0.004}$ & 0.7${\scriptstyle \pm 0.5}$ & 0.0 & 0.0${\scriptstyle \pm 0.0}$ & 194.40${\scriptstyle \pm 14.29}$ & 0.991${\scriptstyle \pm 0.004}$ & 0.7${\scriptstyle \pm 0.5}$ \\
 & MMD-CBS & 0.0 & 8.3 & 12.00 & -- & 67.8${\scriptstyle \pm 1.0}$ & 0.0 & 8.3 & 12.00 & -- & 67.8${\scriptstyle \pm 1.0}$ & 0.0 & 8.3 & 12.00 & -- & 67.8${\scriptstyle \pm 1.0}$ \\
 & PCD-LB & 0.0 & 0.0${\scriptstyle \pm 0.0}$ & 94.96${\scriptstyle \pm 9.00}$ & \textbf{0.982${\scriptstyle \pm 0.010}$} & 3.0 & 0.0 & 0.0${\scriptstyle \pm 0.0}$ & 98.91${\scriptstyle \pm 8.58}$ & \textbf{0.989${\scriptstyle \pm 0.006}$} & 2.7 & 0.0 & 0.0${\scriptstyle \pm 0.0}$ & 100.86${\scriptstyle \pm 8.58}$ & \textbf{0.990${\scriptstyle \pm 0.005}$} & 2.6 \\
 & PCD-SHD & 71.0 & 1.1${\scriptstyle \pm 1.0}$ & 1.45${\scriptstyle \pm 0.19}$ & 0.787${\scriptstyle \pm 0.025}$ & 4.5 & 60.0 & 0.8${\scriptstyle \pm 0.9}$ & 1.08${\scriptstyle \pm 0.18}$ & 0.787${\scriptstyle \pm 0.023}$ & 2.8 & 50.0 & 0.6${\scriptstyle \pm 0.8}$ & 0.59${\scriptstyle \pm 0.09}$ & 0.782${\scriptstyle \pm 0.023}$ & 2.7 \\
 & Final Proj. & \textbf{100.0} & 33.4${\scriptstyle \pm 14.3}$ & 4.80${\scriptstyle \pm 2.24}$ & 0.952${\scriptstyle \pm 0.012}$ & 4.1 & \textbf{100.0} & 83.6${\scriptstyle \pm 5.6}$ & 0.45${\scriptstyle \pm 0.28}$ & 0.968${\scriptstyle \pm 0.008}$ & 4.1 & \textbf{100.0} & 89.6${\scriptstyle \pm 2.8}$ & 0.12${\scriptstyle \pm 0.04}$ & 0.973${\scriptstyle \pm 0.007}$ & 8.2 \\
 & DiRecT (Ours) & \textbf{100.0} & \textbf{99.9${\scriptstyle \pm 0.2}$} & \textbf{0.00${\scriptstyle \pm 0.01}$} & 0.955${\scriptstyle \pm 0.016}$ & 47.7 & \textbf{100.0} & \textbf{99.8${\scriptstyle \pm 0.3}$} & \textbf{0.00${\scriptstyle \pm 0.01}$} & 0.975${\scriptstyle \pm 0.008}$ & 47.4 & \textbf{100.0} & \textbf{99.6${\scriptstyle \pm 0.6}$} & \textbf{0.00${\scriptstyle \pm 0.01}$} & 0.981${\scriptstyle \pm 0.006}$ & 67.0 \\
\midrule
\multirow{6}{*}{16} & Diffuser & 0.0 & 0.0${\scriptstyle \pm 0.0}$ & 352.68${\scriptstyle \pm 19.38}$ & 0.991${\scriptstyle \pm 0.003}$ & 0.8${\scriptstyle \pm 0.5}$ & 0.0 & 0.0${\scriptstyle \pm 0.0}$ & 352.68${\scriptstyle \pm 19.38}$ & 0.991${\scriptstyle \pm 0.003}$ & 0.8${\scriptstyle \pm 0.5}$ & 0.0 & 0.0${\scriptstyle \pm 0.0}$ & 352.68${\scriptstyle \pm 19.38}$ & 0.991${\scriptstyle \pm 0.003}$ & 0.8${\scriptstyle \pm 0.5}$ \\
 & MMD-CBS & 0.0 & 12.5 & 34.00 & -- & 77.2${\scriptstyle \pm 3.7}$ & 0.0 & 12.5 & 34.00 & -- & 77.2${\scriptstyle \pm 3.7}$ & 0.0 & 12.5 & 34.00 & -- & 77.2${\scriptstyle \pm 3.7}$ \\
 & PCD-LB & 0.0 & 0.0${\scriptstyle \pm 0.0}$ & 196.78${\scriptstyle \pm 13.86}$ & \textbf{0.983${\scriptstyle \pm 0.008}$} & 4.1 & 0.0 & 0.0${\scriptstyle \pm 0.0}$ & 204.24${\scriptstyle \pm 12.63}$ & \textbf{0.989${\scriptstyle \pm 0.005}$} & 3.5 & 0.0 & 0.0${\scriptstyle \pm 0.0}$ & 208.07${\scriptstyle \pm 12.00}$ & \textbf{0.991${\scriptstyle \pm 0.004}$} & 3.3 \\
 & PCD-SHD & 1.0 & 0.0${\scriptstyle \pm 0.1}$ & 3.12${\scriptstyle \pm 0.36}$ & 0.710${\scriptstyle \pm 0.024}$ & 5.4 & 1.0 & 0.0${\scriptstyle \pm 0.1}$ & 2.24${\scriptstyle \pm 0.25}$ & 0.709${\scriptstyle \pm 0.021}$ & 3.5 & 0.0 & 0.0${\scriptstyle \pm 0.0}$ & 1.27${\scriptstyle \pm 0.13}$ & 0.703${\scriptstyle \pm 0.021}$ & 3.3 \\
 & Final Proj. & \textbf{100.0} & 14.9${\scriptstyle \pm 8.8}$ & 9.88${\scriptstyle \pm 3.38}$ & 0.919${\scriptstyle \pm 0.011}$ & 6.0 & \textbf{100.0} & 61.8${\scriptstyle \pm 7.4}$ & 1.25${\scriptstyle \pm 0.53}$ & 0.940${\scriptstyle \pm 0.009}$ & 6.0 & \textbf{100.0} & 70.6${\scriptstyle \pm 5.3}$ & 0.42${\scriptstyle \pm 0.10}$ & 0.950${\scriptstyle \pm 0.008}$ & 6.8 \\
 & DiRecT (Ours) & \textbf{100.0} & \textbf{99.5${\scriptstyle \pm 0.7}$} & \textbf{0.01${\scriptstyle \pm 0.02}$} & 0.938${\scriptstyle \pm 0.012}$ & 71.2 & \textbf{100.0} & \textbf{99.1${\scriptstyle \pm 0.9}$} & \textbf{0.01${\scriptstyle \pm 0.01}$} & 0.961${\scriptstyle \pm 0.008}$ & 70.9 & \textbf{100.0} & \textbf{98.4${\scriptstyle \pm 1.1}$} & \textbf{0.02${\scriptstyle \pm 0.01}$} & 0.971${\scriptstyle \pm 0.006}$ & 72.5 \\
\midrule
\multirow{6}{*}{20} & Diffuser & 0.0 & 0.0${\scriptstyle \pm 0.0}$ & 568.27${\scriptstyle \pm 23.79}$ & 0.992${\scriptstyle \pm 0.003}$ & 1.0${\scriptstyle \pm 0.5}$ & 0.0 & 0.0${\scriptstyle \pm 0.0}$ & 568.27${\scriptstyle \pm 23.79}$ & 0.992${\scriptstyle \pm 0.003}$ & 1.0${\scriptstyle \pm 0.5}$ & 0.0 & 0.0${\scriptstyle \pm 0.0}$ & 568.27${\scriptstyle \pm 23.79}$ & 0.992${\scriptstyle \pm 0.003}$ & 1.0${\scriptstyle \pm 0.5}$ \\
 & MMD-CBS & 0.0 & 0.0 & 48.00 & -- & 99.0${\scriptstyle \pm 23.3}$ & 0.0 & 0.0 & 48.00 & -- & 99.0${\scriptstyle \pm 23.3}$ & 0.0 & 0.0 & 48.00 & -- & 99.0${\scriptstyle \pm 23.3}$ \\
 & PCD-LB & 0.0 & 0.0${\scriptstyle \pm 0.0}$ & 342.21${\scriptstyle \pm 16.54}$ & \textbf{0.982${\scriptstyle \pm 0.007}$} & 4.1 & 0.0 & 0.0${\scriptstyle \pm 0.0}$ & 355.78${\scriptstyle \pm 15.43}$ & \textbf{0.989${\scriptstyle \pm 0.004}$} & 4.0 & 0.0 & 0.0${\scriptstyle \pm 0.0}$ & 361.93${\scriptstyle \pm 15.70}$ & \textbf{0.990${\scriptstyle \pm 0.003}$} & 3.8 \\
 & PCD-SHD & 0.0 & 0.0${\scriptstyle \pm 0.0}$ & 5.38${\scriptstyle \pm 0.59}$ & 0.646${\scriptstyle \pm 0.015}$ & 6.7 & 0.0 & 0.0${\scriptstyle \pm 0.0}$ & 3.82${\scriptstyle \pm 0.35}$ & 0.646${\scriptstyle \pm 0.015}$ & 4.0 & 0.0 & 0.0${\scriptstyle \pm 0.0}$ & 2.20${\scriptstyle \pm 0.20}$ & 0.639${\scriptstyle \pm 0.014}$ & 3.8 \\
 & Final Proj. & 88.0 & 3.1${\scriptstyle \pm 2.6}$ & 19.12${\scriptstyle \pm 4.73}$ & 0.866${\scriptstyle \pm 0.017}$ & 8.4 & \textbf{100.0} & 32.1${\scriptstyle \pm 5.2}$ & 2.92${\scriptstyle \pm 0.69}$ & 0.892${\scriptstyle \pm 0.014}$ & 8.3 & \textbf{100.0} & 42.0${\scriptstyle \pm 5.1}$ & 1.14${\scriptstyle \pm 0.18}$ & 0.908${\scriptstyle \pm 0.013}$ & 8.4 \\
 & DiRecT (Ours) & \textbf{100.0} & \textbf{97.7${\scriptstyle \pm 1.3}$} & \textbf{0.06${\scriptstyle \pm 0.04}$} & 0.899${\scriptstyle \pm 0.018}$ & 100.5 & \textbf{100.0} & \textbf{96.3${\scriptstyle \pm 1.6}$} & \textbf{0.06${\scriptstyle \pm 0.03}$} & 0.933${\scriptstyle \pm 0.010}$ & 99.7 & \textbf{100.0} & \textbf{94.1${\scriptstyle \pm 2.4}$} & \textbf{0.06${\scriptstyle \pm 0.03}$} & 0.947${\scriptstyle \pm 0.009}$ & 103.5 \\
\bottomrule
\end{tabular}
\end{adjustbox}
\end{table}